\documentclass{article}

\usepackage[preprint]{neurips_2026}
\usepackage[T1]{fontenc}    
\usepackage{hyperref}       
\usepackage{url}            
\usepackage{booktabs}       
\usepackage{amsfonts}       
\usepackage{nicefrac}       
\usepackage{microtype}      
\usepackage{xcolor}         
\usepackage{mathtools}
\usepackage{caption}
\usepackage{graphicx}
\usepackage{subcaption}
\usepackage{mathrsfs}
\usepackage{amsmath,amsthm,amssymb,stmaryrd}
\usepackage{dsfont}
\usepackage{bm}
\usepackage{centernot}
\usepackage{enumitem}
\usepackage[capitalize,noabbrev]{cleveref}

\newcommand\given[1][]{\:\ifnum\currentgrouptype=16\middle\fi|\:}
\newcommand\stcolon[1][]{\:\colon\:}
\newcommand\stbar\given
\makeatletter
\def\st{\@ifstar\stbar\stcolon}

\def\mC{{\bm{C}}}

\def\mI{{\bm{I}}}

\def\mR{{\bm{R}}}

\def\mT{{\bm{T}}}

\newcommand{\Rb}{\mathbb{R}}
\newcommand{\Eb}{\mathbb{E}}
\newcommand{\Pb}{\mathbb{P}}
\newcommand{\Bb}{\mathbb{B}}
\newcommand{\Nb}{\mathbb{N}}
\newcommand{\BEoS}{\vec{\Theta}_{\mathrm{BEoS}}}

\newcommand{\CN}{\mathcal{N}}

\newcommand{\loss}{\mathcal{L}}
\newcommand{\relu}{\mathrm{ReLU}}
\newcommand{\Dict}{\mathscr{D}}
\newcommand{\Variation}{\mathrm{V}}

\newcommand{\MSE}{\mathrm{MSE}}

\newcommand{\DictReLU}{\Dict_{\phi}}

\newcommand{\Iall}{\mathsf{I}^{\mathrm{all}}}
\newcommand{\Oany}{\mathsf{O}^{\mathrm{any}}}
\newcommand{\sfI}{\mathsf{I}}
\newcommand{\sfO}{\mathsf{O}}
\newcommand{\poly}{\mathrm{poly}}

\newcommand{\Gap}{\mathrm{Gap}}

\newcommand{\Risk}{\mathcal{R}}

\renewcommand{\vec}[1]{{\bm{#1}}}

\newcommand{\T}{\mathsf{T}}
\newcommand{\dd}{\,\mathrm d}

\newcommand{\curly}[1]{\left\{#1\right\}}

\newcommand{\norm}[1]{\left\lVert#1\right\rVert}
\newcommand{\pathnorm}[1]{\left\lVert#1\right\rVert_{\mathrm{path}}}

\newcommand{\Ib}{\mathbb{I}}

\newcommand{\Sph}{\mathbb{S}}
\newcommand{\M}{\mathcal{M}}
\newcommand{\F}{\mathcal{F}}

\newcommand{\GaussC}{\mathcal{G}}

\newcommand{\CP}{\mathcal{P}}
\newcommand{\CD}{\mathcal{D}}

\newcommand{\CS}{\mathcal{S}}
\newcommand{\Excess}{\mathrm{Excess}}
\newcommand{\Gen}{\mathrm{Gen}}
\newcommand{\proj}{\operatorname{proj}}

\theoremstyle{plain}
\newtheorem{theorem}{Theorem}[section]
\newtheorem*{theorem*}{Theorem}
\newtheorem{proposition}[theorem]{Proposition}
\newtheorem{lemma}[theorem]{Lemma}
\newtheorem{corollary}[theorem]{Corollary}

\theoremstyle{definition}
\newtheorem{definition}[theorem]{Definition}

\newtheorem{remark}[theorem]{Remark}

\newtheorem{construction}[theorem]{Construction}

\title{{Does Sparse Connectivity Improve Generalization? \\Convolutional Networks Below the Edge of Stability}}

\author{
Tongtong Liang \\ UC San Diego \\ \texttt{ttliang@ucsd.edu}
\And
Esha Singh \\ UC San Diego \\ \texttt{e3singh@ucsd.edu}
\And
Rahul Parhi \\ UC San Diego \\ \texttt{rahul@ucsd.edu}
\And
Alexander Cloninger \\ UC San Diego \\ \texttt{acloninger@ucsd.edu} \\
\And
Yu-Xiang Wang \\ UC San Diego \\ \texttt{yuxiangw@ucsd.edu} \\
}

\begin{document}

\maketitle
\begin{abstract}
Gradient descent on overparameterized neural networks typically operates at the Edge of Stability (EoS), where the largest Hessian eigenvalue hovers around a step-size-dependent threshold. We study how sparse connectivity changes generalization below this threshold in two-layer ReLU networks. Prior results have shown that for fully-connected networks (FCNs), generalization guarantees in this regime degrade and become vacuous on high-dimensional spherical inputs. Our analysis reveals that sparse connectivity fundamentally alters this picture. Under sparse connectivity, the network processes a collection of low-dimensional patches rather than the full input vector, so the effective constraint imposed by the stability condition is governed by the geometry of the training patch collection. We prove that when the receptive fields are small relative to the ambient dimension, the effective constraint yields non-vacuous generalization bounds in precisely the spherical regime where FCNs provably fail. The same framework also reveals a contrasting failure mode: if the patch collection lacks geometric structure, the constraint becomes unable to prevent overfitting. We corroborate this theory by analyzing the patch geometry of natural images, showing that standard convolutional designs produce patch multiset with low-dimensional structure that facilitates generalization. This provides a principled explanation for the generalization advantage of convolutional networks. Thus, our analysis yields a unified framework that identifies how architecture, data geometry, and gradient descent jointly govern generalization performance.
\end{abstract}
\section{Introduction}
Convolutional neural networks (CNNs) \citep{fukushima1988neocognitron,lecun1998gradient} are among the most successful neural architectures. The renaissance of CNNs in computer vision tasks \citep{krizhevsky2012imagenet,ronneberger2015u,he2016deep} effectively ignited the modern deep learning revolution. CNNs remain a dominant choice in image and video generative models today \citep[see, e.g.,][and the references therein]{lai2025principles}.

Modern CNN models are \emph{overparameterized} and often have more parameters than the number of data points to fit them. It was shown that, despite the ability to fit pure noise perfectly, overparameterized CNNs trained with gradient descent tend to generalize well without explicit regularization, much more so than their densely connected counterparts \citep{zhang2017rethinking,arpit2017memorization}.  It remains a mystery how the CNN architecture interacts with the implicit regularization of gradient descent that ends up finding generalizing solutions. 

A tractable way to study the implicit regularization of gradient descent is to focus on the properties shared by solutions that GD can practically reach and maintain. Empirical work shows that gradient descent on neural networks typically enters the Edge-of-Stability (EoS) regime: the largest eigenvalue of the training loss Hessian increases, and then hovers around a stability threshold controlled by the step size~\citep{cohen2021gradient}. Theoretical analyses of minima stability prove that solutions that GD can stably converge to  are inherently constrained by this threshold~\citep{wu2018sgd,nar2018step,mulayoff2021implicit,nacson2022implicit,damian2023self}. 
We refer to such solutions as \emph{below the edge of stability} (BEoS), following~\citet{qiao2024stable,liang2025datageometry}, and take the BEoS condition as the analytical setup for studying the implicit regularization of GD.

Recent work has analyzed this problem for two-layer ReLU fully connected networks (FCNs), establishing that generalization below the edge of stability depends on the geometry of ambient input vectors ~\citep{mulayoff2021implicit,nacson2022implicit,qiao2024stable,liang2025_neural_shattering,liang2025datageometry}. In particular, \citet{liang2025_neural_shattering,liang2025datageometry} show that as the input distribution concentrates closer to a sphere, the resulting generalization guarantees eventually become vacuous.

Standard normalization preprocessing often maps images near a hypersphere, yet CNNs trained on such data generalize well in this regime. The discrepancy points to the architectural inductive bias. A convolutional layer departs from a fully-connected one in that each neuron only receives input from a small subset of coordinates, a property we term \emph{sparse connectivity}. On the forward side, the network no longer operates on full input vectors but on a collection of lower-dimensional \emph{patches} induced by the sparse connectivity, and thus on the backward side, the gradient dynamics couple to the geometry of the patch collection rather than the ambient point cloud. Therefore, our central question is \emph{how it changes the set of solutions that satisfy the BEoS condition, and in turn, the generalization below the edge of stability.}


To isolate this effect, we study a minimal two-layer ReLU network where each hidden neuron receives input from a prescribed set of coordinate projections, and weights are shared across patches. This architecture captures sparse connectivity while remaining analytically tractable.

\noindent\textbf{Contributions.}
We develop a stability-based theory for sparse connectivity that connects architecture, data geometry, and the generalization below the edge of stability.

\begin{enumerate}
    \item We show that for a sparsely connected two-layer network, the BEoS condition translates into a stepsize dependent bound on a weighted path norm of the network. The weight function in this norm is governed by the geometry of the induced \emph{patch multiset}, namely the collection of all local patches extracted from the training data. This extends stability-based analyses of FCNs and makes precise that sparse patch extraction shifts the relevant data dependence from the full input vectors to the patches themselves.

    \item From this characterization, we prove that on uniformly distributed hyperspherical data, the BEoS-constrained class admits non-vacuous generalization bounds. For fixed patch size $m$ and large ambient dimension $d$, our bound scales as $n^{-1/6+O(m/d)}$ (Theorem~\ref{thm:upper_bound}) and is immune to the curse of dimensionality, which we also verify numerically. This demonstrates that sparse connectivity can yield strong generalization guarantees by changing the representation format the network processes.

    \item We also characterize a limitation: without distributional assumptions, there exist datasets where a sparse network can perfectly interpolate while satisfying the BEoS condition (Theorem~\ref{thm:flat_lcws_gap}). This occurs precisely when individual patches can be isolated from the rest by a single half-space. The result clarifies that the generalization benefit of sparse connectivity relies on a structural property of the patch data, namely that it resists such isolation.

    \item We empirically analyze the patch geometry of natural images (CIFAR-10) and show that standard convolutional receptive fields produce patch geometry that is highly concentrated and resistant to the isolating cuts described above (Section~\ref{sec:patch_geometry_prior}). Together, these results provide a systematic explanation for why convolutional networks generalize better than their fully connected counterparts.
\end{enumerate}

\noindent
In summary, our work shows that the inductive bias of CNNs can be understood through a single geometric lens: sparse connectivity shifts the effective constraint imposed by GD's stability condition from the high-dimensional ambient space to the lower-dimensional patch space, where it naturally favors generalizable features.

\noindent{\textbf{Technical novelty}.} The core technical advance is a stability analysis in which the architecture enters as an explicit variable in the constraint that governs generalization. Under sparse connectivity, each neuron receives only the patches within its receptive field, so its contribution to the loss Hessian is a function of those patches rather than the full input vectors. We show that this sparse connectivity couples each neuron directly to the geometry of the full patch multiset, producing an architecture-dependent regularizer. This architectural dependence is absent from prior work in the same regime \citep{mulayoff2021implicit,nacson2022implicit,qiao2024stable,liang2025_neural_shattering,liang2025datageometry}.

\noindent{\textbf{Scope of analysis}.} 
We analyze the sharpness constraint that gradient descent with a finite step size must satisfy to remain stable.
To make this constraint analytically tractable, we work with two-layer ReLU networks. In particular, we focus on the design of sparse connectivity, which is a  generic structure appearing in nearly every block of modern vision backbones (e.g. CNNs and ViTs).
To expose its role without confounding extra architectural assumptions, we pair it with idealized data distributions that create a sharp contrast while keeping the analysis feasible.
Extending this analysis to deep networks requires understanding how stability constraints propagate through successive layers. In that setting, each layer's output forms the input geometry for the next, and the geometry that sparse connectivity processes is no longer static but an evolving representation, shaped by attention, normalization, and residual connections during training.
How the gradient dynamics couples to this evolving representational hierarchy requires substantial hard work, which we leave as a future work.
\vspace{-2em}
\section{Related Work}
\vspace{-0.5em}
\noindent\textbf{Patch-based representation learning.}
Early unsupervised image representation methods decomposed images into local patches and learned sparse dictionaries or manifolds from these primitives~\cite{aharon2006k, peyre2009manifold}.
Patch based representations have also shown success in unsupervised learning~\cite{paulin2017convolutional} and as preprocessing for image classification~\cite{coates2011analysis, thiry2021unreasonable, brutzkus2022efficient}.
These results highlight that patch space is highly structured, but it remains open how neural network training exploits such structure.
Our work helps bridge this gap by analyzing generalization below the edge of stability.

\noindent\textbf{Geometric awareness in deep learning.}
Recent theoretical work has explored how architectural design incorporates geometric awareness to improve neural network training and approximation. Approximation theory research has demonstrated that CNN architectures overcome the curse of dimensionality by exploiting compositional structure in natural images~\cite{mhaskar2017and, zhou2020universality, poggio2022foundations, poggio2024compositional}, with~\citet{mao2021theory} formally proving the superiority of CNNs over FCNs for learning certain composite functions.
More recently, Vision Transformers~\cite{dosovitskiy2021image} have been analyzed through the same lens: \citet{shi2025approximation} extend these approximation guarantees to ViTs, showing they outperform FCNs for compositional functions, and \citet{trockman2022patches} identify patch extraction, rather than attention, as the critical component, further reinforcing the strength of sparsely connected architectures.
These results, however, either assume the labeling function satisfies a compositional structure or operate in settings with explicit regularization.
Our work requires no assumptions on the labeling function and operates in the overparameterized regime.

\noindent\textbf{Theoretical analysis of CNNs and separation from FCNs.}
The advantages of CNNs over FCNs have been established from both approximation-theoretic and statistical perspectives.
Approximation theory shows that CNNs with appropriate sparse weights can achieve near-minimax optimal sample complexity~\cite{oono2019approximation} and adapt to intrinsic dimension when data lie on a low-dimensional manifold~\cite{liu2021besov}, with \citet{zhang2024nonparametric} extending these results to overparameterized regimes under weight decay.
However, these analyses rely on reducing CNNs to fully connected networks~\cite{yarotsky2017error} and therefore do not reveal the architectural insights specific to CNNs that we describe in this paper.
Specifically, we do not require the data to be supported on a low-dimensional manifold to avoid the curse of dimensionality, nor do we require explicit regularization such as weight decay or a sparsity constraint.
On the statistical side, a sample complexity separating CNNs, locally connected networks without weight sharing, and FCNs has been established~\citep{li2021convolutional,wang2023theoretical,lahoti2024role}.
Our results are for the same model family, but differ in that we consider an overparameterized regime without explicit regularization, which allows us to prove a stronger separation between FCNs and CNNs.
Lastly, the significance of the input data distribution, rather than the labeling function, was not discovered in prior work.

\noindent\textbf{Implicit bias of gradient descent.}
Many existing work on the implicit bias of gradient descent relies on strong assumptions on the data distributions (e.g., linearly separable data), simplified model architectures (e.g., linear activation) to make a gradient dynamics analysis tractable \citep{soudry2018implicit,gunasekar2018characterizing,gunasekar2018implicit}, or weight initialization schemes that keep the model in the kernel regimes \citep{jacot2018neural,arora2019exact}.
While CNNs were studied \citep{gunasekar2018implicit,arora2019exact}, the nature of the results are different from ours.

\noindent\textbf{Edge of stability, minima stability, and generalization by large stepsizes.}
Our approach builds upon a recent line of work that studies the set of solutions that gradient descent can visit (or converge to) via either Edge-of-stability observation or the minima stability theory \citep{ding2024flat,mulayoff2021implicit,nacson2022implicit, wu2023implicit, qiao2024stable,liang2025_neural_shattering,liang2025datageometry}. These approaches enable formal analysis of the generalization properties without having to analyze gradient dynamics. To the best of our knowledge, they all focused on feedforward neural networks, and we are the first to study CNNs and the impact of model architecture choices in the implicit regularization of large stepsizes.

\vspace{-0.5em}
\section{Preliminaries and Notations}\label{sec:prelim}
\vspace{-0.5em}

Throughout the paper, we use $O(\cdot)$ and $\Omega(\cdot)$ to absorb constants, while $\tilde{O}(\cdot)$ absorbs logarithmic factors. $\Sph^{d-1}$ denotes the unit hypersphere in $\Rb^d$ and $\Bb_R^d$ denotes a $d$-dimensional ball of radius $R$. The activation function is the ReLU denoted by $\phi(z) = \max\{0, z\}$. We write $[n] = \{1, 2, \dots, n\}$.

\noindent\textbf{Architectures and sparse connectivity.}
Fix integers $d \ge 1$ and $1 \le m \le d$.
A \emph{sparse connectivity pattern} is specified by a collection of coordinate subsets (receptive fields) $\CS = \{S_j\}_{j=1}^J$, where each $S_j \subset [d]$ has size $m$. Each subset induces a coordinate projection $\pi_j: \Rb^d \to \Rb^m$, and we call $\vec{x}^{(S_j)} := \pi_j(\vec{x})$ a \emph{patch}. This formalizes the property that each hidden neuron only receives input from a restricted subset of the ambient coordinates.

We analyze two-layer sparsely connected ReLU networks with weight sharing and Global Average Pooling (GAP)\footnote{Weight sharing and GAP are non-essential and chosen only for cleaner exposition. Appendix~\ref{app:sparse_no_share} proves an analogous result for SCNs without weight sharing, where the global patch-multiset weight $g_{\CD,\CS}$ is replaced by location wise weights.}. Given receptive fields $\CS$ and width $K \in \Nb$, we consider
\begin{equation}\label{eq:model_CNN}
f_{\vec{\theta}}(\vec{x}) = \sum_{k=1}^{K} v_k \left( \frac{1}{J} \sum_{j=1}^{J} \phi\bigl(\vec{w}_k^\T \pi_j(\vec{x}) - b_k\bigr) \right) + \beta.
\end{equation}
We refer to this architecture as a \emph{sparsely connected network} (SCN) or $\CS$CN if the sparse connectivity pattern is prescribed. 
A unit $\phi(\vec{w}_k^\T \cdot - b_k)$ is called a \emph{filter} or \emph{neuron}. 
It is said to be \emph{activated} on a patch $\vec{x}^{(S_j)}$ if  $\vec{w}_k^{\T} \vec{x}^{(S_j)} > b_k$. Let $\vec{\Theta}^{\CS}$ be the parameter set of all such $\CS$CNs of \emph{any finite} width.

\noindent\textbf{Data and loss function.}
A dataset is $\CD = \{(\vec{x}_i, y_i)\}_{i=1}^n$ with $\vec{x}_i \in \Bb_R^d$ and $y_i \in [-D,D]$. We use the squared loss and define the training objective
$\loss(\vec{\theta}) \coloneqq \frac{1}{2n} \sum_{i=1}^n \bigl( f_{\vec{\theta}}(\vec{x}_i) - y_i \bigr)^2$.

\noindent\textbf{``Edge of Stability'' regime.}
Empirical and theoretical work \citep{cohen2021gradient,damian2023self} has shown that gradient descent on neural networks typically enters the Edge of Stability (EoS) regime: the largest Hessian eigenvalue $\lambda_{\max}(\nabla^2 \loss(\vec{\theta}_t))$ steadily increases (``progressive sharpening") until it hovers around the value $2/\eta$, where $\eta>0$ is the step size. Throughout this paper, GD refers to vanilla gradient descent with learning rate $\eta$, and we assume $\eta < 2$.

\begin{definition}[Below Edge of Stability {\citep[Definition 2.3]{qiao2024stable}}]\label{def:BEoS}
Let $\{\vec{\theta}_t\}_{t \ge 1}$ be a GD trajectory on $\loss$ with step size $\eta$. 
Any parameter $\vec{\theta}_t$ with $\lambda_{\max}\!\bigl(\nabla^2 \loss(\vec{\theta}_t)\bigr) \le \frac{2}{\eta}$ is called a solution \emph{Below the Edge of Stability (BEoS)}, or simply a \emph{BEoS solution}.
\end{definition}

This condition applies to any twice-differentiable solution produced by GD, even when the loss does not fully converge. For our analysis, it is convenient to study the set of all parameters that satisfy this condition, whether or not they lie on a specific GD trajectory. We define
\begin{equation}\label{eq:BEoS_class}
\BEoS^{\CS}(\eta, \CD) \coloneqq \left\{
\vec{\theta} \in \vec{\Theta}^{\CS}
\middle|\lambda_{\max}\bigl(\nabla^2 \loss(\vec{\theta})\bigr) \le \frac{2}{\eta}
\right\}.
\end{equation}

\noindent\textbf{Statistical learning framework.}
We assume the data points $\{(\vec{x}_i, y_i)\}_{i=1}^n$ are i.i.d.~samples from a distribution $\CP$ on $\Bb_R^d \times [-D, D]$. The \emph{population risk} of a predictor $f$ is $ \Risk(f) := \Eb_{(\vec{x}, y) \sim \CP}[(f(\vec{x}) - y)^2]$. The \emph{empirical risk} on dataset $\CD$ is $\widehat{\Risk}_{\CD}(f) = \frac{1}{n} \sum_{i=1}^n (f(\vec{x}_i) - y_i)^2$. The \emph{generalization gap} is defined as $\mathrm{GenGap}(f, \CD) := |\Risk(f) - \widehat{\Risk}_{\CD}(f)|$.
\section{Main Results}\label{sec:main_results}
\vspace{-0.5em}
To start with, we use a toy example to preview how sparse connectivity in the network shapes the gradient dynamics itself. Consider an input $z\sim\mathrm{Uniform}(\mathbb{S}^{2})$ embedded into $\mathbb{R}^{9}$ by padding each coordinate with two zeros. The ambient data is spherical, so a fully connected network sees a high-dimensional unstructured geometry. Now equip a two-layer SCN with two different receptive field systems:
\textbf{(A)} $\CS_1$ (block-aligned): each receptive field groups 3 consecutive coordinates. Every patch has the form $(t,0,0)$, so the patch geometry is one-dimensional; \textbf{(B)} $\CS_2$ (interleaved): each receptive field groups coordinates modulo the stride. A patch recovers the full vector $(z_1,z_2,z_3)$, so the patch geometry is spherical.
\begin{figure}[t]
\centering
\vspace{-4em}
\includegraphics[width=0.85\linewidth]{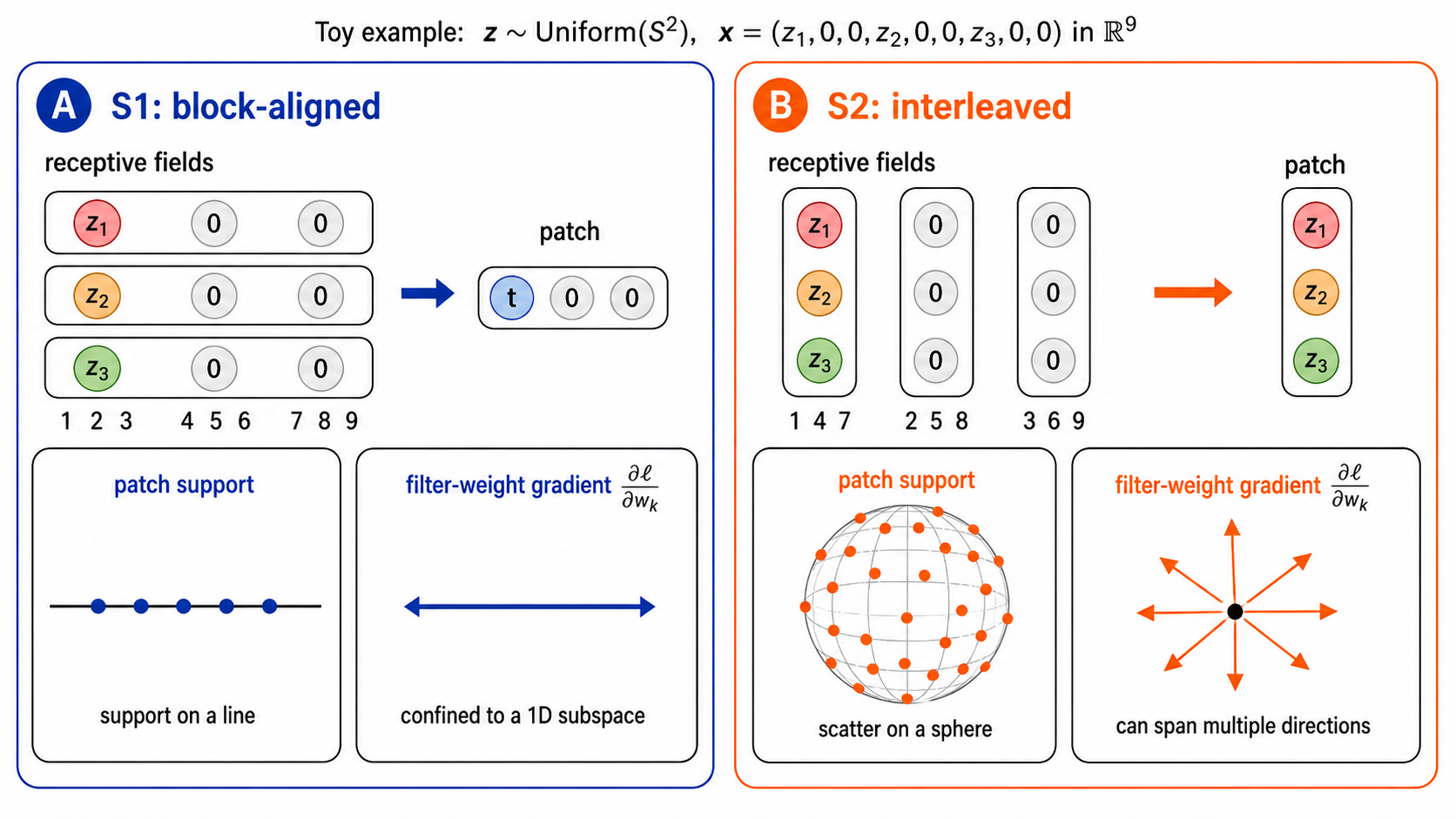}
\caption{\textbf{Data $\times$ architecture determines the geometry that gradient descent interacts with.}}
\label{fig:toy_models_2_receptive_fields}
\vspace{-1em}
\end{figure}
The effect of these two designs becomes visible in the gradient. For the model \eqref{eq:model_CNN}, the gradient of the training loss with respect to a filter $\vec{w}_k$ is
\begin{equation}\label{eq:filter_gradient}
\frac{\partial \loss(\vec{\theta})}{\partial \vec{w}_k}
= \frac{v_k}{nJ}\sum_{i=1}^n 
\sum_{j=1}^J \underbrace{(f_{\vec{\theta}}(\vec{x}_i)-y_i)\mathds{1}\{\vec{w}_k^\T \pi_j(\vec{x}_i) > b_k\}}_{\text{residue and activation scalars}}\,
\textcolor{red}{\underbrace{\pi_j(\vec{x}_i)}_{\text{patch vectors}}}.
\end{equation}
The update to $\vec{w}_k$ is a linear combination of the patches $\textcolor{red}{\pi_j(\vec{x}_i)}$ that activate it. Under $\CS_1$, every activated patch lies in a one-dimensional subspace, so each filter gradient is confined to that subspace. Under $\CS_2$, activated patches span the entire 3-dimensional patch space, and filter gradients can explore all directions. The two receptive field systems thus produce fundamentally different gradient dynamics, and the sets of solutions that GD can stably reach differ dramatically.

The rest of this section develops a framework that makes this intuition precise via the lens of dynamical stability analysis in gradient descent (the BEoS regime). Theorem~\ref{thm:implicit_bias_CNN} shows that the BEoS condition bounds a weighted path norm whose weight is governed by the patch geometry. Theorem~\ref{thm:upper_bound} and Theorem~\ref{thm:flat_lcws_gap} then develop the two contrasting regimes previewed by $\CS_1$ and $\CS_2$: one enables generalization, and the other fails.

\subsection{From the BEoS condition to a patch-geometry-weighted path norm}\label{subsec:beos_to_pathnorm}

Given a dataset $\CD=\{(\vec{x}_i,y_i)\}_{i=1}^n$ and receptive fields $\textcolor{red}{\CS}=\{S_j\}_{j=1}^J$, let $\vec{X}^{\textcolor{red}{\CS}}_{\CD}$ be a random vector drawn uniformly from the patch multiset $\{\pi_j(\vec{x}_i)\}_{(i,j)\in[n]\times[J]}\subset \Rb^m$. Define the weight function $g_{\CD,\textcolor{red}{\CS}}\!:\Sph^{m-1}\times \Rb\rightarrow \Rb$ by 
\begin{align}\label{eq:weight_g_CNN}
g_{\CD,\textcolor{red}{\mathcal{S}}}(\vec{u},t):=\Eb\left[\phi\left(\vec{u}^\T\vec{X}^{\textcolor{red}{\CS}}_{\CD} - t\right)\right]\,
	\sqrt{\Pb\left(\vec{u}^\T\vec{X}^{\textcolor{red}{\CS}}_{\CD} > t \right)^{2}+\left\|\Eb\left[\vec{X}^{\textcolor{red}{\CS}}_{\CD}\,\mathds{1}\Big\{\vec{u}^\T\vec{X}^{\textcolor{red}{\CS}}_{\CD} > t\Big\}\right]\right\|_2^{2}}.
\end{align}

\begin{theorem}\label{thm:implicit_bias_CNN}
Fix $\CD$ and $\textcolor{red}{\CS}$. For any $\vec{\theta}\in\vec{\Theta}^{\textcolor{red}{\CS}}$,
\begin{equation}\label{eq:beos_weighted_pathnorm}
\begin{aligned}
\sum_{k=1}^K |v_k|\,\|\vec{w}_k\|\;
g_{\CD,\textcolor{red}{\CS}}\!\left(
\frac{\vec{w}_k}{\|\vec{w}_k\|},\,
\frac{b_k}{\|\vec{w}_k\|}
\right)
\le\frac{1}{2}\Big(
\lambda_{\max}\!\bigl(\nabla^2\loss(\vec{\theta})\bigr)
+2(R+1)\sqrt{2\loss(\vec{\theta})}-1
\Big).
\end{aligned}
\end{equation}
In particular, for any $\vec{\theta}\in \BEoS^{\textcolor{red}{\CS}}(\eta,\CD)$,
\begin{equation}\label{eq:beos_weighted_pathnorm2}
\begin{aligned}
\sum_{k=1}^K |v_k|\,\|\vec{w}_k\|\;
g_{\CD,\textcolor{red}{\CS}}\!\left(
\frac{\vec{w}_k}{\|\vec{w}_k\|},\,
\frac{b_k}{\|\vec{w}_k\|}
\right)
\le\frac{1}{\eta}-\frac{1}{2}+(R+1)\sqrt{2\loss(\vec{\theta})}.
\end{aligned}
\end{equation}
\end{theorem}

The proof is in Appendix~\ref{app:dd_regularity}. In particular, the connection between the Hessian and $g_{\CD,\CS}$ is detailed in Proposition~\ref{prop:CNN_weight_function_population}. Intuitively, the weight function $g_{\CD,\CS}$ evaluates a neuron's activation boundary $(\vec{u},t)$ against the patch multiset. Crucially, the architecture $\CS$ enters solely in this way: different receptive field systems produce different patch multisets, hence different penalty structures. In particular, the term $\Pb(\vec{u}^\T\vec{X}^{\CS}_{\CD} > t)^2$ in $g_{\CD,\CS}$ suggests that a neuron that fires on a large fraction of patches incurs a large penalty on its path norm, while one that fires on few patches incurs a weak penalty. Returning to our example, under $\CS_1$ the patch multiset is essentially one-dimensional and highly concentrated; a direct calculation shows that $g_{\CD,\CS}$ is bounded away from zero for any neuron that activates on a non-negligible fraction of patches, so the BEoS condition enforces a meaningful path-norm constraint. Under $\CS_2$, the patch multiset is a 3-dimensional sphere; a neuron can isolate an individual patch, driving $g_{\CD,\CS}$ arbitrarily close to zero and rendering the BEoS constraint vacuous. The next two subsections develop the quantitative consequences of this dichotomy.

\vspace{-0.5em}
\subsection{When patch geometry induces strong regularization}\label{subsec:cnn_sphere}
\vspace{-0.5em}
The next result formalizes the $\CS_1$ regime: small receptive fields on spherical ambient data.

\begin{theorem}\label{thm:upper_bound}
Let $\CP$ have marginal $\mathrm{Uniform}(\Sph^{d-1})$, and let $\CD$ be a dataset of $n$ i.i.d samples. Assume $d>3$ and $1\le m<\frac{d(d-3)}{d+3}$. For any $\vec{\theta}\in\BEoS^{\CS}(\eta,\CD)$ with $\|f_{\vec{\theta}}\|_{\infty}\leq M$, with probability $\geq 1 - 2\delta$,
\begin{equation}\label{ineq:cnn_epsilon_trade_off_uniform}
\mathrm{GenGap}(f_{\vec{\theta}},\CD)
\lessapprox_d \poly\left(\frac{1}{\eta},J,M\right)\,
n^{-\frac{(d-m)(d+3)}{6d^2-2md+6d-6m}},
\end{equation}
where  $\lessapprox_d$ hides constants depending on $d$ and logarithmic factors in $n$ and $(J/\delta)$.
\end{theorem}

The formal proof appears in Appendix~\ref{app:proof_upper_bound}. Notably, when $m$ is fixed and $d\to\infty$, the exponent rate scales as $-\frac{1}{6}+O(m/d)$, and thus there is no curse of dimensionality. In the same spherical setting, FCNs admit no non-trivial generalization guarantee under the BEoS condition~\citep{liang2025datageometry}. The geometric reason is that when $m\ll d$, the patch projections concentrate near the origin of $\Rb^m$ (see Proposition~\ref{prop:boundary_tail_for_projected}), which makes the patch multiset harder to shatter. A neuron's activation boundary can isolate/shatter only a small fraction of the patches without cutting through the dense interior, so $\Pb(\vec{u}^\T\vec{X}^{\CS}_{\CD} > t)$ and hence $g_{\CD,\CS}$ are bounded away from zero for any neuron that reaches a meaningful part of the distribution. The path-norm penalty in Theorem~\ref{thm:implicit_bias_CNN}therefore becomes effective. Figure~\ref{fig:gen_gap_two_plots} corroborates this prediction: the empirical generalization gap decays with $n$ at rates consistent with the theorem, and the gap decreases as $d$ grows for fixed patch size $m$. 
\begin{figure}[h]
    \centering
    \begin{subfigure}[t]{0.45\linewidth}
        \centering
        \includegraphics[width=\linewidth]{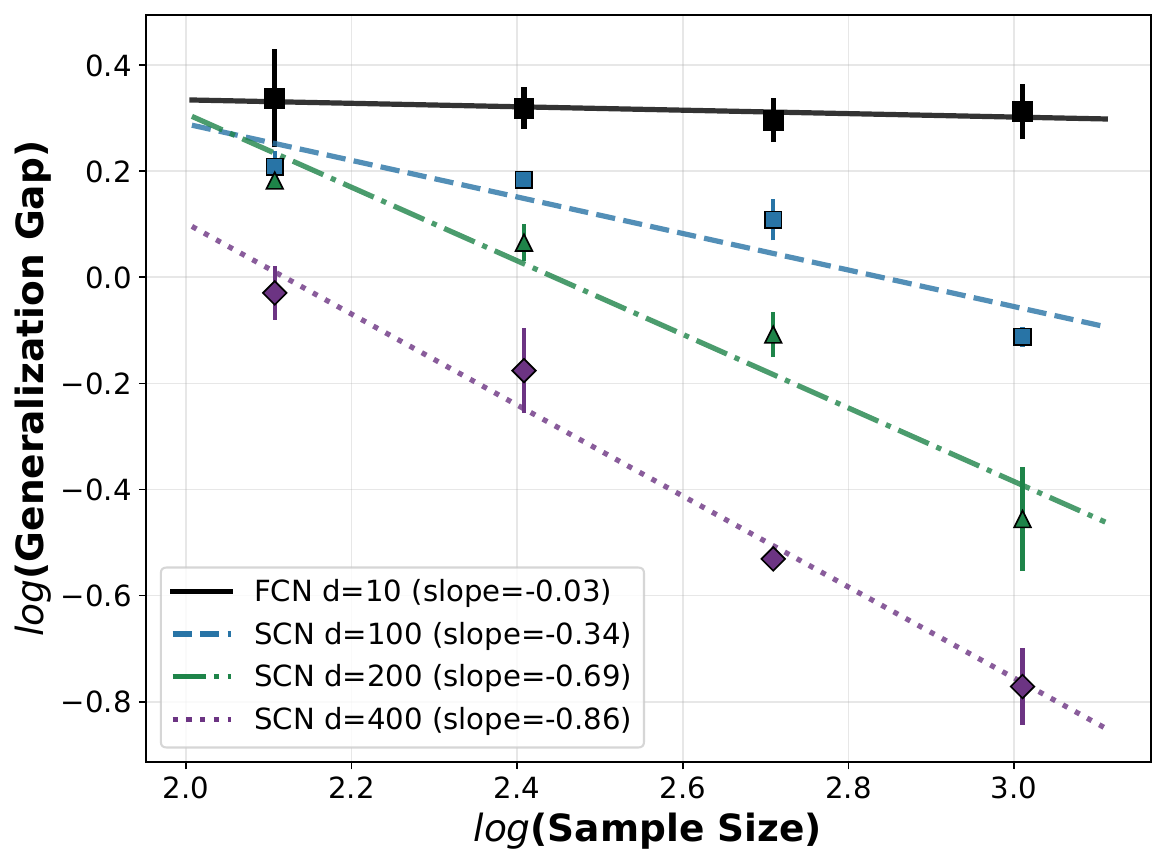}
        \caption{Log--log scaling with $n$.}
        \label{fig:gap_vs_n_loglog}
    \end{subfigure}
    \hfill
    \begin{subfigure}[t]{0.45\linewidth}
        \centering
        \includegraphics[width=\linewidth]{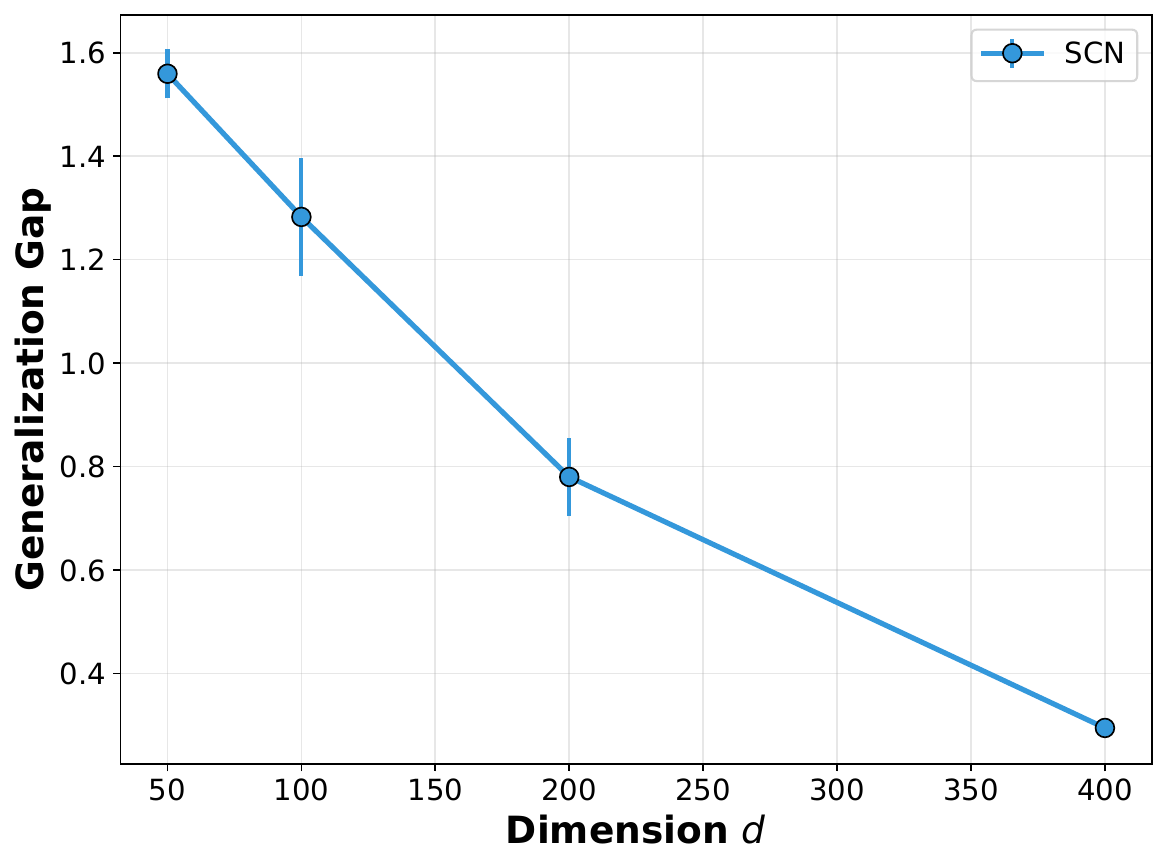}
        \caption{GenGap versus $d$ (fixed $n$).}
        \label{fig:gap_vs_d}
    \end{subfigure}
    \caption{\textbf{Generalization-gap scaling in synthetic experiments.} \textbf{(Left)} $\mathrm{GenGap}(f_{\vec{\theta}},\CD)$ versus the sample size $n$ on a log--log scale. The fitted slope summarizes the empirical rate: if $\mathrm{GenGap}\lesssim n^{-c}$, then $\log(\mathrm{GenGap})\le -c\log n + b$, so a more negative slope indicates faster decay (better generalization). In our experiments, the FCN slope is nearly flat (slope $=-0.03$ at $d=10$), whereas SCN exhibits increasingly negative slopes as $d$ grows (slope $=-0.34$ at $d=100$, $-0.69$ at $d=200$, and $-0.86$ at $d=400$), indicating faster decay with $n$. \textbf{(Right)} $\mathrm{GenGap}(f_{\vec{\theta}},\CD)$ versus the ambient dimension $d$ with $n$ fixed, illustrating that for SCN (with patch size $m\ll d$) the generalization gap remains stable and can even decrease as $d$ increases. More details can be found in Appendix \ref{app:experiment_details.}.}
    \vspace{-1em}
    \label{fig:gen_gap_two_plots}
\end{figure}

\vspace{-1em}
\subsection{When patch geometry fails to constrain}\label{subsec:cnn_overfit}

The previous generalization guaranty relies on the fact that low-dimensional patch projections make the patch multiset concentrated and thus hard to shatter by the activation hyperplanes of the filters in the network. The next result shows that if the sparse connectivity pattern cannot yield a favorable patch geometry, the BEoS condition alone cannot prevent SCNs from overfitting.

\begin{theorem}[Interpolation below the edge of stability]\label{thm:flat_lcws_gap}
Assume all training patches have norm at most one (after rescaling), and every nonzero-labeled training example contains a unit-norm patch that appears nowhere else in the full training patch multiset. There exists a network of the form \eqref{eq:model_CNN} with width $K\le n$ that interpolates $\CD$ and satisfies
\begin{equation}\label{eq:H_op_bound_lcws_gap}
\lambda_{\max}\!\bigl(\nabla_{\vec{\theta}}^2\loss\bigr)
\;\le\;
1+\frac{D^2+2/J^2}{n}.
\end{equation}
\end{theorem}

Theorem~\ref{thm:flat_lcws_gap} shows that the BEoS condition alone does not prevent memorization when the patch multiset admits single-patch isolation. 
If each nonzero-labeled example contains a patch that one filter can isolate from all other training patches, the network can use rarely activated filters to fit the labels. 
Because each such filter fires on only one out of the $nJ$ training patches, its tangent feature is localized, so the Hessian can remain small even at interpolation, yielding \eqref{eq:H_op_bound_lcws_gap}. 
From the weighted path norm perspective of Theorem~\ref{thm:implicit_bias_CNN}, the same rare activations correspond to small values of $g_{\CD,\CS}$, making the stability-induced constraint weak. The detailed proof can be found in Appendix \ref{app:flat_interp_lcws_gap}.

An empirical demonstration is shown in Figure~\ref{fig:flat_interpolation}. 
The experiment uses random spherical patches (case B in Figure \ref{fig:toy_models_2_receptive_fields}). Since these patches are uniformly sampled from  $\Sph^{m-1}$, the uniqueness condition in Theorem~\ref{thm:flat_lcws_gap} holds generically. 
The SCN rapidly interpolates labels while the sharpness hovers near $2/\eta$, showing that memorization is compatible with the EoS regime. 
The path-norm--activation scatter provides a complementary view: many neurons have large path norms but activate on only a small fraction of patches, precisely where the weighted path-norm constraint is weak.

\begin{figure}[h!]
\centering
\includegraphics[width=\linewidth]{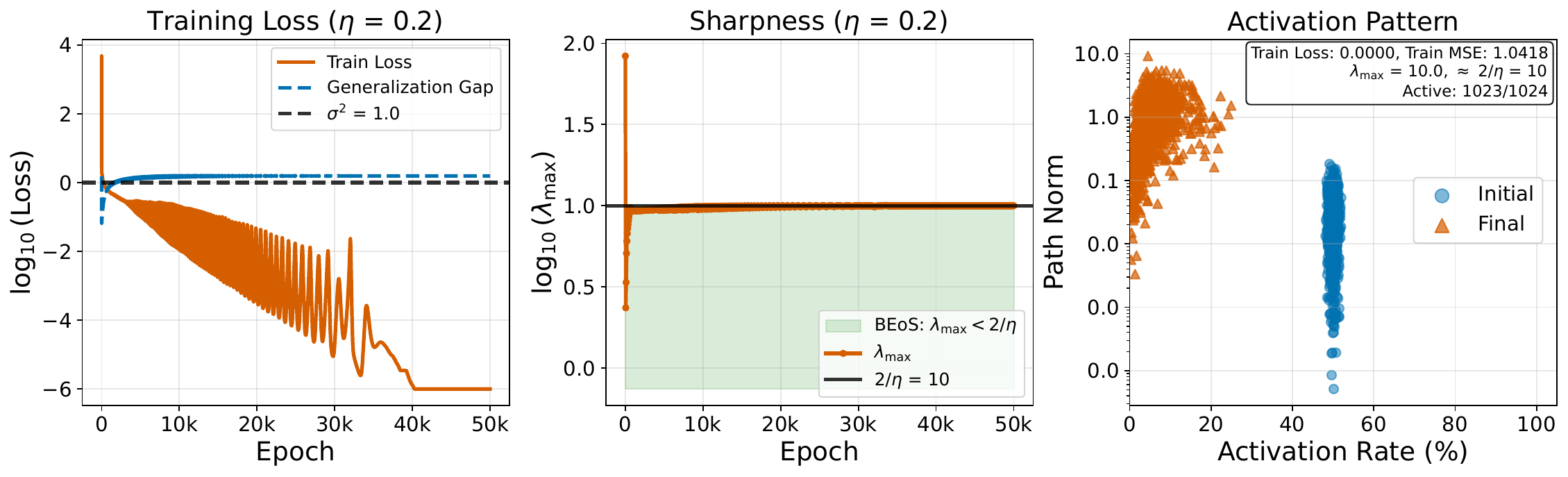}
    \caption{\textbf{Interpolation below the edge of stability.}
    \textbf{Left.} Training loss and generalization gap, showing that the SCN interpolates noisy labels in the random spherical-patch setting.
    \textbf{Middle.} The largest Hessian eigenvalue hovers near $2/\eta \approx 10$, indicating compatibility with the EoS regime.
    \textbf{Right.} Scatter plot of per-neuron path norm $|v_k|\,\|\vec{w}_k\|$ against activation rate
    $\frac{1}{nJ}\sum_{i,j}\mathds{1}\{\vec{w}_k^\T\pi_j(\vec{x}_i)>b_k\}$.
    Many neurons have large path norms while activating on only a small fraction of the patch multiset, consistent with the failure mode in Theorem~\ref{thm:flat_lcws_gap}: rare activations can fit labels while keeping the sharpness small.}
    \label{fig:flat_interpolation}
\end{figure}

Taken together, Theorem~\ref{thm:upper_bound} and Theorem~\ref{thm:flat_lcws_gap} formalize two regimes governed by the patch geometry induced by the receptive-field system $\CS$. 
When $\CS$ and the data produce a concentrated patch multiset that resists single-patch isolation, the BEoS condition yields an effective regularity constraint. 
When the induced patch multiset contains patches that can be isolated from the rest by ReLU half-spaces, the weighted constraint can be too weak to rule out memorization. 
For a fixed training dataset, the architectural design $\CS$ determines the induced patch multiset and hence which regime the network operates in.
\vspace{-1em}
\section{Patch Geometry of Natural Images}\label{sec:patch_geometry_prior}
\vspace{-0.5em}
We now examine which side of the dichotomy established in Section~\ref{sec:main_results} real image data falls on. We extract $3\times 3$ convolutional patches (stride $1$, no padding) from the CIFAR-10 training set, sample $10^7$ patches, and apply standard torchvision normalization. Figure~\ref{fig:patch_vs_image_geometry} compares the patch cloud with the cloud of full images. The patch cloud is strongly concentrated: $90\%$ of its variance lives in just $3$ directions of $\mathbb{R}^{27}$, while the full images in $\mathbb{R}^{3072}$ need over $100$ directions. Recalling the gradient view~\eqref{eq:filter_gradient}, where filter updates are linear combinations of activated patch vectors, such a low-dimensional patch cloud implies that the gradients themselves are constrained to a low-rank subspace---much like the one-dimensional $\CS_1$ example in Section~\ref{sec:main_results}, where filter updates were confined to a single direction. Following \citet{liang2025datageometry}, we further assess shatterability via half-space depth. The right panel reports the concentration curve $\Psi(T)$, the fraction of patches whose half-space depth exceeds $T$. A larger area under $\Psi$ indicates that typical patches lie deeper in the distribution, making them harder to isolate with a single hyperplane. The patch cloud exhibits substantially larger area than the image cloud.

\begin{figure}[h]
    \centering
    \includegraphics[width=0.85\linewidth]{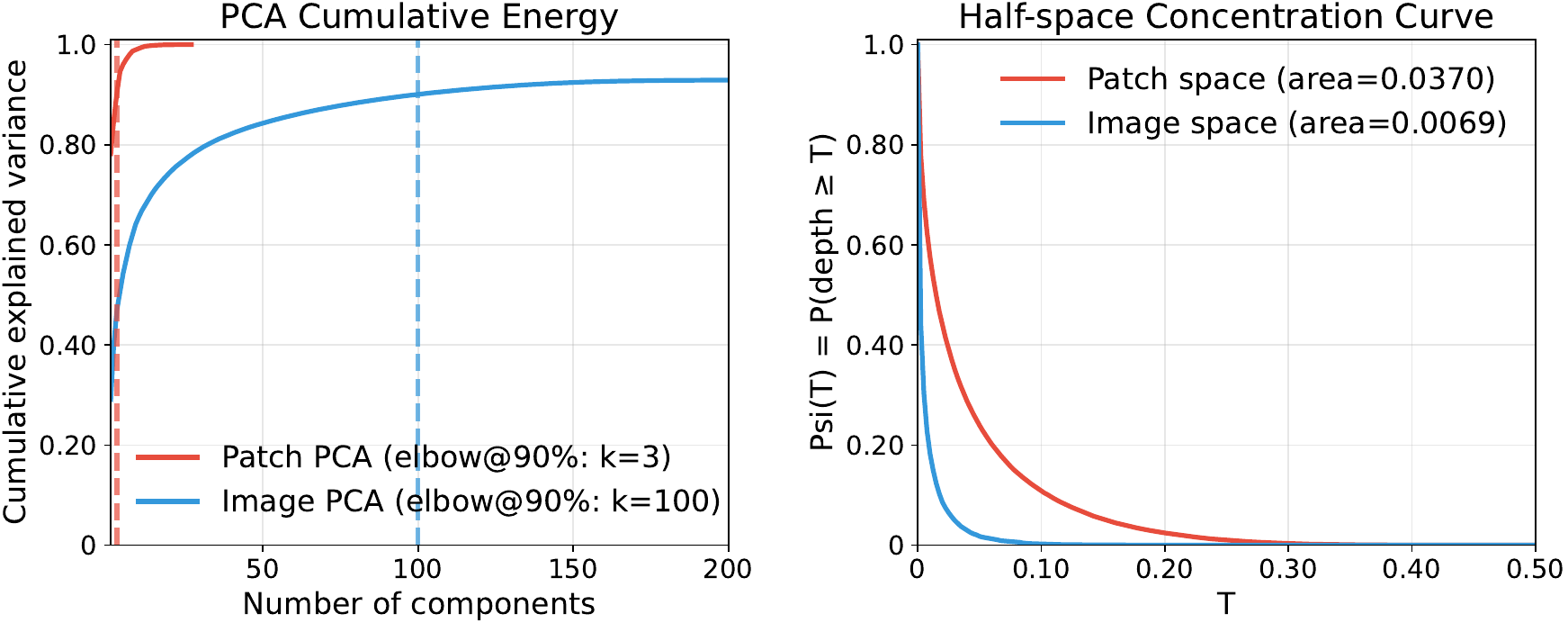}
    \caption{\textbf{Patch geometry vs.\ image geometry on CIFAR-10.}
    \textbf{Left:} PCA explained-variance ratio for the patch cloud ($\mathbb{R}^{27}$, solid) and the ambient image cloud ($\mathbb{R}^{3072}$, dashed). Patches concentrate $90\%$ variance in $3$ principal directions; images need over $100$.
    \textbf{Right:} Half-space concentration curves $\Psi(T)=\mathbb{P}(\operatorname{depth}(\mathbf{X})\ge T)$ for $T\in[0,0.5]$. Larger area means deeper points and fewer opportunities for near-isolating hyperplanes. The patch cloud shows markedly larger concentration.}
    \label{fig:patch_vs_image_geometry}
\end{figure}
\begin{figure}[h!]
    \centering
    \includegraphics[width=\linewidth]{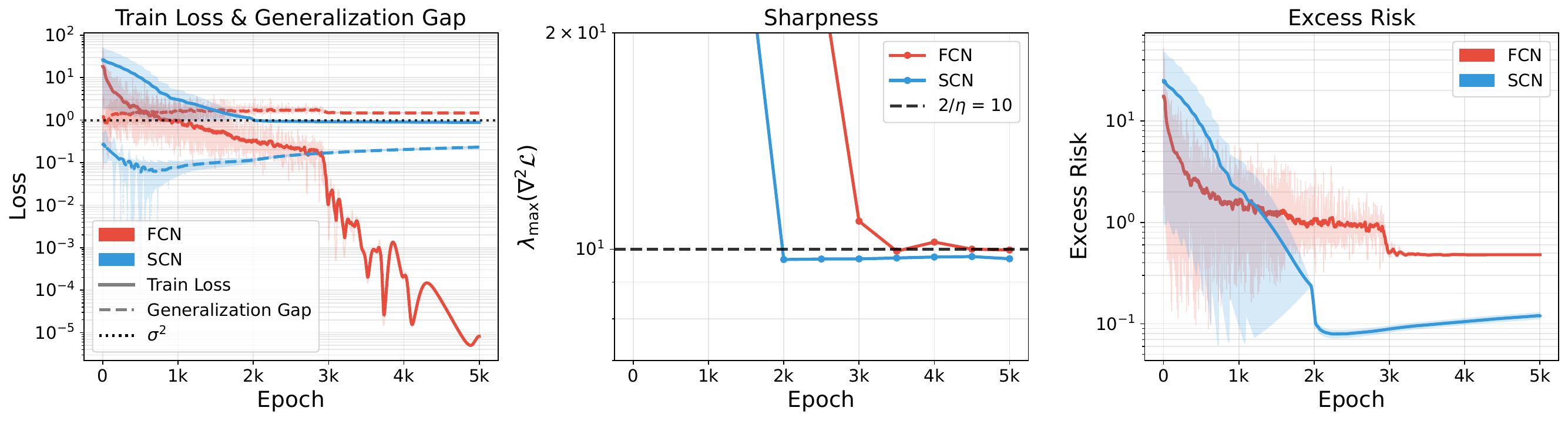}
    \caption{\textbf{Generalization on CIFAR-10 with noisy labels.} Two-layer SCN ($3\times3$ convolution, stride $1$, no padding) vs.\ FCN of the same width, both trained with full-batch GD ($\eta=0.2$).
    \textbf{Left:} Training loss (solid) and generalization gap (dashed) over epochs. The SCN stabilizes near the noise floor; the FCN interpolates.
    \textbf{Middle:} Sharpness $\lambda_{\max}(\nabla^2\mathcal{L})$. Both networks hover around $2/\eta\approx10$ after $\sim$3k epochs (before that the loss of FCN has not started to converge), confirming the BEoS regime.
    \textbf{Right:} Excess risk over training epochs. After $2$k epochs the SCN stabilizes at a much lower excess risk than the FCN.}
    \label{fig:cifar10_regression}
\end{figure}
These geometric properties place convolutional patches on CIFAR-10 on the structured side of the dichotomy. We verify the predicted generalization behavior with a nonparametric regression experiment on CIFAR-10 with noisy labels, comparing a two-layer SCN ($3\times3$ convolution, stride $1$, no padding) and a fully connected network of the same width. Figure~\ref{fig:cifar10_regression} reports the results. The SCN resists overfitting while the FCN interpolates noise (left panel). Both networks enter the BEoS regime after roughly $3000$ epochs, with the largest Hessian eigenvalue hovering around $2/\eta$ (middle panel). The right panel tracks excess risk throughout training: after roughly 2k epochs, the SCN maintains a substantially lower excess risk than the FCN. The contrast follows from the geometric measurements: the structured patch cloud makes the BEoS constraint effective for the SCN, whereas the unstructured image cloud provides no such constraint for the FCN.

\vspace{-1em}
\section{Discussion}\label{sec:discussion}
\vspace{-1em}
We set out to understand how sparse connectivity shapes the way gradient descent trained two-layer ReLU networks generalize, viewed through the lens of stability analysis.
We proved that the stability of gradient descent translates into an explicit constraint on the network whose strength is governed by the geometry of the patches that sparse connectivity produces.
We then proved that this constraint yields non-vacuous generalization bounds on high-dimensional spherical data, where the induced patch geometry is concentrated and low-dimensional, and where fully connected networks provably fail. Conversely, when the patch multiset contains patches that individual filters can isolate from the rest, we constructed networks that interpolate arbitrary labels while satisfying the same stability condition.
We further validated this picture on CIFAR-10, confirming that standard convolutional receptive fields produce exactly the kind of patch geometry that makes the constraint effective, giving a principled explanation for why CNNs generalize well.

This geometric perspective suggests a broader lens for thinking about architecture design. 
Sparse connectivity, the structure at the center of our analysis, is equally fundamental in deep networks: it appears in nearly every building block of modern CNNs and Vision Transformers.
In our two-layer setting, sparse connectivity processes patches of the raw input, and we showed that the effective dimension of these patches controls the strength of implicit regularization.
In a deep network, a local operator at an intermediate layer processes not raw pixels but a patch of the representation produced by earlier layers.
The geometry on which the stability constraint acts is therefore a dynamically evolving learned representation rather than a fixed data distribution.
\begin{table}[t]
\vspace{-3em}
\centering
\caption{The design trend of sparse connectivity in modern vision backbones.}
\label{tab:arch_trends}
\small
\begin{tabular}{@{}lll@{}}
\toprule
\textbf{Representative backbone} & \textbf{Core sparse connectivity pattern} & \textbf{Dimension of the processed unit} \\
\midrule
ResNet-50 \citep{he2016deep}         & $3\times3$ bottleneck convolution & $3\times3\times C_{\mathrm{in}}$ (up to $C_{\text{in}}=512$) \\
ResNeXt-50 \citep{xie2017aggregated}        & $3\times3$ grouped convolution   & $3\times3\times C_{\mathrm{in}}/g$ (up to $g=32$) \\
ConvNeXt-T \citep{liu2022convnet}        & $7\times7$ depthwise convolution & $7\times7=49$ \\
\midrule
ViT-B/16  \citep{dosovitskiy2021image}         & uniform Transformer       & $768$ for all blocks \\
Swin-T  \citep{liu2021swin}          & hierarchical Transformer  & $96\rightarrow192\rightarrow384\rightarrow768$ \\
\bottomrule
\end{tabular}
\vspace{-0.5em}
\end{table}

This shift recasts sparse connectivity as more than an efficiency tool. Through its receptive-field configuration, it also shapes the implicit regularization that gradient descent provides, with the vector dimension seen by each local operator as the relevant quantity. Table~\ref{tab:arch_trends} collects this perspective across architectures. From ResNet~\citep{he2016deep} to ResNeXt~\citep{xie2017aggregated} to ConvNeXt~\citep{liu2022convnet}, the per-operator dimension steadily shrinks through bottlenecks, grouping, and depthwise convolutions. ConvNeXt's kernel-size ablation finds that increasing the depthwise kernel from $3\times3$ to $7\times7$ improves performance, with gains saturating at larger sizes. RepLKNet~\citep{ding2022replknet} extends this direction further by showing that very large spatial kernels can be effective when the input to the spatial operator is restricted to $k^2$ rather than $k^2C$ through depthwise design. We suspect that the separation of spatial and channel mixing may itself produce more structured patch geometries, a question we leave for future work.

Vision Transformers exhibit the same pattern. From ViT~\citep{dosovitskiy2021image} to Swin~\citep{liu2021swin}, a uniform high-dimensional token space gives way to a hierarchy whose dimensions grow gradually, controlling the effective local operator dimension while recovering expressive power through later-stage mixing. Recent systematic patch scaling studies show that ViT performance on classification and dense prediction consistently improves as patch size decreases, even to $1\times1$ pixel tokenization~\citep{wang2025patchification}. In addition, self-supervised ViTs such as DINO also highlight the importance of small patches~\citep{caron2021emerging}. In our framework, reducing patch size changes the geometry of the patch matrix itself, potentially exposing the optimizer to more structured local vectors and thus strengthening the implicit regularization. Appendix~\ref{app:patch_size_local_operator_geometry} provides a more extensive discussion of these connections.

Our analysis began with a minimal toy model, yet the mechanism it reveals carries a broader lesson. Sparse connectivity is not merely a design choice for computational efficiency. It actively participates in the gradient dynamics that govern representation learning and implicit regularization. We believe this work is only a starting point. By offering a geometric lens on how architecture shapes optimization, it opens a direction toward principled system-algorithm co-design where architectural choices and training algorithms are understood as interacting through the geometry they jointly expose to gradient descent.





\clearpage

\bibliography{reference}

@article{fukushima1988neocognitron,
  title={Neocognitron: A hierarchical neural network capable of visual pattern recognition},
  author={Fukushima, Kunihiko},
  journal={Neural networks},
  volume={1},
  number={2},
  pages={119--130},
  year={1988},
  publisher={Elsevier}
}

@article{lecun1998gradient,
  title   = {Gradient-Based Learning Applied to Document Recognition},
  author  = {LeCun, Yann and Bottou, L{\'e}on and Bengio, Yoshua and Haffner, Patrick},
  journal = {Proceedings of the IEEE},
  volume  = {86},
  number  = {11},
  pages   = {2278--2324},
  year    = {1998}
}

@article{soudry2018implicit,
  title={The implicit bias of gradient descent on separable data},
  author={Soudry, Daniel and Hoffer, Elad and Nacson, Mor Shpigel and Gunasekar, Suriya and Srebro, Nathan},
  journal={Journal of Machine Learning Research},
  volume={19},
  number={70},
  pages={1--57},
  year={2018}
}

@inproceedings{gunasekar2018characterizing,
  title={Characterizing implicit bias in terms of optimization geometry},
  author={Gunasekar, Suriya and Lee, Jason and Soudry, Daniel and Srebro, Nathan},
  booktitle={International Conference on Machine Learning},
  pages={1832--1841},
  year={2018},
  organization={PMLR}
}

@article{gunasekar2018implicit,
  title={Implicit bias of gradient descent on linear convolutional networks},
  author={Gunasekar, Suriya and Lee, Jason D and Soudry, Daniel and Srebro, Nati},
  journal={Advances in neural information processing systems},
  volume={31},
  year={2018}
}

@article{arora2019exact,
  title={On exact computation with an infinitely wide neural net},
  author={Arora, Sanjeev and Du, Simon S and Hu, Wei and Li, Zhiyuan and Salakhutdinov, Russ R and Wang, Ruosong},
  journal={Advances in neural information processing systems},
  volume={32},
  year={2019}
}

@inproceedings{he2016deep,
  title={Deep residual learning for image recognition},
  author={He, Kaiming and Zhang, Xiangyu and Ren, Shaoqing and Sun, Jian},
  booktitle={Proceedings of the IEEE conference on computer vision and pattern recognition},
  pages={770--778},
  year={2016}
}

@inproceedings{ronneberger2015u,
  title={U-net: Convolutional networks for biomedical image segmentation},
  author={Ronneberger, Olaf and Fischer, Philipp and Brox, Thomas},
  booktitle={International Conference on Medical image computing and computer-assisted intervention},
  pages={234--241},
  year={2015},
  organization={Springer}
}

@article{lai2025principles,
  title={The principles of diffusion models},
  author={Lai, Chieh-Hsin and Song, Yang and Kim, Dongjun and Mitsufuji, Yuki and Ermon, Stefano},
  journal={arXiv preprint arXiv:2510.21890},
  year={2025}
}

@article{krizhevsky2012imagenet,
  title={Imagenet classification with deep convolutional neural networks},
  author={Krizhevsky, Alex and Sutskever, Ilya and Hinton, Geoffrey E},
  journal={Advances in neural information processing systems},
  volume={25},
  year={2012}
}

@inproceedings{oono2019approximation,
  title={Approximation and non-parametric estimation of ResNet-type convolutional neural networks},
  author={Oono, Kenta and Suzuki, Taiji},
  booktitle={International conference on machine learning},
  pages={4922--4931},
  year={2019},
  organization={PMLR}
}

@article{yarotsky2017error,
  title={Error bounds for approximations with deep ReLU networks},
  author={Yarotsky, Dmitry},
  journal={Neural networks},
  volume={94},
  pages={103--114},
  year={2017},
  publisher={Elsevier}
}

@inproceedings{liu2021besov,
  title={Besov function approximation and binary classification on low-dimensional manifolds using convolutional residual networks},
  author={Liu, Hao and Chen, Minshuo and Zhao, Tuo and Liao, Wenjing},
  booktitle={International Conference on Machine Learning},
  pages={6770--6780},
  year={2021},
  organization={PMLR}
}

@article{zhang2024nonparametric,
  title={Nonparametric classification on low dimensional manifolds using overparameterized convolutional residual networks},
  author={Zhang, Zixuan and Zhang, Kaiqi and Chen, Minshuo and Takeda, Yuma and Wang, Mengdi and Zhao, Tuo and Wang, Yu-Xiang},
  journal={Advances in Neural Information Processing Systems},
  volume={37},
  pages={65738--65764},
  year={2024}
}

@misc{dlmf,
  title        = {{NIST} Digital Library of Mathematical Functions},
  howpublished = {http://dlmf.nist.gov/},
  note         = {See \S5.6(i), Eq.~5.6.E4 for Gautschi's inequality (Gamma ratio bounds).}
}

@inproceedings{arpit2017memorization,
  author    = {Devansh Arpit and Stanislaw Jastrzebski and Nicolas Ballas and David Krueger and Emmanuel Bengio and Maxinder S. Kanwal and Tegan Maharaj and Asja Fischer and Aaron C. Courville and Yoshua Bengio and Simon Lacoste-Julien},
  title     = {A Closer Look at Memorization in Deep Networks},
  booktitle = {Proceedings of the 34th International Conference on Machine Learning (ICML)},
  pages     = {233--242},
  publisher = {PMLR},
  year      = {2017}
}

@article{haussler1992decision,
  title={Decision-theoretic generalizations of the PAC model for neural net and other learning applications},
  author={Haussler, David},
  journal={Information and Computation},
  volume={100},
  number={1},
  pages={78--150},
  year={1992}
}

@book{vanderVaartWellner1996,
  title={Weak Convergence and Empirical Processes},
  author={van der Vaart, Aad W. and Wellner, Jon A.},
  year={1996},
  publisher={Springer}
}

@inproceedings{liang2025_neural_shattering,
  title = {Stable Minima of {ReLU} Neural Networks Suffer from the Curse of Dimensionality: The Neural Shattering Phenomenon},
  author = {Liang, Tongtong and Qiao, Dan and Wang, Yu-Xiang and Parhi, Rahul},
  year = {2025},
  booktitle = {Advances in Neural Information Processing Systems (NeurIPS)},
}

@book{mohri2018foundations,
  title={Foundations of Machine Learning},
  author={Mohri, Mehryar and Rostamizadeh, Afshin and Talwalkar, Ameet},
  year={2018},
  publisher={MIT Press},
  edition={Second}
}

@inproceedings{liang2025datageometry,
  title = {Generalization Below the Edge of Stability: The Role of Data Geometry},
  author = {Liang, Tongtong and Cloninger, Alexander and Parhi, Rahul and Wang, Yu-Xiang},
  year = {2026},
  booktitle = {International Conference on Learning Representations (ICLR)},
}

@inproceedings{cohen2021gradient,
  title={Gradient descent on neural networks typically occurs at the edge of stability},
  author={Cohen, Jeremy and Kaur, Simran and Li, Yuanzhi and Kolter, J Zico and Talwalkar, Ameet},
  booktitle={International Conference on Learning Representations},
  year={2021}
}

@article{jacot2018neural,
  title={Neural tangent kernel: Convergence and generalization in neural networks},
  author={Jacot, Arthur and Gabriel, Franck and Hongler, Cl{\'e}ment},
  journal={Advances in neural information processing systems},
  volume={31},
  year={2018}
}

@article{parhi2023near,
  author    = {Rahul Parhi and Robert D. Nowak},
  title     = {Near-Minimax Optimal Estimation With Shallow {ReLU} Neural Networks},
  journal   = {IEEE Transactions on Information Theory},
  volume    = {69},
  number    = {2},
  pages     = {1125--1139},
  year      = {2023},
  publisher = {IEEE},
}

@inproceedings{wu2023implicit,
  title={The implicit regularization of dynamical stability in stochastic gradient descent},
  author={Wu, Lei and Su, Weijie J},
  booktitle={International Conference on Machine Learning},
  pages={37656--37684},
  year={2023},
  organization={PMLR}
}

@inproceedings{li2021convolutional,
  title={Why Are Convolutional Nets More Sample-Efficient than Fully-Connected Nets?},
  author={Li, Zhiyuan and Zhang, Yi and Arora, Sanjeev},
  booktitle={International Conference on Learning Representations},
year={2021}
}

@inproceedings{lahoti2024role,
  title={Role of Locality and Weight Sharing in Image-Based Tasks: A Sample Complexity Separation between CNNs, LCNs, and FCNs},
  author={Lahoti, Aakash and Karp, Stefani and Winston, Ezra and Singh, Aarti and Li, Yuanzhi},
  booktitle={The Twelfth International Conference on Learning Representations},
year={2024}
}

@article{wang2023theoretical,
  title={Theoretical analysis of the inductive biases in deep convolutional networks},
  author={Wang, Zihao and Wu, Lei},
  journal={Advances in Neural Information Processing Systems},
  volume={36},
  pages={74289--74338},
  year={2023}
}

@article{mulayoff2021implicit,
  title={The implicit bias of minima stability: A view from function space},
  author={Mulayoff, Rotem and Michaeli, Tomer and Soudry, Daniel},
  journal={Advances in Neural Information Processing Systems},
  volume={34},
  pages={17749--17761},
  year={2021}
}

@inproceedings{damian2023self,
  title={Self-Stabilization: The Implicit Bias of Gradient Descent at the Edge of Stability},
  author={Damian, Alex and Nichani, Eshaan and Lee, Jason D},
  booktitle={International Conference on Learning Representations},
  year={2023}
}

@inproceedings{nacson2022implicit,
  title={The Implicit Bias of Minima Stability in Multivariate Shallow {ReLU} Networks},
  author={Nacson, Mor Shpigel and Mulayoff, Rotem and Ongie, Greg and Michaeli, Tomer and Soudry, Daniel},
  booktitle={International Conference on Learning Representations},
  year={2023}
}

@article{wu2018sgd,
  title={How {SGD} selects the global minima in over-parameterized learning: A dynamical stability perspective},
  author={Wu, Lei and Ma, Chao and E, Weinan},
  journal={Advances in Neural Information Processing Systems},
  volume={31},
  year={2018}
}

@book{wainwright2019high,
  title={High-dimensional statistics: A non-asymptotic viewpoint},
  author={Wainwright, Martin J},
  volume={48},
  year={2019},
  publisher={Cambridge university press}
}

@inproceedings{
qiao2024stable,
title={Stable Minima Cannot Overfit in Univariate {ReLU} Networks: Generalization by Large Step Sizes},
author={Dan Qiao and Kaiqi Zhang and Esha Singh and Daniel Soudry and Yu-Xiang Wang},
booktitle={Advances in Neural Information Processing Systems},
	pages = {94163--94208},
	volume = {37},
year={2024},
}

@article{nar2018step,
  title={Step size matters in deep learning},
  author={Nar, Kamil and Sastry, Shankar},
  journal={Advances in Neural Information Processing Systems},
  volume={31},
  year={2018}
}

@article{ding2024flat,
  title={Flat minima generalize for low-rank matrix recovery},
  author={Ding, Lijun and Drusvyatskiy, Dmitriy and Fazel, Maryam and Harchaoui, Zaid},
  journal={Information and Inference: A Journal of the IMA},
  volume={13},
  number={2},
  pages={iaae009},
  year={2024},
  publisher={Oxford University Press}
}

@article{siegel2023characterization,
  title={Characterization of the variation spaces corresponding to shallow neural networks},
  author={Siegel, Jonathan W. and Xu, Jinchao},
  journal={Constructive Approximation},
  pages={1--24},
  year={2023},
  publisher={Springer}
}

@inproceedings{zhang2017rethinking,
  title     = {Understanding Deep Learning Requires Rethinking Generalization},
  author    = {Zhang, Chiyuan and Bengio, Samy and Hardt, Moritz and Recht, Benjamin and Vinyals, Oriol},
  booktitle = {International Conference on Learning Representations (ICLR)},
  year      = {2017},
  url       = {https://openreview.net/forum?id=Sy8gdB9xx},
  eprint    = {1611.03530},
  archivePrefix = {arXiv},
  primaryClass  = {cs.LG}
}

@book{AnthonyBartlett1999,
  author    = {Martin Anthony and Peter L. Bartlett},
  title     = {Neural Network Learning: Theoretical Foundations},
  publisher = {Cambridge University Press},
  year      = {1999}
}

@inproceedings{mhaskar2017and,
  title={When and why are deep networks better than shallow ones?},
  author={Mhaskar, Hrushikesh and Liao, Qianli and Poggio, Tomaso},
  booktitle={Proceedings of the AAAI conference on artificial intelligence},
  volume={31},
  number={1},
  year={2017}
}

@techreport{poggio2022foundations,
  title={Foundations of deep learning: Compositional sparsity of computable functions},
  author={Poggio, Tomaso},
  year={2022},
  institution={CBMM memo 138}
}

@article{poggio2024compositional,
  title={Compositional sparsity of learnable functions},
  author={Poggio, Tomaso and Fraser, Maia},
  journal={Bulletin of the American Mathematical Society},
  volume={61},
  number={3},
  pages={438--456},
  year={2024}
}

@article{zhou2020universality,
  title={Universality of deep convolutional neural networks},
  author={Zhou, Ding-Xuan},
  journal={Applied and computational harmonic analysis},
  volume={48},
  number={2},
  pages={787--794},
  year={2020},
  publisher={Elsevier}
}

@article{aharon2006k,
  title={K-SVD: An algorithm for designing overcomplete dictionaries for sparse representation},
  author={Aharon, Michal and Elad, Michael and Bruckstein, Alfred},
  journal={IEEE Transactions on signal processing},
  volume={54},
  number={11},
  pages={4311--4322},
  year={2006},
  publisher={IEEE}
}

@inproceedings{dosovitskiy2021image,
  title     = {An Image is Worth 16x16 Words: Transformers for Image Recognition at Scale},
  author    = {Dosovitskiy, Alexey and Beyer, Lucas and Kolesnikov, Alexander and Weissenborn, Dirk and Zhai, Xiaohua and Unterthiner, Thomas and Dehghani, Mostafa and Minderer, Matthias and Heigold, Georg and Gelly, Sylvain and Uszkoreit, Jakob and Houlsby, Neil},
  booktitle = {International Conference on Learning Representations},
  year      = {2021}
}

@article{paulin2017convolutional,
  title={Convolutional patch representations for image retrieval: an unsupervised approach},
  author={Paulin, Mattis and Mairal, Julien and Douze, Matthijs and Harchaoui, Zaid and Perronnin, Florent and Schmid, Cordelia},
  journal={International Journal of Computer Vision},
  volume={121},
  number={1},
  pages={149--168},
  year={2017},
  publisher={Springer}
}

@article{peyre2009manifold,
  title={Manifold models for signals and images},
  author={Peyr{\'e}, Gabriel},
  journal={Computer vision and image understanding},
  volume={113},
  number={2},
  pages={249--260},
  year={2009},
  publisher={Elsevier}
}

@article{mao2021theory,
  title={Theory of deep convolutional neural networks III: Approximating radial functions},
  author={Mao, Tong and Shi, Zhongjie and Zhou, Ding-Xuan},
  journal={Neural Networks},
  volume={144},
  pages={778--790},
  year={2021},
  publisher={Elsevier}
}

@article{shi2025approximation,
  title={Approximation and estimation capability of Vision Transformers for hierarchical compositional models},
  author={Shi, Zhongjie and Fang, Zhiying and Cao, Yuan},
  journal={Applied and Computational Harmonic Analysis},
  pages={101849},
  year={2025},
  publisher={Elsevier}
}

@inproceedings{brutzkus2022efficient,
  title={Efficient learning of cnns using patch based features},
  author={Brutzkus, Alon and Globerson, Amir and Malach, Eran and Netser, Alon Regev and Shalev-Schwartz, Shai},
  booktitle={International Conference on Machine Learning},
  pages={2336--2356},
  year={2022},
  organization={PMLR}
}

@inproceedings{coates2011analysis,
  title={An analysis of single-layer networks in unsupervised feature learning},
  author={Coates, Adam and Ng, Andrew and Lee, Honglak},
  booktitle={Proceedings of the fourteenth international conference on artificial intelligence and statistics},
  pages={215--223},
  year={2011},
  organization={JMLR Workshop and Conference Proceedings}
}

@article{trockman2022patches,
  title={Patches are all you need?},
  author={Trockman, Asher and Kolter, J Zico},
  journal={arXiv preprint arXiv:2201.09792},
  year={2022}
}

@inproceedings{thiry2021unreasonable,
  title     = {The Unreasonable Effectiveness of Patches in Deep Convolutional Kernels Methods},
  author    = {Thiry, Louis and Arbel, Michael and Belilovsky, Eugene and Oyallon, Edouard},
  booktitle = {International Conference on Learning Representations},
  year      = {2021}
}

@inproceedings{xie2017aggregated,
  title     = {Aggregated Residual Transformations for Deep Neural Networks},
  author    = {Xie, Saining and Girshick, Ross B. and Doll{\'a}r, Piotr and Tu, Zhuowen and He, Kaiming},
  booktitle = {IEEE Conference on Computer Vision and Pattern Recognition ({CVPR})},
  pages     = {5987--5995},
  year      = {2017}
}

@inproceedings{liu2022convnet,
  title     = {A {ConvNet} for the 2020s},
  author    = {Liu, Zhuang and Mao, Hanzi and Wu, Chao-Yuan and Feichtenhofer, Christoph and Darrell, Trevor and Xie, Saining},
  booktitle = {IEEE/CVF Conference on Computer Vision and Pattern Recognition ({CVPR})},
  pages     = {11966--11976},
  year      = {2022}
}

@inproceedings{liu2021swin,
  title     = {Swin Transformer: Hierarchical Vision Transformer Using Shifted Windows},
  author    = {Liu, Ze and Lin, Yutong and Cao, Yue and Hu, Han and Wei, Yixuan and Zhang, Zheng and Lin, Stephen and Guo, Baining},
  booktitle = {IEEE/CVF International Conference on Computer Vision ({ICCV})},
  pages     = {9992--10002},
  year      = {2021}
}

@inproceedings{yuan2021tokens,
  title     = {Tokens-to-Token ViT: Training Vision Transformers From Scratch on ImageNet},
  author    = {Yuan, Li and Chen, Yunpeng and Wang, Tao and Yu, Weihao and Shi, Yujun and Jiang, Zihang and Tay, Francis E. H. and Feng, Jiashi and Yan, Shuicheng},
  booktitle = {Proceedings of the IEEE/CVF International Conference on Computer Vision (ICCV)},
  pages     = {558--567},
  year      = {2021}
}

@inproceedings{wu2021cvt,
  title     = {CvT: Introducing Convolutions to Vision Transformers},
  author    = {Wu, Haiping and Xiao, Bin and Codella, Noel and Liu, Mengchen and Dai, Xiyang and Yuan, Lu and Zhang, Lei},
  booktitle = {Proceedings of the IEEE/CVF International Conference on Computer Vision (ICCV)},
  pages     = {22--31},
  year      = {2021}
}

@article{wang2022pvtv2,
  title   = {PVT v2: Improved Baselines with Pyramid Vision Transformer},
  author  = {Wang, Wenhai and Xie, Enze and Li, Xiang and Fan, Deng-Ping and Song, Kaitao and Liang, Ding and Lu, Tong and Luo, Ping and Shao, Ling},
  journal = {Computational Visual Media},
  volume  = {8},
  number  = {3},
  pages   = {415--424},
  year    = {2022}
}

@inproceedings{caron2021emerging,
  title     = {Emerging Properties in Self-Supervised Vision Transformers},
  author    = {Caron, Mathilde and Touvron, Hugo and Misra, Ishan and J{\'e}gou, Herv{\'e} and Mairal, Julien and Bojanowski, Piotr and Joulin, Armand},
  booktitle = {Proceedings of the IEEE/CVF International Conference on Computer Vision (ICCV)},
  pages     = {9650--9660},
  year      = {2021}
}

@inproceedings{xiao2021early,
  title     = {Early Convolutions Help Transformers See Better},
  author    = {Xiao, Tete and Singh, Mannat and Mintun, Eric and Darrell, Trevor and Doll{\'a}r, Piotr and Girshick, Ross},
  booktitle = {Advances in Neural Information Processing Systems},
  volume    = {34},
  pages     = {30392--30400},
  year      = {2021}
}

@inproceedings{ding2022replknet,
  title     = {Scaling Up Your Kernels to 31x31: Revisiting Large Kernel Design in CNNs},
  author    = {Ding, Xiaohan and Zhang, Xiangyu and Zhou, Yizhuang and Han, Jungong and Ding, Guiguang and Sun, Jian},
  booktitle = {Proceedings of the IEEE/CVF Conference on Computer Vision and Pattern Recognition (CVPR)},
  pages     = {11963--11975},
  year      = {2022}
}

@inproceedings{beyer2023flexivit,
  title     = {FlexiViT: One Model for All Patch Sizes},
  author    = {Beyer, Lucas and Izmailov, Pavel and Kolesnikov, Alexander and Caron, Mathilde and Kornblith, Simon and Zhai, Xiaohua and Minderer, Matthias and Tschannen, Michael and Alabdulmohsin, Ibrahim and Pavetic, Filip},
  booktitle = {Proceedings of the IEEE/CVF Conference on Computer Vision and Pattern Recognition (CVPR)},
  pages     = {14496--14506},
  year      = {2023}
}

@article{nguyen2024morethan16,
  title   = {An Image is Worth More Than 16x16 Patches: Exploring Transformers on Individual Pixels},
  author  = {Nguyen, Duy-Kien and Assran, Mahmoud and Jain, Unnat and Oswald, Martin R. and Snoek, Cees G. M. and Chen, Xinlei},
  journal = {arXiv preprint arXiv:2406.09415},
  year    = {2024}
}

@inproceedings{wang2025patchification,
  title     = {Scaling Laws in Patchification: An Image Is Worth 50,176 Tokens And More},
  author    = {Wang, Feng and Yu, Yaodong and Wei, Guoyizhe and Shao, Wei and Zhou, Yuyin and Yuille, Alan and Xie, Cihang},
  booktitle = {Proceedings of the 42nd International Conference on Machine Learning},
  series    = {Proceedings of Machine Learning Research},
  volume    = {267},
  pages     = {65278--65290},
  publisher = {PMLR},
  year      = {2025}
}

@inproceedings{touvron2021deit,
  title     = {Training Data-Efficient Image Transformers \& Distillation Through Attention},
  author    = {Touvron, Hugo and Cord, Matthieu and Douze, Matthijs and Massa, Francisco and Sablayrolles, Alexandre and J{\'e}gou, Herv{\'e}},
  booktitle = {Proceedings of the 38th International Conference on Machine Learning},
  series    = {Proceedings of Machine Learning Research},
  volume    = {139},
  pages     = {10347--10357},
  publisher = {PMLR},
  year      = {2021}
}

@InProceedings{Wang_2025_CVPR,
  author    = {Wang, Feng and Yang, Timing and Yu, Yaodong and Ren, Sucheng and Wei, Guoyizhe and Wang, Angtian and Shao, Wei and Zhou, Yuyin and Yuille, Alan and Xie, Cihang},
  title     = {Adventurer: Optimizing Vision Mamba Architecture Designs for Efficiency},
  booktitle = {Proceedings of the IEEE/CVF Conference on Computer Vision and Pattern Recognition (CVPR)},
  month     = {June},
  year      = {2025},
  pages     = {30157-30166}
}

@inproceedings{yao2020pyhessian,
  title={Pyhessian: Neural networks through the lens of the hessian},
  author={Yao, Zhewei and Gholami, Amir and Keutzer, Kurt and Mahoney, Michael W},
  booktitle={2020 IEEE international conference on big data (Big data)},
  year={2020},
}
\bibliographystyle{plainnat}

\clearpage
\appendix
\section{Broader (Informal) Discussion on Network Architecture Design: Bridge our Theoretical Insights to Empirical Observations }
\label{app:patch_size_local_operator_geometry}

This appendix expands on the architectural implications discussed in Section~\ref{sec:discussion}.
The purpose is not to claim that the stability analysis in the main text provides a complete theory of modern deep vision architectures.
Rather, the point is that stability gives a tractable lens through which one can see a more general object: the geometry of the local vectors exposed to gradient descent.
This object appears explicitly in the shallow model analyzed in this paper, and it has close analogues in the patchification and local-operator design choices used in modern CNNs and Vision Transformers.

\noindent\textbf{The patch matrix as the geometry seen by backpropagation.}
In the setting of this paper, consider the patch matrix
\[
\mathbf X\in\Rb^{(nJ)\times m},
\]
whose row indexed by $(i,j)$ equals $\pi_j(\vec{x}_i)^\T$.
Let $r_i:=f_{\vec{\theta}}(\vec{x}_i)-y_i$, and define the lifted residual vector $\bar{\mathbf r}\in\Rb^{nJ}$ by
\[
(\bar{\mathbf r})_{(i,j)}:=r_i.
\]
Let $\mathbf S(\vec{\theta})\in\{0,1\}^{(nJ)\times K}$ be the gating matrix
\[
\mathbf{S}(\vec{\theta})_{(i,j),k}
\;:=\;
\mathds{1}\bigl\{\vec{w}_k^\T \pi_j(\vec{x}_i) > b_k\bigr\},
\]
and let $\vec v:=(v_1,\ldots,v_K)^\T$.
Stack first-layer weights as
\[
\mathbf{W}:=[\vec{w}_1^\T;\ldots;\vec{w}_K^\T]\in\Rb^{K\times m}.
\]
A direct calculation gives a backprop-aligned factorization of the gradient.
Define the patch-level backprop signal
\[
\mathbf R(\vec{\theta})
\!:=
\frac{1}{nJ}\,\bar{\mathbf r}\,\vec v^\T
\in\Rb^{(nJ)\times K},
\quad
\mathbf G(\vec{\theta})
\!:=
\mathbf R(\vec{\theta})\odot \mathbf S(\vec{\theta}),
\]
where $\odot$ denotes entrywise multiplication.
Then
\begin{equation}\label{eq:grad_GtX_exact}
\nabla_{\mathbf W}\loss(\vec{\theta})
\;=\;
\mathbf G(\vec{\theta})^\T \mathbf X.
\end{equation}
Equivalently, writing $\mathbf V:=\mathrm{diag}(v_1,\ldots,v_K)$, one can express the same signal as
\[
\mathbf G(\vec{\theta})
=
\frac{1}{nJ}\,
\mathrm{diag}(\bar{\mathbf r})\,
\mathbf S(\vec{\theta})\,
\mathbf V.
\]

The factorization \eqref{eq:grad_GtX_exact} separates two roles.
The matrix $\mathbf G(\vec{\theta})$ aggregates the learning signal generated by residuals, readout weights, and gates.
The patch matrix $\mathbf X$ then maps this signal into parameter updates.
Thus the architecture does not only restrict the number of parameters, it determines the geometry through which backpropagated signals are converted into motion in parameter space.
In this sense, $\mathbf{X}$ acts as a signal rectifier in the sense that data priors and receptive-field structure jointly shape the directions along which gradient descent can effectively move.
For example, Figure~\ref{fig:patch_vs_image_geometry} shows that, for CIFAR-10, $90\%$ of the energy of the convolutional patch matrix is concentrated in only a few principal directions, suggesting a strong geometric constraint on the filter dynamics.

This viewpoint also clarifies why the stability condition in the main text is informative.
The terms appearing in the definition of $g_{\CD,\CS}$ are drawn from the same ingredients as \eqref{eq:grad_GtX_exact}: gate statistics induced by $\mathbf S(\vec{\theta})$ and geometric moments of the patch cloud encoded by $\mathbf X$.
Theorem~\ref{thm:upper_bound} can therefore be viewed as a static stability proxy for the backpropagation geometry in \eqref{eq:grad_GtX_exact}.
It does not say that stability alone explains generalization.
Indeed, Theorem~\ref{thm:flat_lcws_gap} shows that stable solutions may still memorize when the patch cloud admits isolating half-spaces. If corresponding backpropagation direction may enable a filter move to a location such that it only activates on a rare, nearly isolated patch, then the network may fit labels while paying little stability cost.
If the patch cloud is concentrated and resistant to such isolation, the same stability lens yields an effective regularity constraint.

\noindent\textbf{Patch size as local-operator dimension in Vision Transformers.}
Vision Transformers make the role of patch geometry especially explicit.
The original ViT converts an image into a sequence of non-overlapping patches and applies a shared patch embedding before global token mixing~\citep{dosovitskiy2021image}.
This patchification step is often introduced as an efficiency device, since larger patches shorten the token sequence and reduce the cost of attention.
The factorization \eqref{eq:grad_GtX_exact} suggests a second interpretation.
Patchification also chooses the first local geometry seen by training.
A larger patch presents each shared local operator with a higher-dimensional and more compressed local vector. A smaller patch reduces this local dimension and preserves finer spatial information, while increasing the number of tokens that later layers must mix.

Several empirical lines of work support the idea that patch size is a genuine training and generalization variable, not only an implementation detail.
FlexiViT trains a single ViT over a range of patch sizes and demonstrates that patch size controls an accuracy--compute tradeoff at deployment time~\citep{beyer2023flexivit}.
Self-supervised ViTs also reveal the importance of this axis: DINO identifies small patches as one ingredient behind the strong emergent properties of self-supervised ViTs~\citep{caron2021emerging}.
Early-convolution studies make the same point from the optimization side.
The standard ViT patchify stem is a large-kernel, large-stride convolution, and replacing it by a short stack of small-stride convolutions significantly changes optimization stability and improves performance under common training recipes~\citep{xiao2021early}.
These results are consistent with the idea that the first patch matrix is not an innocuous preprocessing choice: it shapes the geometry of the signals that subsequent layers and backpropagation operate on.

\begin{table}[t]
\centering
\caption{Selected patchification-scaling results adapted from ~\citep{wang2025patchification}.
Patch size denotes the side length of the square image patch used to form one token.
Panel A shows that reducing patch size improves ImageNet accuracy.
Panel B separates true patch-size reduction from only increasing the sequence length by interpolating coarse tokens. Here we only list their results on DeiT-B \citep{touvron2021deit} and Adventurer-B \citep{Wang_2025_CVPR}.}
\label{tab:wang_patchification_scaling}
\small
\setlength{\tabcolsep}{6pt}
\renewcommand{\arraystretch}{1.12}

\begin{tabular}{@{}lccccc@{}}
\toprule
\multicolumn{6}{@{}l}{\textbf{Panel A: ImageNet classification top-1 accuracy}} \\
\midrule
\textbf{Model and input} 
& \textbf{$16\times16$} 
& \textbf{$8\times8$} 
& \textbf{$4\times4$} 
& \textbf{$2\times2$} 
& \textbf{$1\times1$} \\
\midrule
DeiT-B, $64\times64$ input         
& 68.2 & 76.9 & 80.1 & 80.8 & 81.3 \\
DeiT-B, $128\times128$ input       
& 78.1 & 81.0 & 82.3 & 82.9 & -- \\
Adventurer-B, $224\times224$ input 
& 82.6 & 83.9 & 84.3 & 84.5 & 84.6 \\
\midrule[\heavyrulewidth]
\multicolumn{6}{@{}l}{\textbf{Panel B: Sequence-length ablation on ImageNet top-1 accuracy}} \\
\midrule
\textbf{Model and input} 
& \textbf{Seq. length} 
& \multicolumn{2}{c}{\textbf{Interpolated coarse tokens}} 
& \multicolumn{2}{c}{\textbf{True patch scaling}} \\
\midrule
DeiT-B, $128\times128$ input       
& 256    
& \multicolumn{2}{c}{78.2 $(+0.1)$} 
& \multicolumn{2}{c}{81.0 $(+2.9)$} \\
DeiT-B, $128\times128$ input       
& 1,024  
& \multicolumn{2}{c}{78.2 $(+0.1)$} 
& \multicolumn{2}{c}{82.3 $(+4.2)$} \\
Adventurer-B, $224\times224$ input 
& 784    
& \multicolumn{2}{c}{82.7 $(+0.1)$} 
& \multicolumn{2}{c}{83.9 $(+1.3)$} \\
Adventurer-B, $224\times224$ input 
& 3,136  
& \multicolumn{2}{c}{82.8 $(+0.2)$} 
& \multicolumn{2}{c}{84.3 $(+1.7)$} \\
Adventurer-B, $224\times224$ input 
& 12,544 
& \multicolumn{2}{c}{82.8 $(+0.2)$} 
& \multicolumn{2}{c}{84.5 $(+1.9)$} \\
\bottomrule
\end{tabular}
\end{table}

Recent patchification scaling studies push this observation further.
Wang et al.~\citep{wang2025patchification} systematically reduce the patch size from the standard coarse-token regime toward pixel-level tokenization and report smooth gains across classification, semantic segmentation, object detection, and instance segmentation.
Table~\ref{tab:wang_patchification_scaling} summarizes the results most relevant to our discussion.
The most informative part is their sequence-length ablation.
When coarse patch tokens are merely interpolated into a longer sequence, the gains remain small.
When genuinely smaller patches are extracted from the image, the gains are much larger.
This suggests that patch-size scaling changes the local visual primitives exposed to the model, not only the number of tokens.
In the notation of this paper, reducing patch size changes the initial patch matrix itself: it lowers spatial compression in each local vector and changes the multiset of local rows through which the network forms its updates.
This aligns with the gradient factorization $\nabla_{\mathbf W}\loss=\mathbf G^\T\mathbf X$: the relevant object is the geometry of the rows in the patch or token matrix, not only the nominal sequence length.
Related pixel-token studies push the same question to the extreme by asking how far visual models can go when patchification is removed almost entirely~\citep{nguyen2024morethan16}.

\noindent\textbf{Why small patches require hierarchy and mixing.}
Smaller patches improve the local geometry exposed to the model, but they also increase the burden of mixing information across locations.
This tradeoff helps explain a major trend in vision-transformer architecture design.
Hierarchical Transformers such as Swin do not keep a fixed high-dimensional token space throughout the network; they begin with finer local tokens, restrict early mixing to local windows, and then gradually merge tokens while increasing representation dimension~\citep{liu2021swin}.
PVTv2 follows a related logic through overlapping patch embeddings, linear-complexity attention, and convolutional feed-forward layers~\citep{wang2022pvtv2}.
CvT introduces convolutional token embeddings and convolutional projections inside Transformer blocks~\citep{wu2021cvt}.
Tokens-to-Token ViT replaces one-shot patchification by a progressive tokenization module that recursively aggregates local neighboring tokens~\citep{yuan2021tokens}.
These designs are not only remedies for a defective large-patch ViT.
They can be read as different ways of managing the same geometric tradeoff: expose early layers to controlled local vectors, then recover expressivity through hierarchy, overlap, convolutional projection, or stage-wise mixing.

In the notation of \eqref{eq:grad_GtX_exact}, a deep network replaces the fixed raw patch matrix $\mathbf X$ by a sequence of representation matrices.
At layer $\ell$, one should imagine a matrix
\[
\mathbf X_\ell
=
\bigl[
\vec z_{\ell,j}(\vec x_i)^\T
\bigr]_{(i,j)}
\]
whose rows are local tokens or local feature vectors produced by the preceding layers.
Backpropagation through a local operator at layer $\ell$ again has the schematic form
\[
\nabla_{\mathbf W_\ell}\loss
\approx
\mathbf G_\ell(\vec{\theta})^\T \mathbf X_\ell,
\]
up to the additional Jacobian factors introduced by normalization, residual connections, attention, and nonlinear mixing.
Thus the relevant geometry is no longer fixed by the raw data distribution.
It evolves during training.
Patch size, windowing, overlap, token merging, and convolutional stems can all be viewed as mechanisms for shaping the sequence
\[
\mathbf X_0,\mathbf X_1,\ldots,\mathbf X_L
\]
of optimization-facing geometries.

\noindent\textbf{Patch geometry in CNNs: from spatial-channel mixing to depthwise convolutional kernel.}
CNN architecture design provides a sequence of examples in which the vector seen by each local operator is progressively controlled.
In a residual bottleneck block, the $3\times3$ spatial convolution acts after a channel-reducing $1\times1$ projection, so the spatial operator does not see the full channel dimension of the block~\citep{he2016deep}.
ResNeXt makes this control explicit through grouped convolutions: each group applies a $3\times3$ operator to only a fraction of the channels, reducing the local vector from $3\times3\times C_{\mathrm{in}}$ to $3\times3\times C_{\mathrm{in}}/g$~\citep{xie2017aggregated}.
ConvNeXt pushes the same principle further by using depthwise spatial convolutions inside an inverted-bottleneck block, with expressive channel mixing supplied by pointwise layers~\citep{liu2022convnet}.
Thus, from ResNet to ResNeXt to ConvNeXt, the spatial operator is increasingly restricted to a structured local geometry, while expressivity is recovered through channel mixing and depth.

The kernel-size ablation in ConvNeXt gives a useful empirical instance of the trade-off suggested by our theory.
Increasing the depthwise kernel from $3\times3$ to $7\times7$ improves ImageNet accuracy, while increasing it further to $9\times9$ or $11\times11$ yields little additional gain~\citep{liu2022convnet}.
In the language of this paper, enlarging the spatial kernel increases the context available to each local operator, reducing the bias induced by overly local processing.
At the same time, a larger local vector exposes gradient descent to a richer patch geometry, which can weaken the regularizing effect of locality.
The saturation around $7\times7$ is consistent with a bias--variance trade-off for the geometry seen by local operators: enough spatial support improves representation, while excessive support brings less benefit once the local computation is already sufficiently expressive.

RepLKNet develops this direction further by treating very large kernels as a separate architectural scaling dimension~\citep{ding2022replknet}.
It shows that replacing ConvNeXt's $7\times7$ kernels with much larger stage-wise kernels can further improve performance, especially on dense prediction tasks.
The key point for our purposes is not only the larger spatial support, but the training-time parameterization used to make such kernels effective.
RepLKNet does not train a bare large kernel.
It introduces \emph{small-kernel structural re-parameterization}: during training, a large kernel branch is paired with a small kernel branch, such as a $3\times3$ or $5\times5$ branch. After training, the small kernel is padded to the center of the large kernel and algebraically merged, so inference uses a single large kernel.

This mechanism has a direct interpretation in our gradient-geometry framework.
Let $\vec{x}_L\in\Rb^{K^2}$ denote the large local patch and let $\vec{x}_s=\mC\vec{x}_L\in\Rb^{s^2}$ be the centered small patch selected from it.
Ignoring normalization for clarity, the training-time operator has the form
\[
h(\vec{x})= \vec{w}_L^\T\vec{x}_L+ \vec{w}_s^\T\vec{x}_s
     =(\vec{w}_L+\mC^\T \vec{w}_s)^\T\vec{x}_L,
\qquad
\vec{w}_{\mathrm{eff}}=\vec{w}_L+\mC^\T \vec{w}_s .
\]
At inference time this is just a large kernel with weight $\vec{w}_{\mathrm{eff}}$.
During training, however, the two branches expose different local geometries to gradient descent.
If $r$ denotes the backpropagated scalar signal, then
\[
\nabla_{\vec{w}_L}\loss=r\vec{x}_L,\qquad \nabla_{\vec{w}_s}\loss=r\mC\vec{x}_L,
\qquad
\Delta \vec{w}_{\mathrm{eff}}=-\eta r(\mI+\mC^\T \mC)\vec{x}_L ,
\]
up to the branch-dependent scaling introduced by normalization.
The small branch therefore gives the centered local subspace an additional optimization path.
It biases training toward small-scale local structure while the large branch captures broad spatial context.

This example sharpens the message of the patch-matrix factorization $\nabla_{\mathbf W}\loss=\mathbf G^\T\mathbf X$.
The geometry relevant to learning is not determined solely by the inference-time operator.
RepLKNet uses an inference-time large kernel, but its training-time architecture exposes gradient descent to both a large-patch geometry and a centered small-patch geometry.
Thus structural re-parameterization is a practical mechanism for modifying the optimization-facing patch geometry without changing the final operator class.
Together with the ConvNeXt ablation, this supports the broader view that CNN design balances spatial context against implicit regularization induced by controlled local geometry.

\section{Functional Analysis of Shallow ReLU Networks}

\subsection{Path-norm and Variation Semi-norm of ReLU Networks}
This subsection collects several basic facts from \citep{parhi2023near} and \citep{siegel2023characterization} that we will use later.

\begin{definition}
	Let $f_\vec{\theta}(\vec{x})=\sum_{k=1}^K v_k\,\phi(\vec{w}_k^\T \vec{x}-b_k)+\beta$ be a fully-connected two-layer neural network. The (unweighted) path-norm of $f_\vec{\theta}$ is defined to be
	\begin{equation}\label{eq: pathnorm}
		\|f_{\vec{\theta}}\|_{\mathrm{path}}\!:= \sum_{k=1}^K |v_k|\norm{\vec{w}_k}_2.
	\end{equation}
\end{definition}

\noindent\textbf{Dictionary representation of ReLU networks.}
Using the positive $1$-homogeneity of $\relu$, one may rescale each hidden unit while leaving the realized function unchanged:
\[
v_k\,\phi(\vec{w}_k^\T \vec{x}-b_k)
= a_k\,\phi(\vec{u}_k^\T \vec{x}-t_k),
\quad 
\vec{u}_k:=\frac{\vec{w}_k}{\norm{\vec{w}_k}_2}\in \Sph^{d-1},\;
t_k:=\frac{b_k}{\norm{\vec{w}_k}_2},\;
a_k:=v_k\norm{\vec{w}_k}_2.
\]
Consequently, $f_\theta$ can be written in the normalized finite-sum form
\begin{equation}\label{eq:reduced_form}
f(\vec{x})=\sum_{k=1}^{K'} a_k\,\phi(\vec{u}_k^\T \vec{x}-t_k)+\vec{c}^\T \vec{x}+c_0.
\end{equation}

Define the (ReLU) ridge dictionary as $\DictReLU:=\curly{\phi(\vec{u}^\T \cdot - t):\; \vec{u}\in \Sph^{d-1},\; t\in \Rb}$. We focus on the \emph{overparameterized, width-agnostic} collection obtained by taking the \emph{union over all finite widths}
\begin{equation}\label{eq:union-class}
\F_{\mathrm{fin}}
:= \bigcup_{K\ge 1}\curly{\sum_{k=1}^{K} a_k\,\phi(\vec{u}_k^\T \cdot - t_k)+\vec{c}^\T(\cdot)+c_0},
\end{equation}
and quantify complexity via the smallest path-norm among all realizations of $f$:
\[
\|f\|_{\mathrm{path},\min}:=\inf\curly{\|f_{\vec{\theta}}\|_{\mathrm{path}}:\; f_{\vec{\theta}}\equiv f\text{ of the form \eqref{eq:reduced_form}} }.
\]

\noindent\textbf{From finite sums to a \emph{width-agnostic} integral representation.}
Rather than fixing a particular width $K$, it is convenient to work with a convex, measure-theoretic formulation that \emph{captures the closure/convex hull of} \eqref{eq:union-class}. Concretely, let $\nu$ be a finite signed Radon measure on $\Sph^{d-1} \times [-R, R]$ and consider
\begin{equation}\label{eq:integral-rep}
f(\vec{x})=\int_{\Sph^{d-1} \times [-R, R]}\phi(\vec{u}^\T \vec{x}-t)\, \dd\nu(\vec{u},t)+c^\T \vec{x}+c_0.
\end{equation}

Every finite network of the form \eqref{eq:reduced_form} corresponds to an \emph{atomic} (hence sparse) measure $\nu=\sum_{k=1}^{K} a_k\,\delta_{(\vec{u}_k,t_k)}$, and conversely any such atomic $\nu$ yields a finite network. Therefore, \eqref{eq:integral-rep} should be viewed as a \emph{width-agnostic relaxation} aligned with \eqref{eq: pathnorm}, rather than as an assumption of an infinite-width limit.

\begin{definition}\label{def: unweighted variation space}
	The (unweighted) variation (semi)norm
\begin{equation}\label{eq:var-norm}
|f|_{\Variation}:=\inf\curly{\norm{\nu}_{\M}:\; f \text{ admits \eqref{eq:reduced_form} for some }(\nu,c,c_0)},
\end{equation}
where $\norm{\nu}_{\M}$ is the total variation of $\nu$.

    For the compact region $\Omega=\Bb_R^d$, we define the bounded variation function class as
\begin{equation}\label{eq: variation_space}
	\Variation_C(\Omega)\!:=\left\{f\!:\Omega\rightarrow \Rb\mid f=\int_{\Sph^{d-1} \times [-R, R]} \phi(\vec{u}^\T\vec{x} - t) \dd\nu(\vec{u}, t)+\vec{c}^{\T}\vec{x}+b,\,|f|_{\Variation}\leq C \right\}.
\end{equation}
\end{definition}
In particular, identifying \eqref{eq:reduced_form} with the atomic measure $\nu=\sum_k a_k\delta_{(\vec{u}_k,t_k)}$ yields
\[
|f|_{\Variation}\le \sum_{k}|a_k|=\|f_{\vec{\theta}}\|_{\mathrm{path}},
\quad\text{hence}\quad
|f|_{\Variation}\le \|f\|_{\mathrm{path},\min}.
\]
Moreover, the minimal total variation required to represent $f$ coincides with the minimal path-norm over all finite decompositions:
\begin{equation}\label{eq:path-eq-var}
\|f\|_{\mathrm{path},\min} = |f|_{\Variation}.
\end{equation}
Thus, \eqref{eq:var-norm} provides a \emph{nonparametric} analogue of the path-norm: it encodes the same complexity notion while \emph{not committing to a fixed width} $K$.

\begin{remark}[``Arbitrary width'' $\neq$ ``infinite width'']\label{rem:arbitrary-vs-infinite}
All statements here are about $\F_{\mathrm{fin}}$ in \eqref{eq:union-class}, namely networks of \emph{finite} (but unconstrained) width. The integral representation \eqref{eq:integral-rep} is introduced as a convenient convexification/closure of this union for analysis and regularization; it \emph{does not} posit an infinite-width limit. In particular, when training in the variational form with a total-variation penalty on $\nu$, first-order optimality implies that optimal measures are sparse (i.e., have finite support), which corresponds exactly to \emph{finite-width} networks. Therefore, our results hold for \emph{arbitrary (yet finite) width}; the continuum measure serves only as a tool to characterize and control $\|f\|_{\mathrm{path},\min}$.
\end{remark}

\subsection{The Metric Entropy of Variation Spaces}

Metric entropy is a standard way to describe how ``compact'' a subset $A$ is inside a metric space $(X,\rho_X)$. We recall the notions of covering numbers and metric entropy.

\begin{definition}[Covering Number and Entropy]
Let $A$ be a compact subset of a metric space $(X, \rho_X)$. For $t > 0$, the \emph{covering number} $N(A, t, \rho_X)$ is the smallest number of closed balls of radius $t$ whose union contains $A$:
\begin{equation}
	N(t, A, \rho_X) := \min \left\{ N \in \mathbb{N} : \exists\, x_1, \dots, x_N \in X \text{ s.t. } A \subset \bigcup_{i=1}^N \Bb(x_i, t) \right\},
\end{equation}
where $\Bb(x_i, t) = \{ y \in X : \rho_X(y,x_i) \le t \}$. The \emph{metric entropy} of $A$ at scale $t$ is then
\begin{equation}\label{defn: metric_entropy}
	H_{t}(A)_X := \log N(t, A, \rho_X).
\end{equation} 
\end{definition}

Covering/entropy bounds for bounded-variation-type classes have been established in the literature. In what follows, we will use the estimate stated below.
\begin{proposition}[{\citealt[Appendix D]{parhi2023near}}]\label{prop:metric_entropy_of_bounded_variation_space}
	The metric entropy of $\Variation_C(\Bb_R^d)$ (see Definition \ref{def: unweighted variation space}) with respect to the $L^\infty(\Bb_R^d)$-distance $\|\cdot \|_{\infty}$ satisfies 
	\begin{equation}
		\log N(t,\Variation_C(\Bb_R^d),\| \cdot\|_{\infty})\lessapprox_d \left(\frac{C}{t} \right)^{\frac{2d}{d+3}}.
	\end{equation}
	where $\lessapprox_d$ hides constants (which could depend on $d$) and logarithmic factors.
\end{proposition}
\subsection{Generalization Gap of Unweighted Variation Function Class}
As a middle step towards bounding the generalization gap of the weighted variation
function class, we bound the generalization gap of the unweighted variation function
class using chaining and Gaussian complexity, together with the $L^\infty$ metric entropy
bound in Proposition~\ref{prop:metric_entropy_of_bounded_variation_space}.

\begin{proposition}[{\citealt[Chapter~13]{wainwright2019high}}]
\label{prop:gauss-chaining}
Fix design points $\vec{x}_{1},\dots,\vec{x}_{n}$ and denote the empirical norm
\[
      \|f\|_{n}^{2}:=\frac{1}{n}\sum_{i=1}^{n}f(\vec{x}_{i})^2.
\]
Let $\mathcal{F}$ be a class of real-valued functions on $\{\vec{x}_i\}_{i=1}^n$, and define
\[
      \widehat{\GaussC}_{n}(\mathcal{F})
      :=\sup_{f\in\mathcal{F}}
          \frac{1}{n}\sum_{i=1}^{n}\varepsilon_{i}f(\vec{x}_{i}),
      \qquad
      \varepsilon_{1},\dots,\varepsilon_{n}\stackrel{\text{i.i.d.}}{\sim}\mathcal{N}(0,1),
      \quad
      \GaussC_{n}(\mathcal{F}):=\mathop{\mathbb{E}}\,\widehat{\GaussC}_{n}(\mathcal{F}).
\]
Then
\begin{equation}
      \GaussC_{n}(\mathcal{F})
      \;\lesssim\;
      \frac{16}{\sqrt{n}}
      \int_{0}^{\mathrm{diam}(\mathcal{F},\|\cdot\|_n)}
          \sqrt{\log N
            \bigl(t,\;\mathcal{F},\;\|\cdot\|_{n}\bigr)}
      \,\dd t,
      \label{eq:gauss-dudley}
\end{equation}
where $\mathrm{diam}(\mathcal{F},\|\cdot\|_n):=\sup_{f_1,f_2\in\mathcal{F}}\|f_1-f_2\|_n$.
Moreover, with probability at least $1-\delta$,
\begin{equation}
      \widehat{\GaussC}_{n}(\mathcal{F})
      \;\le\;
      \GaussC_{n}(\mathcal{F})
      + \mathrm{diam}(\mathcal{F},\|\cdot\|_{n})\,
        \frac{\sqrt{\log(1/\delta)}}{\sqrt{n}}
      \qquad(\delta>0).
      \label{eq:gauss-dudley-highprob}
\end{equation}
\end{proposition}

\begin{lemma}
\label{lem:generalization‐RBV}
Let $\mathcal{F}_{M,C}= \{f\in \Variation_C(\Bb_R^d)\mid \|f\|_{\infty} \leq M\}$ with $M\geq D$,
and let $\CD=\{(\vec{x}_i,y_i)\}_{i=1}^n\sim \CP^{\otimes n}$ where $|Y|\le D$ a.s.
Then with probability at least $1-\delta$,
\begin{equation}
\label{eq:gap-gauss-chaining}
\sup_{f\in\mathcal{F}_{M,C}}
\bigl|R(f)-\widehat R_{\CD}(f)\bigr|
\;\lesssim_d\;
C^{\frac{d}{d+3}}
\,M^{\frac{d+6}{d+3}}
\,n^{-\frac{1}{2}}
\;+\;
M^2\left(\frac{\log(1/\delta)}{n}\right)^{\frac{1}{2}}.
\end{equation}
\end{lemma}

\begin{proof}
Let $\ell_f(\vec{x},y):=(y-f(\vec{x}))^2$ and $\mathcal{L}_{M,C}:=\{\ell_f: f\in\mathcal{F}_{M,C}\}$.
Since $|y|\le D\le M$ and $\|f\|_\infty\le M$, for any $f,g\in\mathcal{F}_{M,C}$ and any $(\vec{x},y)$,
\[
\bigl|\ell_f(\vec{x},y)-\ell_g(\vec{x},y)\bigr|
  =|f(\vec{x})-g(\vec{x})|\,|f(\vec{x})+g(\vec{x})-2y|
  \le 4M\,|f(\vec{x})-g(\vec{x})|.
\]
Therefore, for the empirical norm on $\{\vec{x}_i\}_{i=1}^n$,
\begin{equation}\label{eq:loss-lip-emp}
\|\ell_f-\ell_g\|_n \le 4M\,\|f-g\|_n \le 4M\,\|f-g\|_\infty,
\end{equation}
and $\mathrm{diam}(\mathcal{L}_{M,C},\|\cdot\|_n)\le 8M^2$.

A standard Gaussian symmetrization/concentration argument (see, e.g.,
\citealt[Chapter~5]{wainwright2019high}) yields that with probability at least $1-\delta$,
\begin{equation}\label{eq:symm-to-gauss}
\sup_{f\in\mathcal{F}_{M,C}}
\bigl|R(f)-\widehat R_{\CD}(f)\bigr|
\;\lesssim\;
\widehat{\GaussC}_{n}(\mathcal{L}_{M,C})
\;+\;
M^2\left(\frac{\log(1/\delta)}{n}\right)^{\frac{1}{2}}.
\end{equation}
It remains to bound $\widehat{\GaussC}_{n}(\mathcal{L}_{M,C})$ by chaining.
By \eqref{eq:loss-lip-emp}, any $t$-cover of $\mathcal{F}_{M,C}$ in $\|\cdot\|_n$ induces a
$(4Mt)$-cover of $\mathcal{L}_{M,C}$ in $\|\cdot\|_n$, hence
\[
\log N\bigl(t,\mathcal{L}_{M,C},\|\cdot\|_n\bigr)
\le
\log N\Bigl(\frac{t}{4M},\mathcal{F}_{M,C},\|\cdot\|_n\Bigr)
\le
\log N\Bigl(\frac{t}{4M},\Variation_C(\Bb_R^d),\|\cdot\|_\infty\Bigr),
\]
where we used $\|h\|_n\le \|h\|_\infty$. Proposition~\ref{prop:metric_entropy_of_bounded_variation_space}
then gives, up to logarithmic factors,
\[
\log N\bigl(t,\mathcal{L}_{M,C},\|\cdot\|_n\bigr)
\;\lessapprox_d\;
\Bigl(\frac{MC}{t}\Bigr)^{\frac{2d}{d+3}}.
\]
Applying Proposition~\ref{prop:gauss-chaining} with $\mathcal{G}=\mathcal{L}_{M,C}$ and
$\mathrm{diam}(\mathcal{L}_{M,C},\|\cdot\|_n)\le 8M^2$ yields
\begin{align*}
\widehat{\GaussC}_{n}(\mathcal{L}_{M,C})
&\;\lesssim_d\;
\frac{1}{\sqrt{n}}
\int_{0}^{8M^2}
\Bigl(\frac{MC}{t}\Bigr)^{\frac{d}{d+3}}
\dd t
\;+\;
M^2\left(\frac{\log(1/\delta)}{n}\right)^{\frac{1}{2}} \\
&\;\asymp_d\;
C^{\frac{d}{d+3}}\,M^{\frac{d+6}{d+3}}\,n^{-\frac{1}{2}}
\;+\;
M^2\left(\frac{\log(1/\delta)}{n}\right)^{\frac{1}{2}}.
\end{align*}
Combining with \eqref{eq:symm-to-gauss} proves \eqref{eq:gap-gauss-chaining}.
\end{proof}

\section{Proof of Theorem \ref{thm:implicit_bias_CNN}}
\label{app:dd_regularity}
Let $\iota_j: \mathbb{R}^m\rightarrow\mathbb{R}^d$ be the dual embedding such that $ \pi_j\circ \iota_j=\operatorname{id}_{\Rb^m}$. Then \eqref{eq:model_CNN} is equivalent to a fully connected neural network in a form of
\begin{align}
	f_{\vec{\theta}}(\vec{x}) &\;=\; \sum_{k=1}^{K}\frac{1}{J}\sum_{j=1}^J v_k\,\phi\big(\iota_j(\vec{w}_k)^\T  \vec{x}- b_k\big) + \beta\label{eq:model_CNN_2}.
\end{align}
Note that for any fixed $k$, $\norm{\iota_j(\vec{w}_k)}_2=\|\vec{w}_k\|_2$ for all $j=1,\cdots,J$. Therefore, the notion of path norm and variation norm (together with their weighted version) still make sense for the CNN model
\begin{equation}\label{eq:path_norm_cnn}
	\pathnorm{f_{\vec{\theta}}}=\frac{1}{J}\sum_{k=1}^K\sum_{j=1}^J|v_k|\norm{\iota_j(\vec{w}_k)}_2=\sum_{k=1}^K|v_k|\norm{\vec{w}_k}_2.
\end{equation}

By direct computation, the Hessian matrix of the loss function is expressed as
\begin{equation}\label{eq:Hessian_1}
	\nabla^2_{\vec{\theta}}\loss =\underbrace{\frac{1}{n}\sum_{i=1}^n\nabla_{\vec{\theta}}f(\vec{x}_i)\nabla_{\vec{\theta}}f(\vec{x}_i)^{\T}}_{\displaystyle \mT_{\CD} } +\underbrace{\frac{1}{n}\sum_{i=1}^n(f(\vec{x}_i)-y_i)\nabla^2_{\vec{\theta}}f(\vec{x}_i)}_{\displaystyle \mR_\CD}.
\end{equation}
\begin{definition}\label{def:weight_function_CNN_appendix}
	Given a dataset $\CD=\{(\vec{x}_i,y_i)\}_{i=1}^n$ and a set of local receptive fields $\CS=\{S_j\}_{j=1}^J$, we define a random vector $\vec{X}^{\CS}_{\CD}$ uniformly draw from $\{\pi_j(\vec{x}_i)\}^{n\times J}_{(i,j)}\subset \Rb^m$. For any $\vec{u}\in \Sph^{m-1}, t\in \Rb$, we define the weight function \begin{equation}\label{eq:weight_g_CNN2_app}
g_{\CD,\mathcal{S}}(\vec{u},t)=\min\Big\{\tilde g_{\CD,\mathcal{S}}(\vec{u},t),\ \tilde g_{\CD,\mathcal{S}}(-\vec{u},-t)\Big\}.
\end{equation}
where
	\begin{equation}\label{eq:weight_g_CNN_app}
	g_{\CD,\mathcal{S}}(\vec{u},t):=\Eb\left[\phi\left(\vec{u}^\T\vec{X}^{\CS}_{\CD} - t\right)\right]\,
	\sqrt{\Pb\left(\vec{u}^\T\vec{X}^{\CS}_{\CD} > t \right)^{2}+\left\|\Eb\left[\vec{X}^{\CS}_{\CD}\,\mathds{1}\Big\{\vec{u}^\T\vec{X}^{\CS}_{\CD} > t\Big\}\right]\right\|_2^{2}}.
	\end{equation}
	\end{definition}

\begin{proposition}\label{prop:CNN_stability_induced_regularity}
Given a data set $\CD=\curly{\vec{x}_i,y_i}_{i=1}^n$ and a network model \eqref{eq:model_CNN} with local receptive fields $\mathcal{S}$.
Then we have
\begin{equation}\label{eq:weighted_path_norm_CNN_TFM}
	\lambda_{\max}\left(\mT_\CD\right)\geq 1+2\sum_{k=1}^K|v_k|\norm{\vec{w}_k}\,
	g_{\CD,\mathcal{S}}\left(\frac{\vec{w}_k}{\|\vec{w}_k\|},\frac{b_k}{\|\vec{w}_k\|} \right).
\end{equation}
\end{proposition}

\begin{proof}
We write $\mT_{\CD}$ in terms of tangent features
\[
\mT_{\CD}=\frac{1}{n}\sum_{i=1}^n \nabla_\vec{\theta} f(\vec{x}_i)\,\nabla_\vec{\theta} f(\vec{x}_i)^\T
\;=\;\frac{1}{n}\,\Phi\Phi^\T.
\]
Consequently,
\begin{equation}\label{eq:top-eig}
\lambda_{\max}\!\bigl(\mT_{\CD}\bigr)
\;=\;\max_{u\in\mathbb{S}^{n-1}}\frac{1}{n}\,\|\Phi u\|^2
\;\ge\;\frac{1}{n^2}\,\|\Phi\vec{1}\|^2.
\end{equation}

For any point $\vec{x}$ and abbreviate its patch extraction on $S_j$ by $\vec{x}^{(S_j)}\!:=\pi_j(\vec{x})$.
Define the ReLU gate
\[
	m^{(S_j)}_k(\vec{x}) \coloneqq \mathds{1}\curly{\vec{w}_k^\T \vec{x}^{(S_j)}>b_k}.
\]
For any sample $\vec{x}_i$, denote $m^{(S_j)}_{k,i}\!:=m^{(S_j)}_{k}(\vec{x}_i)$. Then the partial derivatives are
\[
\frac{\partial f(\vec{x}_i)}{\partial v_k}
= \frac{1}{J}\sum_{j=1}^J m^{(S_j)}_{k,i}\cdot\big(\vec{w}_k^\T \vec{x}_i^{(S_j)}-b_k\big),\quad
\frac{\partial f(\vec{x}_i)}{\partial \vec{w}_k}
= \frac{1}{J}\sum_{j=1}^J m^{(S_j)}_{k,i}\cdot v_k\cdot\vec{x}_i^{(S_j)},
\]
\[
\frac{\partial f(\vec{x}_i)}{\partial b_k}= -\frac{1}{J}\sum_{j=1}^J m^{(S_j)}_{k,i}\cdot\,v_k,\quad
\frac{\partial f(\vec{x}_i)}{\partial \beta}= 1.
\]

Stacking these over $i$ and plugging $u=\vec{1}/\sqrt{n}$ in \eqref{eq:top-eig}, we get
\begin{align}
\frac{1}{n^2}\,\|\Phi\vec{1}\|^2
&= 1 + \frac{1}{n^2}\sum_{k=1}^K\Bigg[
(v_k)^2\!\left(
\left\|
\sum_{i=1}^n\frac{1}{J}\sum_{j=1}^J m^{(S_j)}_{k,i} \vec{x}_i^{(S_j)}
\right\|^2
+\left(\sum_{i=1}^n \frac{1}{J}\sum_{j=1}^J m^{(S_j)}_{k,i}\right)^2\right)\notag\\
&\hspace{3.0cm}
+\left(\sum_{i=1}^n \frac{1}{J}\sum_{j=1}^J \phi\left(\vec{w}_k^\T \vec{x}_i^{(S_j)}-b_k\right)\right)^2
\Bigg]\notag\\
&= 1 + \frac{1}{n^2}\sum_{k=1}^K\Bigg[
(v_k)^2\!\left(
\left\|
\sum_{i=1}^n\frac{1}{J}\sum_{j=1}^J m^{(S_j)}_{k,i} \vec{x}_i^{(S_j)}
\right\|^2
+\left(\sum_{i=1}^n \frac{1}{J}\sum_{j=1}^J m^{(S_j)}_{k,i}\right)^2\right)\notag\\
&\hspace{3.0cm}
+\|\vec w_k\|_2^2\left(\sum_{i=1}^n \frac{1}{J}\sum_{j=1}^J \phi\left(\vec{u}_k^\T \vec{x}_i^{(S_j)}-t_k\right)\right)^2
\Bigg]\label{eq:CNN_refined_step_no_mu}\\
&\ge 1 + \frac{2}{n^2}\sum_{k=1}^K
|v_k|\,\|\vec w_k\|_2\,
\left(\sum_{i=1}^n \frac{1}{J}\sum_{j=1}^J \phi\left(\vec{u}_k^\T \vec{x}_i^{(S_j)}-t_k\right)\right)\notag\\
&\hspace{2.4cm}\cdot
\sqrt{
\left\|
\sum_{i=1}^n\frac{1}{J}\sum_{j=1}^J m^{(S_j)}_{k,i} \vec{x}_i^{(S_j)}
\right\|^2
+\left(\sum_{i=1}^n \frac{1}{J}\sum_{j=1}^J m^{(S_j)}_{k,i}\right)^2
}\quad\text{(since $a^2+b^2\ge 2ab$)}\notag\\
&= 1 + 2\sum_{k=1}^K
|v_k|\,\|\vec w_k\|_2\,
\left(\frac{1}{nJ}\sum_{i=1}^n \sum_{j=1}^J \phi\left(\vec{u}_k^\T \vec{x}_i^{(S_j)}-t_k\right)\right)\notag\\
&\hspace{2.4cm}\cdot
\sqrt{
\left\|
\frac{1}{nJ}\sum_{i=1}^n\sum_{j=1}^J m^{(S_j)}_{k,i} \vec{x}_i^{(S_j)}
\right\|^2
+\left(\frac{1}{nJ}\sum_{i=1}^n \sum_{j=1}^J m^{(S_j)}_{k,i}\right)^2
}\notag\\
&= 1 + 2\sum_{k=1}^K
|v_k|\,\|\vec w_k\|_2\,
\Eb\left[\phi\left(\vec{u}^\T\vec{X}^{\CS}_{\CD} - t\right)\right] \sqrt{\Pb\left(\vec{u}^\T\vec{X}^{\CS}_{\CD} > t \right)^{2}+\left\|\Eb\left[\vec{X}^{\CS}_{\CD}\,\mathds{1}\Big\{\vec{u}^\T\vec{X}^{\CS}_{\CD} > t\Big\}\right]\right\|_2^{2}}\notag\\
&= 1 + 2\sum_{k=1}^K
|v_k|\,\|\vec w_k\|_2\,
g_{\CD,\CS}(\vec u_k,t_k),\qquad
\vec u_k=\frac{\vec w_k}{\|\vec w_k\|_2},\quad t_k=\frac{b_k}{\|\vec w_k\|_2}.\notag
\end{align}

Then combining \eqref{eq:top-eig} and the above inequality yields the claim.
\end{proof}

\begin{lemma}\label{lemma:cnn_upper_bound_operator_norm}
Consider the model \eqref{eq:model_CNN}
\[
f(\vec{x}) \;=\; \sum_{k=1}^{K} \frac{v_k}{J} \sum_{j=1}^{J} \phi\left(\vec{w}_k^\T \pi_j(\vec{x}) - b_k\right) + \beta.
\]
where the input satisfies $\|\vec{x}\|_2\le R$, each patch extractor $\pi_j:\mathbb{R}^d\to\mathbb{R}^{m}$ is a coordinate projection, and $\phi(t)=\max\{0,t\}$.
Let $\vec{\theta}=\big(\vec{w}_1^{\T},\dots,\vec{w}_K^{\T},\, b_1,\dots,b_K,\, v_1,\dots,v_K,\, \beta\big)^{\T}$ collect all parameters.
Assume $f_{\vec{\theta}}(\vec{x})$ is twice differentiable with respect to $\vec{\theta}$ at $\vec{x}$, i.e., for all $k$ and $S_j\in \CS$ we have $\vec{w}_k^{\T}\vec{x}^{(S_j)}\neq b_k$.
Then for any perturbation vector $\vec{\omega}$ with $\|\vec{\omega}\|_2=1$, it holds that
\[
\big|\vec{\omega}^{\T}\,\nabla_{\vec{\theta}}^{2} f_{\vec{\theta}}(\vec{x})\,\vec{\omega}\big|\;\le\; 2\,(R+1).
\]

\end{lemma}

\begin{proof}
Write $\vec{\theta}=\big(\vec{w}_1^{\T},\dots,\vec{w}_K^{\T},\, b_1,\dots,b_K,\, v_1,\dots,v_K,\, \beta\big)^{\T}$.
The total number of parameters is $N=K\cdot m + K + K + 1 = K(m+2)+1$.

Let the corresponding perturbation vector be
\[
\vec{\omega}=\big(\vec{\alpha}_1^{\T},\dots,\vec{\alpha}_K^{\T},\, \delta_1,\dots,\delta_K,\, \gamma_1,\dots,\gamma_K,\, \iota\big)^{\T}\in\mathbb{R}^{N},
\]
where $\vec{\alpha}_k\in\mathbb{R}^{m}$ corresponds to $\vec{w}_k$, $\delta_k\in\mathbb{R}$ to $b_k$, $\gamma_k\in\mathbb{R}$ to $v_k$, and $\iota\in\mathbb{R}$ to $\beta$.
The normalization constraint is
\[
\|\vec{\omega}\|_2^2=\sum_{k=1}^{K}\|\vec{\alpha}_k\|_2^2+\sum_{k=1}^{K}\delta_k^2+\sum_{k=1}^{K}\gamma_k^2+\iota^2=1.
\]

For the fixed input $\vec{x}$, set $\vec{x}^{(S_j)}:=\pi_j(\vec{x})\in\mathbb{R}^{m}$ and define the ReLU gate
\[
m^{(S_j)}_k \;:=\; \mathds{1}\!\big\{\vec{w}_k^{\T}\vec{x}^{(S_j)}> b_k\big\}\in\{0,1\}.
\]
By the twice-differentiability assumption, all gates are constant in a neighborhood of $\vec{\theta}$.

Within this gate-fixed region, $f_{\vec{\theta}}$ is affine in $(\vec{w}_k,b_k)$ once $v_k$ is held fixed, and affine in $v_k$ once $(\vec{w}_k,b_k)$ are held fixed.
Therefore the only nonzero second partial derivatives inside the $k$-th neuron block are the mixed ones with $v_k$
\begin{align}
\frac{\partial^2 f_{\vec{\theta}}}{\partial \vec{w}_k\,\partial v_k}
&= \frac{\partial}{\partial v_k}\!\left(\frac{1}{J}\sum_{S_j\in \CS} v_k\,m^{(S_j)}_k\,\vec{x}^{(S_j)}\right)
= \frac{1}{J}\sum_{j=1}^J m^{(S_j)}_k\,\vec{x}^{(S_j)}
=: \vec{s}_k\in\mathbb{R}^{m},
\label{eq:cnn_hessian_non_zero_term_1}\\
\frac{\partial^2 f_{\vec{\theta}}}{\partial b_k\,\partial v_k}
&= \frac{\partial}{\partial v_k}\!\left(\frac{1}{J}\sum_{j=1}^J v_k\,m^{(S_j)}_k\right)
= -\frac{1}{J}\sum_{j=1}^J m^{(S_j)}_k
=:t_k\in\mathbb{R}\label{eq:cnn_hessian_non_zero_term_2}.
\end{align}
All other second derivatives inside the block vanish (as do any cross-neuron second derivatives and all those involving $\beta$).
Hence, with the block variable $\theta_k:=(\vec{w}_k^{\T},\,b_k,\,v_k)^{\T}$, we have
\begin{equation}
	\nabla_{(\theta_k)}^{2} f_{\vec{\theta}}(\vec{x})
=
\begin{pmatrix}
\mathbf{0}_{m\times m } & \vec{0}_m & \vec{s}_k \\
\vec{0}_m^{\T} & 0 & t_k \\
\vec{s}_k^{\T} & t_k & 0
\end{pmatrix}.
\end{equation}

The full quadratic form splits over neuron blocks:
\begin{align}
	\vec{\omega}^{\T}\,\nabla_{\vec{\theta}}^{2} f_{\vec{\theta}}(\vec{x})\,\vec{\omega}
&=\sum_{k=1}^{K}
\begin{pmatrix}\vec{\alpha}_k^{\T} & \delta_k & \gamma_k\end{pmatrix}
\begin{pmatrix}
\mathbf{0} & \vec{0} & \vec{s}_k \\
\vec{0}^{\T} & 0 & t_k \\
\vec{s}_k^{\T} & t_k & 0
\end{pmatrix}
\begin{pmatrix}\vec{\alpha}_k \\ \delta_k \\ \gamma_k\end{pmatrix}\nonumber \\
&= \sum_{k=1}^{K} 2\,\gamma_k\big(\vec{\alpha}_k^{\T}\vec{s}_k+\delta_k\,t_k\big)\label{eq:operator_norm_expansion_cnn}.
\end{align}

Recall the definition of the notations \eqref{eq:cnn_hessian_non_zero_term_1} and \eqref{eq:cnn_hessian_non_zero_term_2},
\begin{align}
|t_k|=\frac{1}{J}\sum_{j=1}^J m^{(S_j)}_k &\leq 1,\label{ineq:tk_CNN}\\
\|\vec{s}_k\|_2=\norm{\frac{1}{J}\sum_{j=1}^J m^{(S_j)}_k\,\vec{x}^{(S_j)}}_2 
\le \frac{1}{J}\sum_{j=1}^J m^{(S_j)}_k\,\|\vec{x}^{(S_j)}\|_2
&\le \frac{1}{J}\sum_{j=1}^J m^{(S_j)}_k\,R
= R|t_k|,\label{ineq:sk_CNN}
\end{align}
because $\vec{x}^{(S_j)}=\pi_j(\vec{x})$ is a coordinate projection of $\vec{x}$,
 $\|\vec{x}^{(S_j)}\|_2\le \|\vec{x}\|_2\le R$.

Then for each $k$,
\begin{align}
	\big|2\,\gamma_k(\vec{\alpha}_k^{\T}\vec{s}_k+\delta_k\,t_k)\big|
&\le 2\,|\gamma_k|\Big(\|\vec{\alpha}_k\|_2\,\|\vec{s}_k\|_2+|\delta_k|\,t_k\Big)\nonumber
\quad\text{(Cauchy-Schwarz)}\\
\text{\eqref{ineq:sk_CNN}$\implies$ }&\le 2\,|\gamma_k|\Big(R\,t_k\,\|\vec{\alpha}_k\|_2+|\delta_k|\,t_k\Big)\nonumber\\
\text{\eqref{ineq:tk_CNN}$\implies$ }&\le 2\,|\gamma_k|\Big(R\,\|\vec{\alpha}_k\|_2+|\delta_k|\Big),\label{ineq:neuron_wise_hessian_inequality}
\end{align}
Summing over $k$ for \eqref{ineq:neuron_wise_hessian_inequality} and plug in \eqref{eq:operator_norm_expansion_cnn}, we have
\begin{align}
\big|\vec{\omega}^{\T}\,\nabla_{\vec{\theta}}^{2} f_{\vec{\theta}}(\vec{x})\,\vec{\omega}\big|
&\le 2\Big(R\,\sum_{k} |\gamma_k|\,\|\vec{\alpha}_k\|_2 + \sum_{k} |\gamma_k|\,|\delta_k|\Big)\nonumber\\
\text{(Cauchy-Schwarz)}&\le 2\Big(R\,\sqrt{\textstyle\sum_k \gamma_k^2}\sqrt{\textstyle\sum_k \|\vec{\alpha}_k\|_2^2}
+ \sqrt{\textstyle\sum_k \gamma_k^2}\sqrt{\textstyle\sum_k \delta_k^2}\Big).
\end{align}
Using the normalization $\sum_k \|\vec{\alpha}_k\|_2^2+\sum_k \delta_k^2+\sum_k \gamma_k^2+\iota^2=1$, we have
$\sqrt{\sum_k \gamma_k^2}\le 1$, $\sqrt{\sum_k \|\vec{\alpha}_k\|_2^2}\le 1$, and $\sqrt{\sum_k \delta_k^2}\le 1$.
Thus $\big|\vec{\omega}^{\T}\,\nabla_{\vec{\theta}}^{2} f_{\vec{\theta}}(\vec{x})\,\vec{\omega}\big|
\le 2\,(R+1)$.
\end{proof}

\begin{theorem}[Restate Theorem \ref{thm:implicit_bias_CNN}]
Given a data set $\CD=\curly{\vec{x}_i,y_i}_{i=1}^n$ and a network model \eqref{eq:model_CNN} with the set of local receptive fields $\mathcal{S}$, we have
\[
\sum_{k=1}^K|v_k|\norm{\vec{w}_k} g_{\CD,\mathcal{S}}\left(\frac{\vec{w}_k}{\|\vec{w}_k\|},\frac{b_k}{\|\vec{w}_k\|} \right)\;\leq\; \frac{1}{2}
\left(\lambda_{\max}(\nabla^2_{\vec{\theta}}\mathcal{L}(\vec{\theta}))+2(R+1)\sqrt{2\loss(\vec{\theta})}-1\right).
\]	
\end{theorem}
\begin{proof}
	
It suffices to prove the first assertion. Recall that by direct computation, the Hessian matrix of the loss function is expressed as
\begin{equation}\label{eq:Hessian_2}
	\nabla^2_{\vec{\theta}}\loss =\underbrace{\frac{1}{n}\sum_{i=1}^n\nabla_{\vec{\theta}}f(\vec{x}_i)\nabla_{\vec{\theta}}f(\vec{x}_i)^{\T}}_{\displaystyle \mT_{\CD} } +\underbrace{\frac{1}{n}\sum_{i=1}^n(f(\vec{x}_i)-y_i)\nabla^2_{\vec{\theta}}f(\vec{x}_i)}_{\displaystyle \mR_\CD}.
\end{equation}

Let $\vec{\omega}$ be the unit eigenvector (i.e., $\|\vec{\omega}\|_2=1$) corresponding to the largest eigenvalue of the matrix $\mT_{\CD}$, the maximum eigenvalue of the Hessian matrix of the loss can be lower-bounded as follows:
\begin{equation}\label{equ:qd3}
\begin{aligned}
\lambda_{\max}(\nabla^2_{\vec{\theta}}\mathcal{L}(\vec{\theta})) &\;\geq\; \vec{\omega}^{\T} \nabla^2_{\vec{\theta}}\mathcal{L}(\vec{\theta}) \vec{\omega} \\
&\;\geq\; \underbrace{\lambda_{\max}(\mT_{\CD})}_{\displaystyle \text{(Term A)}} + \underbrace{\frac{1}{n}\sum_{i=1}^n(f_{\vec{\theta}}(\vec{x}_i)-y_i)\vec{\omega}^{\T} \nabla^2_{\vec{\theta}}f_{\vec{\theta}}(\vec{x}_i)\vec{\omega}}_{\displaystyle \text{(Term B)}}.
\end{aligned}
\end{equation}
According to Proposition \ref{prop:CNN_stability_induced_regularity}, Term A is lower bounded by
\begin{equation}\label{equ:qd1}
	\lambda_{\max}\left(\mT_\CD\right)\geq 1+2\sum_{k=1}^K|v_k|\norm{\vec{w}_k} g_{\CD,\mathcal{S}}\left(\frac{\vec{w}_k}{\|\vec{w}_k\|},\frac{b_k}{\|\vec{w}_k\|} \right),
\end{equation}

For (Term B), an upper bound can be established using the training loss $\mathcal{L}(\vec{\theta})$ via the Cauchy-Schwarz inequality. This also employs a notable uniform upper bound for $|\vec{\omega}^{\T} \nabla^2_{\vec{\theta}}f_{\vec{\theta}}(\vec{x}_n)\vec{\omega}|$, as detailed in Lemma \ref{lemma:cnn_upper_bound_operator_norm}:
\begin{equation}
|\text{(Term B)}| \leq \sqrt{\frac{1}{n}\sum_{i=1}^n\left(f_{\vec{\theta}}(\vec{x}_i)-y_i\right)^2} \cdot \sqrt{\frac{1}{n}\sum_{i=1}^n\left(\vec{\omega}^\T \nabla^2_{\vec{\theta}}f_{\vec{\theta}}(\vec{x}_i)\vec{\omega}\right)^2} \leq 2(R+1)\sqrt{2\loss(\vec{\theta})}.
\end{equation}
Thus, we have
\begin{equation}\label{equ:qd2}
	\lambda_{\max}\left(\mT_\CD\right)\leq \lambda_{\max}(\nabla^2_{\vec{\theta}}\mathcal{L}(\vec{\theta}))+|\text{(Term B)}|\leq \lambda_{\max}(\nabla^2_{\vec{\theta}}\mathcal{L}(\vec{\theta}))+2(R+1)\sqrt{2\loss(\vec{\theta})}.
\end{equation}

Finally, we plug \eqref{equ:qd1} into \eqref{equ:qd2} 
\[
	\sum_{k=1}^K|v_k|\norm{\vec{w}_k} g_{\CD,\mathcal{S}}\left(\frac{\vec{w}_k}{\|\vec{w}_k\|},\frac{b_k}{\|\vec{w}_k\|} \right)\;\leq\; \frac{1}{2}
\left(\lambda_{\max}(\nabla^2_{\vec{\theta}}\mathcal{L}(\vec{\theta}))+2(R+1)\sqrt{2\loss(\vec{\theta})}-1\right).
\]
\end{proof}
\subsection{Empirical Process for the Weight Function}
We study uniform deviations of the empirical weight function $g_{\CD,\CS}$ from its population counterpart $g_{\CP,\CS}$ under i.i.d.\ sampling $\vec{x}_1,\dots,\vec{x}_n\sim \CP_{\vec{X}}$. 
Although the collection $\{\pi_j(\vec{x}_i)\}_{i,j}$ is not i.i.d.\ across $(i,j)$ in general, for each fixed $j$ the patches $\{\pi_j(\vec{x}_i)\}_{i=1}^n$ are i.i.d.\ in $\Rb^m$.

\begin{definition}\label{def:weight_function_CNN_refined}
	For any $S_j\in \mathcal{S}$, let $\vec{X}^{(S_j)}_{\CD}$ be a random vector drawn uniformly at random from the training examples $\{\pi_j(\vec{x}_i)\}_{i=1}^n$ with patch extraction to the local receptive field $S_j$.
	For any $\vec{u}\in \Sph^{m-1}, t\in \Rb$, define the aggregated gate probability, ReLU margin, and gated first moment
	\begin{equation}\label{eq:weightfunction_CNN_refined}
	\begin{aligned}
		\bar p_{\CD,\mathcal{S}}(\vec{u},t)
		&:=\frac{1}{J}\sum_{j=1}^J\Pb\left(\vec{u}^\T\vec{X}^{(S_j)}_{\CD} > t \right),\\
		\bar r_{\CD,\mathcal{S}}(\vec{u},t)
		&:=\frac{1}{J}\sum_{j=1}^J\Eb\left[\phi\left(\vec{u}^\T\vec{X}^{(S_j)}_{\CD} - t\right)\right],\\
		\bar{\vec a}_{\CD,\mathcal{S}}(\vec{u},t)
		&:=\frac{1}{J}\sum_{j=1}^J\Eb\left[\vec{X}^{(S_j)}_{\CD}\,\mathds{1}\Big\{\vec{u}^\T\vec{X}^{(S_j)}_{\CD} > t\Big\}\right]\in\Rb^m.
	\end{aligned}
	\end{equation}
\end{definition}

Based on these quantities, define the empirical weight function
\begin{equation}\label{eq:refined_g_CNN}
	g_{\CD,\mathcal{S}}(\vec{u},t)
	:=\bar r_{\CD,\mathcal{S}}(\vec{u},t)\,
	\sqrt{\bar p_{\CD,\mathcal{S}}(\vec{u},t)^{2}+\bigl\|\bar{\vec a}_{\CD,\mathcal{S}}(\vec{u},t)\bigr\|_2^{2}}.
\end{equation}

For each $j\in[J]$ and $(\vec{u},t)\in\Sph^{m-1}\times[-1,1]$, define the population components
\[
p_j(\vec{u},t):=\Pb\!\left(\vec{u}^\T \pi_j(\vec{X})>t\right),\quad
r_j(\vec{u},t):=\Eb\!\left[\phi(\vec{u}^\T \pi_j(\vec{X})-t)\right],\quad
\vec{a}_j(\vec{u},t):=\Eb\!\left[\pi_j(\vec{X})\,\mathds{1}\{\vec{u}^\T \pi_j(\vec{X})>t\}\right].
\]
Let $\bar p_{\CP,\CS}$, $\bar r_{\CP,\CS}$, and $\bar{\vec a}_{\CP,\CS}$ denote their averages over $j\in[J]$. Define
\[
g_{\CP,\CS}(\vec{u},t):=\bar r_{\CP,\CS}(\vec{u},t)\sqrt{\bar p_{\CP,\CS}(\vec{u},t)^2+\|\bar{\vec a}_{\CP,\CS}(\vec{u},t)\|_2^2}.
\]

\begin{lemma}[Complexity of the component classes]\label{lem:g-component-complexity}
Under the constraint $\|\vec z\|_2\le1$, the following classes are uniformly bounded VC-subgraph classes with VC-subgraph dimension $O(m)$:
\[
\mathcal H:=\left\{\vec z\mapsto \mathds{1}\{\vec u^\T\vec z>t\}:\vec u\in\Sph^{m-1},\,t\in[-1,1]\right\},
\]
\[
\mathcal F_{\mathrm{ReLU}}:=\left\{\vec z\mapsto \phi(\vec u^\T\vec z-t):\vec u\in\Sph^{m-1},\,t\in[-1,1]\right\},
\]
and
\[
\mathcal A:=\left\{\vec z\mapsto \vec v^\T\vec z\,\mathds{1}\{\vec u^\T\vec z>t\}:\vec u,\vec v\in\Sph^{m-1},\,t\in[-1,1]\right\}.
\]
Moreover, $\mathcal H$ and $\mathcal A$ are bounded by $1$, and $\mathcal F_{\mathrm{ReLU}}$ is bounded by $2$.
\end{lemma}

\begin{proof}
The class $\mathcal H$ is the class of halfspace indicators in $\Rb^m$, hence has VC-dimension at most $m+1$.

For $\mathcal F_{\mathrm{ReLU}}$, the subgraph of $\vec z\mapsto\phi(\vec u^\T\vec z-t)$ is
\[
\{(\vec z,s):s<\phi(\vec u^\T\vec z-t)\}
=
\bigl(\{\vec u^\T\vec z>t\}\cap\{s<\vec u^\T\vec z-t\}\bigr)
\cup
\bigl(\{\vec u^\T\vec z\le t\}\cap\{s<0\}\bigr).
\]
For $\mathcal A$, the subgraph of $\vec z\mapsto \vec v^\T\vec z\,\mathds{1}\{\vec u^\T\vec z>t\}$ is
\[
\{(\vec z,s):s<\vec v^\T\vec z\,\mathds{1}\{\vec u^\T\vec z>t\}\}
=
\bigl(\{\vec u^\T\vec z>t\}\cap\{s<\vec v^\T\vec z\}\bigr)
\cup
\bigl(\{\vec u^\T\vec z\le t\}\cap\{s<0\}\bigr).
\]
Both subgraph classes are constant-size Boolean combinations of halfspaces in $\Rb^{m+1}$. By the standard growth-function closure argument for finite Boolean combinations of VC classes, their VC dimensions are $O(m)$. Equivalently, their underlying real-valued function classes have pseudo-dimension $O(m)$, see \cite[Ch.~3 and Definition~11.2]{anthonyBartlett1999}. The boundedness claims follow from $\|\vec z\|_2\le1$, $t\in[-1,1]$, and $\|\vec u\|_2=\|\vec v\|_2=1$.
\end{proof}

\begin{theorem}[Uniform deviation for $g_{\CD,\CS}$]\label{thm:g-deviation}
Assume $\|\pi_j(\vec{X})\|_2\le 1$ almost surely for all $j\in[J]$. 
There exists a universal constant $C_{\mathrm{ep}}>0$ such that for every $\delta\in(0,1)$, with probability at least $1-\delta$,
\[
\sup_{\vec{u}\in\Sph^{m-1},\, t\in[-1,1]}
\big|g_{\CD,\CS}(\vec{u},t)-g_{\CP,\CS}(\vec{u},t)\big|
\le
C_{\mathrm{ep}}\sqrt{\frac{m+\log(2J/\delta)}{n}}.
\]
\end{theorem}

\begin{proof}
For each fixed $j\in[J]$, let $P_j$ denote the law of $\vec Z_j:=\pi_j(\vec X)\in\Rb^m$, and let $P_{n,j}:=n^{-1}\sum_{i=1}^n\delta_{\pi_j(\vec x_i)}$. Define
\[
\hat p_j(\vec u,t):=P_{n,j}\mathds{1}\{\vec u^\T\vec z>t\},\qquad
\hat r_j(\vec u,t):=P_{n,j}\phi(\vec u^\T\vec z-t),\qquad
\hat{\vec a}_j(\vec u,t):=P_{n,j}\!\left[\vec z\,\mathds{1}\{\vec u^\T\vec z>t\}\right].
\]
By the VC-subgraph entropy bound and maximal inequality \cite[Theorems~2.6.7 and~2.14.1]{vanDerVaartWellner1996}, together with McDiarmid's inequality, any class $\mathcal F$ uniformly bounded by $B$ with VC-subgraph dimension $V$ satisfies
\[
\sup_{f\in\mathcal F}|(P_n-P)f|\lesssim B\sqrt{\frac{V+\log(1/\eta)}{n}}
\]
with probability at least $1-\eta$.

Applying this bound to $\mathcal H$ and $\mathcal F_{\mathrm{ReLU}}$ from Lemma~\ref{lem:g-component-complexity}, for every fixed $j$ and $\eta\in(0,1)$, with probability at least $1-\eta$,
\[
\sup_{\vec u,t}|\hat p_j(\vec u,t)-p_j(\vec u,t)|\lesssim \sqrt{\frac{m+\log(1/\eta)}{n}},\qquad
\sup_{\vec u,t}|\hat r_j(\vec u,t)-r_j(\vec u,t)|\lesssim \sqrt{\frac{m+\log(1/\eta)}{n}},
\]
where all suprema are over $\vec u\in\Sph^{m-1}$ and $t\in[-1,1]$.

For the vector moment, use duality:
\[
\begin{aligned}
\sup_{\vec u,t}\|\hat{\vec a}_j(\vec u,t)-\vec a_j(\vec u,t)\|_2
&=
\sup_{\vec u,t}\sup_{\vec v\in\Sph^{m-1}}
\left|\vec v^\T\bigl(\hat{\vec a}_j(\vec u,t)-\vec a_j(\vec u,t)\bigr)\right|\\
&=
\sup_{\vec u,\vec v,t}
\left|(P_{n,j}-P_j)\left[\vec v^\T\vec z\,\mathds{1}\{\vec u^\T\vec z>t\}\right]\right|.
\end{aligned}
\]
The last supremum is over the class $\mathcal A$ from Lemma~\ref{lem:g-component-complexity}. Hence, with probability at least $1-\eta$,
\[
\sup_{\vec u,t}\|\hat{\vec a}_j(\vec u,t)-\vec a_j(\vec u,t)\|_2
\lesssim
\sqrt{\frac{m+\log(1/\eta)}{n}}.
\]

Taking $\eta=\delta/(3J)$ and applying a union bound over the three component bounds and over $j\in[J]$, we obtain that, with probability at least $1-\delta$, simultaneously for all $j\in[J]$,
\[
\sup_{\vec u,t}|\hat p_j(\vec u,t)-p_j(\vec u,t)|\le\varepsilon,\qquad
\sup_{\vec u,t}|\hat r_j(\vec u,t)-r_j(\vec u,t)|\le\varepsilon,\qquad
\sup_{\vec u,t}\|\hat{\vec a}_j(\vec u,t)-\vec a_j(\vec u,t)\|_2\le\varepsilon,
\]
where all suprema are over $\vec u\in\Sph^{m-1}$ and $t\in[-1,1]$, and $\varepsilon:=C_0\sqrt{(m+\log(3J/\delta))/n}$ for a universal constant $C_0>0$. Averaging over $j$ does not increase the deviations. Therefore, uniformly over $(\vec u,t)\in\Sph^{m-1}\times[-1,1]$,
\[
|\bar p_{\CD,\CS}(\vec u,t)-\bar p_{\CP,\CS}(\vec u,t)|\le\varepsilon,\qquad
|\bar r_{\CD,\CS}(\vec u,t)-\bar r_{\CP,\CS}(\vec u,t)|\le\varepsilon,\qquad
\|\bar{\vec a}_{\CD,\CS}(\vec u,t)-\bar{\vec a}_{\CP,\CS}(\vec u,t)\|_2\le\varepsilon.
\]

It remains to pass from the components to $g$. For both empirical and population quantities,
\[
0\le\bar p\le1,\qquad 0\le\bar r\le2,\qquad \|\bar{\vec a}\|_2\le\bar p\le1.
\]
Define $F(r,p,\vec a):=r\sqrt{p^2+\|\vec a\|_2^2}$. On this domain,
\[
|F(r,p,\vec a)-F(r',p',\vec a')|
\le
\sqrt{2}\,|r-r'|+2\sqrt{(p-p')^2+\|\vec a-\vec a'\|_2^2}.
\]
Thus, uniformly over $(\vec u,t)\in\Sph^{m-1}\times[-1,1]$,
\[
\big|g_{\CD,\CS}(\vec u,t)-g_{\CP,\CS}(\vec u,t)\big|
\lesssim
\varepsilon
\lesssim
\sqrt{\frac{m+\log(2J/\delta)}{n}},
\]
after adjusting the universal constant. This proves the theorem.
\end{proof}
\subsection{The Computation of the Population Weight Functions for Uniform Distributions on Spheres} 
Assume that $\CP_{\vec{X}}$ is a uniform distribution on $\Sph^{d-1}\subset \Rb^d$. Here we compute the population version of the components in the weighted function $g_{\CD,\mathcal{S}}$ defined in Definition \ref{def:weight_function_CNN_refined}. 

Let $d\ge 2$, $m<d$. Draw $\vec{X}\sim \mathrm{Uniform}(\Sph^{d-1})\subset\Rb^d$, set $\vec{Z}=\proj(\vec{X})\in\Rb^m$ by projecting the first $m$ coordinates, and fix any unit $\vec{u}\in\Rb^m$.
Define
\begin{equation}
Z:=\vec{u}^{\T}\vec{Z}\in[-1,1].\end{equation}
By rotational invariance in $\Rb^d$, $Z$ has the one–coordinate marginal of $\mathrm{Uniform}(\Sph^{d-1})$, hence its density is
\begin{equation}
f_d(z)
= c_d\,(1-z^2)^{\frac{d-3}{2}}\mathds{1}{\{-1\le z\le 1\}},
\qquad 
c_d:=\frac{\Gamma\!\left(\frac d2\right)}{\sqrt{\pi}\,\Gamma\!\left(\frac{d-1}{2}\right)}.
\label{eq:fd}
\end{equation}
All bounds below are independent of $m$ and $\vec{u}$, depending only on $d$.

For shorthand, write $\alpha:=\frac{d-3}{2}\ge -\tfrac12$ and observe
\begin{equation}
(1-z^2)^\alpha = \big[(1-z)(1+z)\big]^\alpha,\qquad z\in[-1,1],
\label{eq:fact}
\end{equation}
together with the elementary bounds $1\le 1+z \le 2$ for $z\in[0,1]$.

\begin{lemma}[(Tail of $Z$)]\label{lemma:CNN_Tail_P}
For all $t\in[0,1)$,
\begin{equation}
\frac{c_d}{\alpha+1}\,(1-t)^{\alpha+1} 
\;\le\;
\Pb(Z>t)
\;\le\;
\frac{c_d\,2^{\alpha}}{\alpha+1}\,(1-t)^{\alpha+1}.
\label{eq:tail-bds}
\end{equation}
Consequently $\Pb(Z>t)\asymp (1-t)^{\frac{d-1}{2}}$ as $t\rightarrow 1^-$.
\end{lemma}
\begin{proof}
Using \eqref{eq:fd} and \eqref{eq:fact},
\begin{align}
\Pb(Z>t)
&=\int_t^1 c_d\,(1-z^2)^{\alpha}\,dz
= c_d\int_t^1 \big[(1-z)(1+z)\big]^{\alpha}\,dz.
\end{align}
Since $1\le 1+z\le 2$ for $z\in[t,1]$, we have
\begin{equation}
c_d\int_t^1 (1-z)^{\alpha}\,dz
\;\le\;
\Pb(Z>t)
\;\le\;
c_d\,2^{\alpha}\int_t^1 (1-z)^{\alpha}\,dz.
\end{equation}
Evaluating $\int_t^1 (1-z)^{\alpha}\,dz = \frac{(1-t)^{\alpha+1}}{\alpha+1}$ (valid for $\alpha>-1$; here $\alpha\ge -\tfrac12$), we obtain \eqref{eq:tail-bds}. 
\end{proof}
\begin{lemma}
For all $t\in[0,1)$,
\begin{equation}
\frac{c_d}{(\alpha+1)(\alpha+2)}\,(1-t)^{\alpha+2}
\;\le\;
\Eb\big[\phi(Z-t)\big]
\;\le\;
\frac{c_d\,2^{\alpha}}{(\alpha+1)(\alpha+2)}\,(1-t)^{\alpha+2}.
\label{eq:relu-bds}
\end{equation}
Consequently $\Eb[\phi(Z-t)]\asymp (1-t)^{\frac{d+1}{2}}$ as $t\uparrow 1$.
\end{lemma}
\begin{proof}
By definition and \eqref{eq:fd},
\begin{align}
\Eb[\phi(Z-t)]
&=\int_t^1 (z-t) f_d(z)\,dz
= c_d \int_t^1 (z-t)\big[(1-z)(1+z)\big]^{\alpha}\,dz.
\end{align}
Bounding $1\le 1+z\le 2$ for $z\in[t,1]$ yields
\begin{equation}
c_d\int_t^1 (z-t)(1-z)^{\alpha}\,dz
\;\le\;
\Eb[\phi(Z-t)]
\;\le\;
c_d\,2^{\alpha}\int_t^1 (z-t)(1-z)^{\alpha}\,dz.
\end{equation}
Substitute $y=1-z$ to compute the shared integral:
\begin{align}
\int_t^1 (z-t)(1-z)^{\alpha}\,dz
&=\int_{0}^{1-t} \big[(1-t)-y\big]\,y^{\alpha}\,dy
=\frac{(1-t)^{\alpha+2}}{(\alpha+1)(\alpha+2)},
\end{align}
valid for $\alpha>-1$. This gives \eqref{eq:relu-bds}. 	
\end{proof}

\begin{proposition}\label{prop:CNN_weight_function_population}
There exist absolute constants (depending only on $d$)
\begin{equation}
c_L(d):=\frac{c_d^2}{\big(\alpha+1\big)^2\big(\alpha+2\big)},
\qquad
c_U(d):=\frac{c_d^2\,2^{\,2\alpha}}{\big(\alpha+1\big)^2\big(\alpha+2\big)},
\label{eq:consts}
\end{equation}
such that for all $t\in[0,1)$,
\begin{equation}
c_L(d)\,(1-t)^{2\alpha+3}
\;\le\;
\Pb(Z>t)\cdot \Eb\big[\phi(Z-t)\big]
\;\le\;
c_U(d)\,(1-t)^{2\alpha+3}.
\label{eq:rho-bds}
\end{equation}
Equivalently, since $\alpha=\frac{d-3}{2}$,
\begin{equation}
\Pb(Z>t)\cdot \Eb\big[\phi(Z-t)\big]\asymp (1-t)^{d}\qquad \text{as } t\uparrow 1,
\end{equation}
with explicit
\begin{equation}
c_L(d)= \frac{c_d^2}{\big(\tfrac{d-1}{2}\big)^2\big(\tfrac{d+1}{2}\big)},
\qquad
c_U(d)= \frac{c_d^2\,2^{\,d-3}}{\big(\tfrac{d-1}{2}\big)^2\big(\tfrac{d+1}{2}\big)},
\qquad
c_d=\frac{\Gamma\!\left(\frac d2\right)}{\sqrt{\pi}\,\Gamma\!\left(\frac{d-1}{2}\right)}.
\end{equation}
\end{proposition}
\begin{proof}
	The inequality can be deduced by combining \eqref{eq:tail-bds} and \eqref{eq:relu-bds}.
\end{proof}

\begin{proposition}[Boundary tail for projected radius]\label{prop:boundary_tail_for_projected}
Fix integers $d\ge 2$ and $1\le m<d$. Let $\vec{X}\sim \mathrm{Uniform}(\Sph^{d-1})\subset\Rb^d$ and $\vec{Z}=\proj(\vec{X})\in\Rb^m$. Then there exist universal constants 
\[
c_L(d,m),\,c_U(d,m)\in(0,\infty)
\]
depending only on $d,m$ (and independent of the choice of projection and of $t$) such that for all $t\in(0,\tfrac14]$,
\begin{equation}
c_L(d,m)\; t^{\frac{d-m}{2}}
\;\le\;
\Pb\big(\|\vec{Z}\|>1-t\big)
\;\le\;
c_U(d,m)\; t^{\frac{d-m}{2}}.
\label{eq:proj-radius-two-sided}
\end{equation}
In particular,
\begin{equation}
\Pb\big(\|\vec{Z}\|>1-t\big)\asymp t^{\frac{d-m}{2}}
\qquad\text{as } t\downarrow 0,
\end{equation}
with one admissible choice
\begin{equation}
c_L(d,m)=\frac{2^{-\lvert \frac{m}{2}-1\rvert}}{\frac{d-m}{2}\,B\!\left(\frac{m}{2},\frac{d-m}{2}\right)},
\qquad
c_U(d,m)=\frac{2^{\,\lvert \frac{m}{2}-1\rvert+\frac{d-m}{2}}}{\frac{d-m}{2}\,B\!\left(\frac{m}{2},\frac{d-m}{2}\right)},
\label{eq:explicit-consts}
\end{equation}
where $B(a,b)=\Gamma(a)\Gamma(b)/\Gamma(a+b)$ denotes the Beta function.
\end{proposition}

\begin{proof}
Write $R:=\|\vec{Z}\|\in[0,1]$ and $S:=R^2$. A standard Gaussian-ratio representation gives
\begin{equation}
S\sim \mathrm{Beta}\!\left(a,b\right),\qquad a:=\frac{m}{2},\quad b:=\frac{d-m}{2},
\end{equation}
with density
\begin{equation}
f_S(x)=\frac{x^{a-1}(1-x)^{b-1}}{B(a,b)}\,\mathbf{1}_{\{0<x<1\}}.
\end{equation}
For $t\in(0,1)$,
\begin{equation}
\Pb(R>1-t)=\Pb\big(S>(1-t)^2\big)=\frac{1}{B(a,b)}\int_{(1-t)^2}^{1} x^{a-1}(1-x)^{b-1}\,dx.
\label{eq:tail-int}
\end{equation}
Set $y:=1-(1-t)^2=2t-t^2$. Then $y\in(0,2t)$ and for $t\in(0,\tfrac14]$ we have $y\le \tfrac12$ and $(1-y)\ge \tfrac12$.
Since $x\in[1-y,1]$, the elementary bound
\begin{equation}
2^{-\lvert a-1\rvert}\ \le\ x^{a-1}\ \le\ 2^{\lvert a-1\rvert}
\qquad\text{for all }x\in\big[\tfrac12,1\big]
\label{eq:x-a-1-bracket}
\end{equation}
holds deterministically (because $\ln x\in[-\ln 2,0]$ and $x^{a-1}=e^{(a-1)\ln x}$).
Using \eqref{eq:x-a-1-bracket} in \eqref{eq:tail-int} and integrating the $(1-x)^{b-1}$ part exactly gives
\begin{equation}
\frac{2^{-\lvert a-1\rvert}}{B(a,b)}\int_{1-y}^{1} (1-x)^{b-1}\,dx
\;\le\;
\Pb(R>1-t)
\;\le\;
\frac{2^{\lvert a-1\rvert}}{B(a,b)}\int_{1-y}^{1} (1-x)^{b-1}\,dx.
\end{equation}
Since $\int_{1-y}^{1}(1-x)^{b-1}\,dx=\frac{y^{b}}{b}$, we obtain
\begin{equation}
\frac{2^{-\lvert a-1\rvert}}{b\,B(a,b)}\,y^{b}
\;\le\;
\Pb(R>1-t)
\;\le\;
\frac{2^{\lvert a-1\rvert}}{b\,B(a,b)}\,y^{b}.
\label{eq:bounds-in-y}
\end{equation}
Finally, because $t\le y\le 2t$, we have $t^{b}\le y^{b}\le (2t)^{b}$, and \eqref{eq:bounds-in-y} yields
\begin{equation}
\frac{2^{-\lvert a-1\rvert}}{b\,B(a,b)}\,t^{b}
\;\le\;
\Pb(R>1-t)
\;\le\;
\frac{2^{\lvert a-1\rvert+b}}{b\,B(a,b)}\,t^{b},
\end{equation}
which is exactly \eqref{eq:proj-radius-two-sided} with the explicit constants in \eqref{eq:explicit-consts}. This proves $\Pb(\|\vec{Z}\|>1-t)\asymp t^{b}=t^{(d-m)/2}$ as $t\downarrow 0$.
\end{proof}

\section{Proof of Theorem \ref{thm:upper_bound}}\label{app:proof_upper_bound}

\begin{theorem}[Detailed version of Theorem~\ref{thm:upper_bound}]
Suppose $\CP$ is a joint distribution of $(\vec{x},y)$. Assume that the marginal distribution of $\vec{x}$ is $\mathrm{Uniform}(\Sph^{d-1})$ and the marginal distribution of $y$ is supported on $[-D,D]$ for some $D>0$. Fix a dataset $\CD=\{(\vec{x}_i,y_i)\}_{i=1}^n$, where each data point is drawn i.i.d. from $\CP$. Let $M\geq D$. If $\vec{\theta}\in\BEoS^{\CS}(\eta,\CD)$ and $\|f_{\vec{\theta}}\|_{\infty}\leq M$,
then, with probability $\geq 1 - 2\delta$, we have that for the plug-in risk estimator $\widehat{\Risk}_{\CD}(f):=\frac{1}{n}\sum_{i=1}^n\left(f(\vec{x}_i)-y_i\right)^2$,
\begin{equation}\label{ineq:cnn_epsilon_trade_off_uniform_app}
\begin{aligned}
&\Bigg|\mathop{\mathbb{E}}_{(\vec{x},y)\sim \mathcal{P}}\left[\left(f_\vec{\theta}(\vec{x})-y\right)^2\right]-\widehat{\Risk}(f_\vec{\theta})
        \Bigg|
\\
&\lessapprox_d J M^2\,\varepsilon^{\frac{d-m}{2}}+
\Big(A\,\varepsilon^{-d}\Big)^{\frac{d}{d+3}}\,
M^{\frac{d+6}{d+3}}\,
n^{-\frac{1}{2}},
\end{aligned}
\end{equation}
for any $\varepsilon \in (0,1)$ such that $\varepsilon^d\gtrsim d^2\sqrt{\frac{m+\log(J/\delta)}{n}}$\footnote{We only need $\poly(d)$ samples to make the feasible choice of $\varepsilon$ non-vacuous. Here we hide the universal constant.}. Here
$A=\frac{1}{\eta}-\frac{1}{2}+4M$, and $\lessapprox_d$ hides constants depending only on $d$
and logarithmic factors in $n$ and $(J/\delta)$.
\end{theorem}

\begin{proof}
For $\varepsilon\in(0,1)$, we define
\begin{align}
\mathsf{I}^{\mathrm{all}}_\varepsilon&:=\Big\{\vec{x}\in\Sph^{d-1}:\ \max_{S_j\in \mathcal{S}}\|\vec{x}^{(S_j)}\|\le 1-\varepsilon\Big\},\\
\mathsf{O}^{\mathrm{any}}_\varepsilon&:=\Big\{\vec{x}\in\Sph^{d-1}:\ \exists\,S_j\in \mathcal{S},\ \|\vec{x}^{(S_j)}\|>1-\varepsilon\Big\}.
\end{align}
Then $\Sph^{d-1}=\mathsf{I}^{\mathrm{all}}_\varepsilon\cup \mathsf{O}^{\mathrm{any}}_\varepsilon$ (disjoint). According to the law of total expectation, the population risk is decomposed into
    
  	\begin{equation}
\mathop{\mathop{\mathbb{E}}}_{(\vec{x},y)\sim \mathcal{P}}\left[\left(f(\vec{x})-y\right)^2\right]=\Pb(\vec{x}\in \mathsf{O}^{\mathrm{any}}_\varepsilon)\cdot\mathbb{E}_{\sfO}\left[\left(f(\vec{x})-y\right)^2\right]+\Pb(\vec{x}\in \Iall_\varepsilon)\cdot\mathbb{E}_{\sfI}\left[\left(f(\vec{x})-y\right)^2\right],
  	\end{equation}
    where $\mathbb{E}_{\sfO}$ means that $\{\vec{x},y\}$ is a new sample from the data distribution conditioned on $\vec{x}\in \Oany_{\varepsilon}$ and  $\mathbb{E}_{\sfI}$ means that $(\vec{x},y)$ is a new sample from the data distribution conditioned on $\vec{x}\in \Iall_\varepsilon$.
    
 Similarly, we also have this decomposition for empirical risk
 
 \begin{equation}\begin{aligned}
 	\frac{1}{n}\sum_{i=1}^n(f(\vec{x}_i)-y_i)^2&=\frac{1}{n}\left(\sum_{i\in I}(f(\vec{x}_i)-y_i)^2+\sum_{i\in O}(f(\vec{x}_i)-y_i)^2\right)\\
 	&=\frac{n_I}{n}\frac{1}{n_I}\sum_{i\in I}(f(\vec{x}_i)-y_i)^2+\frac{n_O}{n}\frac{1}{n_O}\sum_{i\in O}(f(\vec{x}_i)-y_i)^2, \end{aligned}
 \end{equation} 	
  	where $I$ is the set of data points with $\vec{x}_i\in \Iall$ and $O$ is the set of data points with $\vec{x}_i\in \Oany_{\varepsilon}$. Then the generalization gap can be decomposed into
  	\begin{align}
  	|R(f)-\hat{R}_\CD(f)|	&\leq \Pb(\vec{x}\in \Oany_{\varepsilon})\cdot\mathbb{E}_{\sfO}\left[\left(f_\vec{\theta}(\vec{x})-y\right)^2\right]+\frac{n_O}{n}\frac{1}{n_O}\sum_{i\in O}(f(\vec{x}_i)-y_i)^2 \label{eq:boundrary_iart_generalization} \\
    &+\left|\Pb(\vec{x}\in \Iall_{\varepsilon})-\frac{n_I}{n}\right|\frac{1}{n_I}\sum_{i\in I}(f(\vec{x}_i)-y_i)^2\label{eq: high_probability_MSE}\\
&+\Pb(\vec{x}\in\Iall_{\varepsilon})\cdot\left|\mathbb{E}_{\sfI}\left[\left(f(\vec{x})-y\right)^2\right]-\frac{1}{n_I}\sum_{i\in I}(f(\vec{x}_i)-y_i)^2\right|.\label{eq:interior_iart_generalization}
  	\end{align}
 Using the property that the marginal distribution of $\vec{x}$ is $\mathrm{Uniform}(\Sph^{d-1})$ and its concentration property (Proposition \ref{prop:boundary_tail_for_projected} + union bound), 
 \begin{equation}
\Pb\big(\mathsf{O}^{\mathrm{any}}_\varepsilon\big)
=\Pb\Big(\bigcup_{S_j\in \mathcal{S}}\{\|\vec{x}^{(S_j)}\|>1-\varepsilon\}\Big)
\lesssim_d J\cdot \varepsilon^{\frac{d-m}{2}}.
\label{eq:any-annulus}
\end{equation}
 The Hoeffding inequality guarantees that, with probability at least $1-\delta/3$, 
 \begin{equation}
 	\frac{n_O}{n}\leq J\cdot \varepsilon^{\frac{d-m}{2}}+\sqrt{\frac{\log(6/\delta)}{n}}
 \end{equation}
 Therefore, we may conclude that
 \begin{equation}\label{ineq:upper_bound_boundrary_error_cnn}
\eqref{eq:boundrary_iart_generalization}\lessapprox_d JM^2\varepsilon^{\frac{d-m}{2}},
 \end{equation}
 where $\lessapprox_d $ hides the constants that could depend on $d$ and logarithmic factors of $1/\delta$.

 For the term \eqref{eq: high_probability_MSE}, with probability $1-\delta/3$
\begin{equation}
    \begin{cases}
       \left| \Pb(\vec{x}\in \Iall_{\varepsilon})-\frac{n_I}{n}\right|& \lesssim \sqrt{\frac{\varepsilon^{(d-m)/{2}}\log(6J/\delta)}{n}},\quad \\
        \frac{1}{n_I}\sum_{i\in I}(f(\vec{x}_i)-y_i)^2&\lesssim M^2
    \end{cases}
\end{equation}
so we may also conclude that
\begin{equation}\label{ineq: high_probability_MSE_upper_bound_cnn}
    \eqref{eq: high_probability_MSE}\lesssim M^2\sqrt{\frac{\varepsilon^{(d-m)/{2}}\log(6J/\delta)}{n}}\leq M^2\sqrt{\frac{\log(6J/\delta)}{n}}
\end{equation}
 
 For the part of the interior \eqref{eq:interior_iart_generalization}, the scalar $\Pb(\vec{x}\in \Ib_{\varepsilon}^d)$ is less than 1 with high-probability. Therefore, we just need to deal with the term  
\begin{equation}\label{eq: Interior_Generalization_Gap}
	\mathbb{E}_{\Ib}\left[\left(f(\vec{x})-y\right)^2\right]-\frac{1}{n_I}\sum_{i\in I}(f(\vec{x}_i)-y_i)^2.
\end{equation}

Define the projected core index set $\mathcal{C}_\varepsilon:=\{(\vec{u},t)\in\Sph^{m-1}\times[-1,1]: |t|\le 1-\varepsilon\}$. 
For any $S_j\in\mathcal{S}$, $\vec{X}\sim\mathrm{Uniform}(\Sph^{d-1})$ and $\vec{X}^{(S_j)}=\pi_j(\vec{X})$, 
\begin{itemize}
	\item \eqref{eq:tail-bds} shows that $\Pb(\vec{u}^{\T}\vec{X}^{(S_j)}>t)\asymp (1-t)^{\frac{d-1}{2}}$, and
	\item \eqref{eq:relu-bds} shows that $\Eb[\phi(\vec{u}^{\T}\vec{X}^{(S_j)}-t)] \asymp (1-t)^{\frac{d+1}{2}}$.
\end{itemize}
Therefore we have
\begin{equation}
\left(\frac{1}{J}\sum_{j=1}^J\Pb(\vec{u}^{\T}\vec{X}^{(S_j)}>t)\right)\,\left(\frac{1}{J}\sum_{j=1}^J\Eb[\phi(\vec{u}^{\T}\vec{X}^{(S_j)}-t)]\right)\ \asymp\ (1-t)^d,\qquad t\uparrow 1.
\end{equation}
According to \eqref{eq:refined_g_CNN} and the definition $g=\min\{\tilde g(\vec u,t),\tilde g(-\vec u,-t)\}$,
we have the pointwise lower bound
\[
g_{\CP,\mathcal{S}}(\vec{u},t)
\;\ge\;
\min\Big\{\bar r_{\CP,\mathcal{S}}(\vec{u},t)\,\bar p_{\CP,\mathcal{S}}(\vec{u},t),\;
\bar r_{\CP,\mathcal{S}}(-\vec{u},-t)\,\bar p_{\CP,\mathcal{S}}(-\vec{u},-t)\Big\},
\]
since $\sqrt{\bar p^2+\|\bar{\vec a}\|_2^2}\ge \bar p$.

Under $\vec{X}\sim \mathrm{Uniform}(\Sph^{d-1})$, by rotational invariance, for each fixed $j$ and unit $\vec u$ the quantities
$\Pb(\vec{u}^\T\vec{X}^{(S_j)}>t)$ and $\Eb[\phi(\vec{u}^\T\vec{X}^{(S_j)}-t)]$ depend on $t$ only through the scalar marginal $Z$ in \eqref{eq:fd}.
Moreover both are non-increasing in $t$. Hence the product
\[
h(t):=\Pb(Z>t)\cdot \Eb[\phi(Z-t)]
\]
is non-increasing for $t\in[0,1)$, and for any $t\in[-1,1]$ we have
\[
g_{\CP,\mathcal{S}}(\vec{u},t)\ \gtrsim_d\ h(|t|).
\]
Therefore, for the core $\mathcal{C}_\varepsilon=\{(\vec u,t):|t|\le 1-\varepsilon\}$,
\[
g_{\CP,\mathcal{S},\min}(\varepsilon)
:=\inf_{(\vec{u},t)\in\mathcal{C}_\varepsilon}g_{\CP,\mathcal{S}}(\vec{u},t)
\ \gtrsim_d\ h(1-\varepsilon).
\]
By Proposition~\ref{prop:CNN_weight_function_population}, $h(1-\varepsilon)\asymp \varepsilon^d$, hence there exists $c_g(d)>0$ such that
\begin{equation}
g_{\CP,\mathcal{S},\min}(\varepsilon)\ \ge\ c_g(d)\,\varepsilon^{d}.
\label{eq:rho-min}
\end{equation}
On the simultaneous core $\mathsf{I}^{\mathrm{all}}_\varepsilon$, for \emph{every} local receptive field $S_j$ and any unit $\vec{u}\in\Rb^m$,
\begin{equation}
\|\vec{x}^{(S_j)}\|\le 1-\varepsilon\ \Rightarrow\
\begin{cases}
t\ge 1-\varepsilon\ \Rightarrow\ \phi(\vec{u}^{\T}\vec{x}^{(S_j)}-t)=0,\\
t\le -1+\varepsilon\ \Rightarrow\ \phi(\vec{u}^{\T}\vec{x}^{(S_j)}-t)=\vec{u}^{\T}\vec{x}^{(S_j)}-t\ \text{(affine)}.
\end{cases}
\label{eq:affine-on-core}
\end{equation}
Therefore all \emph{large-offset} units ($|t|\ge 1-\varepsilon$) are affine on $\mathsf{I}^{\mathrm{all}}_\varepsilon$ \emph{simultaneously across all views} and can be absorbed into a global affine term. 
The remaining \emph{core-offset} units with $|t|\le 1-\varepsilon$ are controlled by the $g_{\mathcal{S}}$-weighted variation. Using $|f_\theta|_{\Variation_g}\le A$ and \eqref{eq:rho-min},
\begin{equation}
\Big|f_\theta\Big|_{\Variation(\Iall_\varepsilon)} \le\ \frac{|f_\theta|_{\Variation_g}}{g_{\CP,\mathcal{S},\min}(\varepsilon)}
\ \lesssim\,A\,\varepsilon^{-d}.
\label{eq:absorb}
\end{equation}

Therefore, we may leverage the generalization bounds for the unweighted path-norm constraint (see Lemma \ref{lem:generalization‐RBV}) to deduce that with probability at least $1-\delta/3$,
\begin{equation}\label{ineq:upper_interior_generalization_cnn}
\eqref{eq:interior_iart_generalization}\lessapprox_d (A\,\varepsilon^{-d})^{\frac{d}{d+3}}
\,M^{\frac{d+6}{d+3}}
\,n^{-\frac{1}{2}},
\end{equation}

Combining \eqref{ineq:upper_bound_boundrary_error_cnn}, 
\eqref{ineq: high_probability_MSE_upper_bound_cnn} 
and \eqref{ineq:upper_interior_generalization_cnn}, 
we obtain
\begin{equation}\label{ineq:cnn_epsilon_trade_off}
	\sup_{{\vec{\theta}}\in \BEoS(\eta,\CD)}
	\Gen\Gap(f_{\vec{\theta}};\mathcal{D})
	\;\lessapprox_d\;
	JM^2\varepsilon^{\frac{d-m}{2}}
	+(A\varepsilon^{-d})^{\frac{d}{d+3}}
	\,M^{\frac{d+6}{d+3}}
	\,n^{-\frac{1}{2}}.
\end{equation}

According to standard empirical process theory (see Theorem~\ref{thm:g-deviation}), for any $\delta\in(0,1)$, with probability at least $1-\delta$,
\begin{equation}\label{eq:g_dev_appendix}
   \sup_{\substack{\vec{u}\in\Sph^{m-1}\\ t\in[-1,1]}}
   \bigl|g_{\CD,\CS}(\vec{u},t)- g_{\CP,\CS}(\vec{u},t)\bigr|
   \leq C_{\mathrm{ep}}
   \sqrt{\frac{m + \log(2J/\delta)}{n}}
   \eqqcolon \zeta_n .
\end{equation}
Consequently, for any $\varepsilon\in(0,1)$,
\begin{equation}\label{eq:gmin_vs_pop_appendix}
    g_{\CD,\CS,\min}(\varepsilon)
    \;\coloneqq\;
    \inf_{\substack{\vec{u}\in \Sph^{m-1}\\ |t|\le 1-\varepsilon}}
    g_{\CD,\CS}(\vec{u},t)
    \;\geq\;
    g_{\CP,\CS,\min}(\varepsilon)-\zeta_n,
\end{equation}
where $g_{\CP,\CS,\min}(\varepsilon)$ is the corresponding population quantity.

For $\CP_{\vec{X}}=\mathrm{Uniform}(\Sph^{d-1})$, Proposition~\ref{prop:CNN_weight_function_population}
(see \eqref{eq:proj-radius-two-sided}) gives the explicit lower bound
\begin{equation}\label{eq:gpop_min_lb_appendix}
g_{\CP,\CS,\min}(\varepsilon)\;\ge\;c_L(d)\,\varepsilon^d.
\end{equation}
Therefore, a sufficient \textbf{validity condition} ensuring $g_{\CD,\CS,\min}(\varepsilon)\ge 2\zeta_n$ is
\begin{equation}\label{eq:validity_condition_cnn}
   \varepsilon^{d}
   \;\geq\;
   \frac{2C_{\mathrm{ep}}}{c_L(d)}\sqrt{\frac{m+\log(2J/\delta)}{n}}.
\end{equation}

The constant $C_{\mathrm{ep}}>0$ is universal (it does not depend on $d,m,n,J$), and can be taken explicitly
from any standard VC-/pseudo-dimension uniform convergence inequality; see, e.g.,
\citet[Chapter~3]{mohri2018foundations} or \citet{haussler1992decision}.
Moreover, the constant $c_L(d)$ defined in Proposition~\ref{prop:CNN_weight_function_population} satisfies
\begin{equation}\label{eq:cLd_poly_lb}
c_L(d)\;\gtrsim\;\frac{1}{d^2}\qquad (d\ge 3),
\end{equation}
which follows from standard bounds on Gamma-function ratios (e.g.\ Gautschi's inequality;
see \citet[\S5.6(i)]{dlmf}).
Hence the admissible range for $\varepsilon$ is nonempty (e.g.\ $\varepsilon\in[\varepsilon_{\min},1)$)
as soon as $n\gtrsim \poly(d)\,(m+\log(2J/\delta))$.
\end{proof}

\begin{corollary}\label{corollary:optimal_bounds}
Under the same conditions as Theorem~\ref{thm:upper_bound}, assume $d>3$ and $1\le m<\frac{d(d-3)}{d+3}$.
Let
\[
Q \coloneqq 3d^2+3d-md-3m
\qquad\text{and}\qquad
\varepsilon_{\min}\coloneqq
\left(
c_d\,d^2\sqrt{\frac{m+\log(J/\delta)}{n}}
\right)^{\!1/d}.
\]
Define the optimal choice
\begin{equation}\label{eq:eps_star_clean}
\varepsilon^\star
\;\asymp\;
A^{\frac{2d}{Q}}\,J^{-\frac{2(d+3)}{Q}}M^{-\frac{2d}{Q}}n^{-\frac{d+3}{D}},
\end{equation}
and choose the feasible/truncated value
\[
\varepsilon^\dagger \coloneqq \max\{\varepsilon_{\min},\,\varepsilon^\star\}.
\]
Then, with probability at least $1-\delta$,
\begin{equation}\label{ineq:cor_main_clean}
\sup_{\vec{\theta}\in \BEoS^{\CS}(\eta,\CD)}
\Bigg|
\mathbb{E}\big[(f_{\vec{\theta}}(\vec{x})-y)^2\big]
-\widehat{\Risk}_{\CD}(f_{\vec{\theta}})
\Bigg|
\;\lessapprox_d\;
JM^2(\varepsilon^\dagger)^{\frac{d-m}{2}}
+\Big(A(\varepsilon^\dagger)^{-d}\Big)^{\frac{d}{d+3}}
M^{\frac{d+6}{d+3}}n^{-1/2}.
\end{equation}

Moreover, in the regime where $\varepsilon^\star\ge \varepsilon_{\min}$ (e.g., for $n$ sufficiently large, treating $A,J,M$ as constants),
plugging $\varepsilon=\varepsilon^\star$ into \eqref{ineq:cnn_epsilon_trade_off_uniform} yields the optimized rate
\begin{equation}\label{ineq:cor_rate_compact}
\sup_{\vec{\theta}\in \BEoS^{\CS}(\eta,\CD)}
\Bigg|
\mathbb{E}\big[(f_{\vec{\theta}}(\vec{x})-y)^2\big]
-\widehat{\Risk}_{\CD}(f_{\vec{\theta}})
\Bigg|
\;\lessapprox_d\;
A^{\alpha_A}\,J^{\alpha_J}\,M^{\alpha_M}\,n^{-\alpha_n},
\end{equation}
where
\[
\alpha_A=\frac{d(d-m)}{Q},\qquad
\alpha_J=\frac{2d^2}{Q},\qquad
\alpha_M=\frac{4d^2}{Q}
+\frac{(d+6)(d-m)(d+3)}{(d+3)Q},\qquad
\alpha_n=\frac{(d-m)(d+3)}{2Q}.
\]
In particular, if $m$ is fixed and $d\to\infty$, then
\[
\alpha_n\to \frac{1}{6},\qquad \alpha_A\to \frac{1}{3},\qquad \alpha_J\to \frac{2}{3},\qquad \alpha_M\to \frac{5}{3}.
\]
\end{corollary}

\begin{proof}
The optimal choice of $\varepsilon$ minimizing the RHS of 
\eqref{ineq:cnn_epsilon_trade_off} is
\begin{equation}\label{eq:optimal_epsilon_cnn}
	\varepsilon^*
	\asymp
	A^{\frac{2d}{Q}}
	J^{-\frac{2(2d+3)}{Q}}
	M^{-\frac{2d(d+3)}{(2d+3)Q}}
	n^{-\frac{d+3}{Q}},
\end{equation}
where
\[
Q=(d-m)(d+3)+2d^2
=3d^2+3d-md-3m.
\]

Plugging $\varepsilon^*$ into \eqref{ineq:cnn_epsilon_trade_off} yields
\begin{equation}
\begin{aligned}
	\sup_{{\vec{\theta}}\in \BEoS(\eta,\CD)}
	\Gen\Gap(f_{\vec{\theta}};\mathcal{D})
	&\;\lessapprox_d\;
	A^{\frac{2d}{Q}}\,J^{-\frac{2(d+3)}{Q}}M^{-\frac{2d}{Q}}n^{-\frac{d+3}{D}}\\
	&\lessapprox_{A,M,J,d}
	O\Big(
	n^{-\frac{(d-m)(d+3)}
	{2\big((d-m)(d+3)+2d^2\big)}}
	\Big).
\end{aligned}
\end{equation}

In particular, when $m$ is fixed and $d\to\infty$,
\[
\frac{(d-m)(d+3)}
{2\big((d-m)(d+3)+2d^2\big)}
\;\longrightarrow\;
\frac16.
\]
Therefore, this rate does not suffer from the curse of dimensionality 
when $m$ is fixed.

Finally, we verify the validity of the plug-in choice $\varepsilon^*$.
The validity condition \eqref{eq:validity_condition_cnn} requires
$\varepsilon^d \gtrsim_d n^{-1/2}$ (up to logarithmic factors).
Since $(\varepsilon^*)^d \asymp n^{-d(d+3)/Q}$ (treating $A,J,M$ as constants),
plugging \eqref{eq:optimal_epsilon_cnn} into \eqref{eq:validity_condition_cnn}
gives
\[
n^{-\frac{d(d+3)}{Q}}
\gtrsim_d
n^{-1/2}.
\]
This requires
\[
\frac{1}{2}>\frac{d(d+3)}{Q}
=\frac{d(d+3)}
{(d-m)(d+3)+2d^2},
\]
which is equivalent to
\[
(d-m)(d+3)+2d^2>2d(d+3)
\iff d^2-3d>m(d+3)
\iff
m<\frac{d(d-3)}{d+3}.
\]
\end{proof}

\section{Proof of Theorem \ref{thm:flat_lcws_gap}}\label{app:flat_interp_lcws_gap}

Let $\CD=\{(\vec{x}_i,y_i)\}_{i=1}^n$ be a dataset with $\vec{x}_i\in\Rb^d$, and let
$\CS=\{\pi_j:\Rb^d\to\Rb^m\}_{j=1}^J$ be a collection of coordinate projections (patch extractors).
For each $j\in[J]$, abbreviate the extracted patch by $\vec{x}^{(S_j)}\coloneqq \pi_j(\vec{x})$, and for each sample write
$\vec{x}_i^{(S_j)}\coloneqq \pi_j(\vec{x}_i)$.

Assume labels are uniformly bounded: $|y_i|\le D$ for all $i\in[n]$.
Define $I_{\neq 0}\coloneqq\{i\in[n]:\,y_i\neq 0\}$.

Consider width-$K$ two-layer sparsely connected ReLU models with Global Average Pooling (GAP),
\begin{equation}\label{eq:model_CNN_app2}
f_{\vec{\theta}}(\vec{x})
\;=\;
\sum_{k=1}^{K} \frac{v_k}{J} \sum_{j=1}^{J} \phi\!\left(\vec{w}_k^\T \vec{x}^{(S_j)} - b_k\right) + \beta,
\qquad
\phi(t)=\max\{t,0\},
\end{equation}
where $\vec{w}_k\in\Rb^m$, $b_k\in\Rb$ are the shared filter weights and bias, $v_k\in\Rb$ is the output weight, and $\beta\in\Rb$.
We write $\vec{\theta}=\{(\vec{w}_k,b_k,v_k)\}_{k=1}^K\cup\{\beta\}$.

\begin{theorem}[Flat interpolation with width $\le n$]\label{thm:flat_lcws_gap_app}
Assume that $\bigl\|\vec{x}_i^{(S_j)}\bigr\|_2 \le 1$  for all $i\in[n],\, j\in[J]$, and there exists a map $\tau:I_{\neq 0}\to[J]$ that assigns $\vec{p}_i \coloneqq \vec{x}_i^{(S_{\tau(i)})}$ such that $\|\vec{p}_i\|_2=1$ and $\vec{p}_i \neq \vec{x}_\ell^{(S_j)}\ \text{ for all }(\ell,j)\neq(i,\tau(i))$. There exists a width $K\le n$ network of the form \eqref{eq:model_CNN} that interpolates the dataset and whose Hessian operator norm satisfies
\begin{equation}
\lambda_{\max}\!\bigl(\nabla_{\vec{\theta}}^2\loss\bigr)\;\le\;1+\frac{D^2+2/J^2}{n}.
\end{equation}
\end{theorem}

\begin{construction}\label{construct:flat_lcws_gap}
Let $K\coloneqq |I_{\neq 0}|\le n$ and index the hidden units by $k\in I_{\neq 0}$.
For each $k\in I_{\neq 0}$ define the anchor patch $\vec{p}_k\coloneqq \vec{x}_k^{(S_{\tau(k)})}$ and set
\begin{equation}\label{eq:rho_b_choice_lcws_gap}
\rho_k \coloneqq \max_{(\ell,j)\neq(k,\tau(k))}\bigl(\vec{x}_\ell^{(S_j)}\bigr)^\T \vec{p}_k
\;<\; 1,
\qquad
b_k\in(\rho_k,1),
\qquad
\vec{w}_k \coloneqq \vec{p}_k.
\end{equation}
Set the output bias and output weights
\begin{equation}\label{eq:v_beta_choice_lcws_gap}
\beta \coloneqq 0,
\qquad
v_k \coloneqq \frac{J\,y_k}{1-b_k},
\qquad k\in I_{\neq 0}.
\end{equation}

\noindent
(Justification of $\rho_k<1$.) According to the assumption, for any $(\ell,j)\neq(k,\tau(k))$,
\[
\bigl(\vec{x}_\ell^{(S_j)}\bigr)^\T\vec{p}_k \le \bigl\|\vec{x}_\ell^{(S_j)}\bigr\|_2\,\|\vec{p}_k\|_2 \le 1.
\]
If equality held then necessarily $\|\vec{x}_\ell^{(S_j)}\|_2=1$ and $\vec{x}_\ell^{(S_j)}=\vec{p}_k$, contradicting to the assumption.
Hence $\rho_k<1$.
\end{construction}

By \eqref{eq:rho_b_choice_lcws_gap}, for any sample index $i$ and any patch index $j$,
\begin{equation}\label{eq:sign_separation_lcws_gap}
\vec{w}_k^\T \vec{x}_i^{(S_j)} - b_k
=
\begin{cases}
1-b_k \;>\; 0, & (i,j)=(k,\tau(k)),\\[2pt]
\le \rho_k-b_k \;<\; 0, & (i,j)\neq (k,\tau(k)).
\end{cases}
\end{equation}
Thus neuron $k$ is activated on exactly one patch in the entire collection $\{\vec{x}_i^{(S_j)}\}_{i,j}$, namely $\vec{x}_k^{(S_{\tau(k)})}$, and inactivated on all other patches.
Using \eqref{eq:sign_separation_lcws_gap} and \eqref{eq:v_beta_choice_lcws_gap}, at $\vec{x}_k$ we have
\[
f_{\vec{\theta}}(\vec{x}_k)
=
\frac{v_k}{J}\,\phi(1-b_k)
=
\frac{v_k}{J}\,(1-b_k)
=
y_k,
\qquad k\in I_{\neq 0},
\]
and for $i\notin I_{\neq 0}$ all constructed units are inactivated on all patches of $\vec{x}_i$, hence
$f_{\vec{\theta}}(\vec{x}_i)=\beta=0=y_i$.
Therefore, $f_{\vec{\theta}}(\vec{x}_i)=y_i$ for all $i\in[n]$.

For each constructed unit, define
\begin{equation}\label{eq:homog_lcws_gap}
\tilde{v}_k := \operatorname{sign}(v_k)\in\{\pm 1\},
\qquad
\tilde{\vec{w}}_k := |v_k|\,\vec{w}_k,
\qquad
\tilde{b}_k := |v_k|\,b_k.
\end{equation}
Then for any input $\vec{x}$,
\begin{equation}\label{eq:homog_eq_lcws_gap}
\frac{\tilde{v}_k}{J}\sum_{j=1}^J \phi\!\left(\tilde{\vec{w}}_k^\T \vec{x}^{(S_j)}-\tilde{b}_k\right)
=
\frac{v_k}{J}\sum_{j=1}^J \phi\!\left(\vec{w}_k^\T \vec{x}^{(S_j)}-b_k\right),
\end{equation}
so interpolation is preserved.
Moreover, the activation pattern on the full patch collection is unchanged because \eqref{eq:sign_separation_lcws_gap} has strict inequalities and $|v_k|>0$.
At the unique activated patch $(k,\tau(k))$, the (post-rescaling) pre-activation is
\begin{equation}\label{eq:zkk_lcws_gap}
\tilde{z}_k \coloneqq \tilde{\vec{w}}_k^\T \vec{x}_k^{(S_{\tau(k)})}-\tilde{b}_k
= |v_k|\,(1-b_k)
= J|y_k|
\;>\;0,
\qquad
|\tilde{v}_k|=1.
\end{equation}
In what follows we work with the reparameterized network and drop tildes for readability, implicitly assuming $|v_k|=1$ for all $k\in I_{\neq 0}$ and
\begin{equation}\label{eq:drop_tildes_lcws_gap}
z_k \coloneqq \vec{w}_k^\T \vec{x}_k^{(S_{\tau(k)})}-b_k = J|y_k|.
\end{equation}

\begin{proposition}\label{prop:flat_lcws_gap}
Let $\vec{\theta}$ be the model in Construction~\ref{construct:flat_lcws_gap} after the reparameterization \eqref{eq:homog_lcws_gap}. Then
\[
\lambda_{\max}\!\bigl(\nabla_{\vec{\theta}}^2\loss\bigr)
\le
1+\frac{D^2+2/J^2}{n}.
\]
\end{proposition}

\begin{proof}
By direct computation, the Hessian $\nabla^2_{\vec{\theta}}\loss$ is given by
\begin{equation}\label{eq:Hessian_lcws_gap}
\nabla^2_{\vec{\theta}}\loss
=
\frac{1}{n}\sum_{i=1}^n\nabla_{\vec{\theta}}f_{\vec{\theta}}(\vec{x}_i)\nabla_{\vec{\theta}}f_{\vec{\theta}}(\vec{x}_i)^{\T}
+\frac{1}{n}\sum_{i=1}^n\bigl(f_{\vec{\theta}}(\vec{x}_i)-y_i\bigr)\nabla^2_{\vec{\theta}}f_{\vec{\theta}}(\vec{x}_i).
\end{equation}
Since the model interpolates $f_{\vec{\theta}}(\vec{x}_i)=y_i$ for all $i$, we have
\begin{equation}\label{eq:hessian_interpolation_lcws_gap}
\nabla^2_{\vec{\theta}}\loss
=
\frac{1}{n}\sum_{i=1}^n\nabla_{\vec{\theta}}f_{\vec{\theta}}(\vec{x}_i)\nabla_{\vec{\theta}}f_{\vec{\theta}}(\vec{x}_i)^{\T}.
\end{equation}

Denote the tangent features matrix by
\begin{equation}\label{eq:TFM_lcws_gap}
\Phi=\left[ \nabla_{\vec{\theta}}f_{\vec{\theta}}(\vec{x}_1),\nabla_{\vec{\theta}}f_{\vec{\theta}}(\vec{x}_2),\cdots,\nabla_{\vec{\theta}}f_{\vec{\theta}}(\vec{x}_n)\right].
\end{equation}
Then \eqref{eq:hessian_interpolation_lcws_gap} can be written as $\nabla^2_{\vec{\theta}}\loss=\Phi\Phi^{\T}/n$, and thus
\begin{equation}\label{eq:operator_norm_by_TFM_lcws_gap}
\lambda_{\max}(\nabla^2_{\vec{\theta}}\loss)
=
\max_{\vec{\gamma}\in \Sph^{(m+2)K}}\frac{1}{n}\|\Phi^\T\vec{\gamma}\|^2
=
\max_{\vec{u}\in \Sph^{n-1}}\frac{1}{n}\|\Phi\vec{u}\|^2.
\end{equation}

For the gate $m_k^{(S_j)}(\vec{x})\coloneqq\mathds{1}\curly{\vec{w}_k^\T \vec{x}^{(S_j)}>b_k}$, let $m^{(S_j)}_{k,i}\coloneqq m_k^{(S_j)}(\vec{x}_i)$.
From direct computation,
\begin{align}\label{eq:per_unit_grads_lcws_gap}
\frac{\partial f_{\vec{\theta}}(\vec{x}_i)}{\partial v_k}
&=
\frac{1}{J}\sum_{j=1}^J m^{(S_j)}_{k,i}\cdot\big(\vec{w}_k^\T \vec{x}_i^{(S_j)}-b_k\big),
&
\frac{\partial f_{\vec{\theta}}(\vec{x}_i)}{\partial \vec{w}_k}
&=
\frac{1}{J}\sum_{j=1}^J m^{(S_j)}_{k,i}\cdot v_k\cdot\vec{x}_i^{(S_j)},\\
\frac{\partial f_{\vec{\theta}}(\vec{x}_i)}{\partial b_k}
&=
-\frac{1}{J}\sum_{j=1}^J m^{(S_j)}_{k,i}\cdot v_k,
&
\frac{\partial f_{\vec{\theta}}(\vec{x}_i)}{\partial \beta}
&=
1.\notag
\end{align}

By the one-to-one activation property \eqref{eq:sign_separation_lcws_gap}, each sample $\vec{x}_i$ with $i\in I_{\neq 0}$ activates exactly one unit (the unit $k=i$) on exactly one patch $S_{\tau(i)}$, and samples with $i\notin I_{\neq 0}$ activate none.
Hence the sample-wise gradient $\nabla_{\vec{\theta}} f_{\vec{\theta}}(\vec{x}_i)$ has support only on the parameter triplet $(\vec{w}_i,b_i,v_i,\beta)$ when $i\in I_{\neq 0}$, and is zero on $(\vec{w}_k,b_k,v_k)$ for all $k$ when $i\notin I_{\neq 0}$.
Writing the nonzero gradient block explicitly (recall $|v_i|=1$ after reparameterization and \eqref{eq:drop_tildes_lcws_gap}),
\begin{equation}\label{eq:neuron_grad_block_lcws_gap}
\begin{aligned}
\nabla_{(\vec{w}_i,b_i,v_i,\beta)} f_{\vec{\theta}}(\vec{x}_i)
&=
\begin{pmatrix}
\nabla_{(\vec{w}_i,b_i,v_i)} f_{\vec{\theta}}(\vec{x}_i)\\[2pt]
1
\end{pmatrix},
\\[4pt]
\nabla_{(\vec{w}_i,b_i,v_i)} f_{\vec{\theta}}(\vec{x}_i)
&=
\begin{cases}
\begin{pmatrix}
\frac{v_i}{J}\,\vec{x}_i^{(S_{\tau(i)})}\\[4pt]
-\frac{v_i}{J}\\[4pt]
\frac{1}{J}\bigl(\vec{w}_i^\T \vec{x}_i^{(S_{\tau(i)})}-b_i\bigr)
\end{pmatrix}
=
\begin{pmatrix}
\frac{v_i}{J}\,\vec{x}_i^{(S_{\tau(i)})}\\[4pt]
-\frac{v_i}{J}\\[4pt]
|y_i|
\end{pmatrix},
& (i\in I_{\neq 0}),
\\[14pt]
\vec{0},
& (i\notin I_{\neq 0}).
\end{cases}
\end{aligned}
\end{equation}

Let $\vec{u}=(u_1,\dots,u_n)\in \Sph^{n-1}$ and plug \eqref{eq:neuron_grad_block_lcws_gap} into \eqref{eq:operator_norm_by_TFM_lcws_gap}. As in the fully connected case, after a row permutation (grouping neuron parameters) we obtain
\begin{align}
\lambda_{\max}(\nabla^2_{\vec{\theta}}\loss)
&=
\max_{\vec{u}\in \Sph^{n-1}}\frac{1}{n}\|\Phi\vec{u}\|^2\notag\\
&=
\frac{1}{n}\max_{\vec{u}\in \Sph^{n-1}}
\sum_{i=1}^n u_i^2\bigl\|\nabla_{(\vec{w}_i,b_i,v_i)} f_{\vec{\theta}}(\vec{x}_i)\bigr\|_2^2
+\Big(\sum_{i=1}^n u_i\Big)^2 \label{eq:TFM3_lcws_gap}\\
&=
\frac{1}{n}\max_{\vec{u}\in \Sph^{n-1}}
\sum_{i\in I_{\neq 0}} u_i^2\left(\frac{1}{J^2}\bigl\|\vec{x}_i^{(S_{\tau(i)})}\bigr\|_2^2+\frac{1}{J^2}+y_i^2\right)
+\Big(\sum_{i=1}^n u_i\Big)^2.\label{eq:TFM4_lcws_gap}
\end{align}
According to the assumption, $\|\vec{x}_i^{(S_{\tau(i)})}\|_2=1$ for all $i\in I_{\neq 0}$, and by bounded labels $y_i^2\le D^2$ for all $i$.
Thus
\begin{align}
\lambda_{\max}(\nabla^2_{\vec{\theta}}\loss)
&\le
\frac{1}{n}\left(\max_{i\in[n]}\left(\frac{1}{J^2}\bigl\|\vec{x}_i^{(S_{\tau(i)})}\bigr\|_2^2+\frac{1}{J^2}+y_i^2\right)
+\max_{\vec{u}\in \Sph^{n-1}}\Big(\sum_{i=1}^n u_i\Big)^2\right)\notag\\
&\le
\frac{1}{n}\left(\frac{2}{J^2}+D^2+n\right)
=
1+\frac{D^2+2/J^2}{n}.\notag
\end{align}
If we remove the output bias term $\beta$ from the parameters, then the last term $(\sum_i u_i)^2$ in \eqref{eq:TFM3_lcws_gap} is removed.
\end{proof}

\section{Extending the Framework beyond Weight Sharing and Global Average Pooling}
\label{app:sparse_no_share}

The main text analyzes the minimal convolutional model~\eqref{eq:model_CNN} with weight sharing and global average pooling (GAP).
These two design choices simplify the exposition but are not prerequisites for the stability-induced regularization mechanism.
To demonstrate that \emph{sparse connectivity alone} is the essential ingredient, this appendix analyzes a variant that retains the sparse receptive-field pattern of the main-text SCN but removes both weight sharing and global pooling.

\subsection{Sparsely connected network without weight sharing}

\noindent\textbf{Setup and notation.}
Fix the collection of local receptive fields $\CS=\{S_j\}_{j=1}^{J}$ (each $S_j\subset[d]$ with $|S_j|=m$) and the corresponding coordinate projections $\pi_j:\Rb^{d}\to\Rb^{m}$.

\begin{definition}[SCN without weight sharing]\label{def:scn_unshared}
A \emph{sparsely connected network without weight sharing}, with receptive fields~$\CS$ and width~$K$, is a function of the form
\begin{equation}\label{eq:scn_unshared}
    f_{\vec{\theta}}(\vec{x})
    \;=\; \sum_{k=1}^{K}\sum_{j=1}^{J}
          v_{k,j}\,\phi\!\bigl(\vec{w}_{k,j}^{\T}\pi_j(\vec{x})-b_{k,j}\bigr)
          \;+\; \beta,
\end{equation}
where $\vec{w}_{k,j}\in\Rb^{m}$, $b_{k,j}\in\Rb$, $v_{k,j}\in\Rb$ are \emph{independent} parameters for each neuron index $k\in[K]$ and location index $j\in[J]$, and $\beta\in\Rb$ is a scalar bias.
We write $\vec{\theta}=\{(\vec{w}_{k,j},b_{k,j},v_{k,j})\}_{k\in[K],j\in[J]}\cup\{\beta\}$ for the full parameter vector.
\end{definition}

The key structural property shared with the main-text SCN~\eqref{eq:model_CNN} is \emph{sparse connectivity}: neuron $(k,j)$ receives input only from the $m$-dimensional patch $\pi_j(\vec{x})$.
The architecture~\eqref{eq:scn_unshared} differs from~\eqref{eq:model_CNN} in two respects: (i)~filter weights $\vec{w}_{k,j}$ and biases $b_{k,j}$ are \emph{not} shared across locations~$j$, and (ii)~output weights $v_{k,j}$ may vary freely with~$j$, so there is no global average pooling.

\noindent\textbf{Data and loss.}
As in the main text, the dataset is $\CD=\{(\vec{x}_i,y_i)\}_{i=1}^{n}$ with $\vec{x}_i\in B_R^{d}$ and $y_i\in[-D,D]$, and the training objective is $\loss(\vec{\theta}):=\frac{1}{2n}\sum_{i=1}^{n}(f_{\vec{\theta}}(\vec{x}_i)-y_i)^{2}$.

\noindent\textbf{Location-specific weight function.}
Because neurons at different locations $j$ now operate with independent parameters on different patch distributions, the relevant weight function is defined \emph{per location}.
For each $j\in[J]$, let $\vec{X}^{(j)}_{\CD}$ be a random vector drawn uniformly from the training patches at position~$j$, i.e.\ from $\{\pi_j(\vec{x}_i)\}_{i=1}^{n}\subset\Rb^{m}$.
Define
\begin{equation}\label{eq:g_local}
    g_{\CD,\CS}^{(j)}(\vec{u},t)
    \;:=\; \Eb\!\Bigl[\phi\bigl(\vec{u}^{\T}\vec{X}^{(j)}_{\CD}-t\bigr)\Bigr]\;
    \sqrt{\Pb\bigl(\vec{u}^{\T}\vec{X}^{(j)}_{\CD}>t\bigr)^{2}
          +\Bigl\|\Eb\!\Bigl[\vec{X}^{(j)}_{\CD}\,
            \mathds{1}\Big\{\vec{u}^{\T}\vec{X}^{(j)}_{\CD}>t\Big\}\Bigr]\Bigr\|_{2}^{2}}\,.
\end{equation}
This is the direct analogue of the global weight function $g_{\CD,\CS}$ in Definition~\ref{def:weight_function_CNN_appendix}, specialized to a single location.

\subsection{From the BEoS condition to a location-wise weighted path norm}

We establish the counterpart of Theorem~\ref{thm:implicit_bias_CNN} for the unshared architecture~\eqref{eq:scn_unshared}.
The proof follows the same strategy---lower-bounding $\lambda_{\max}(\mT_{\CD})$ via the all-ones direction---but now each neuron~$(k,j)$ contributes independently, which allows the bound to decompose over locations.

\begin{lemma}[Hessian residual bound for the unshared SCN]\label{lem:hessian_unshared}
For the architecture~\eqref{eq:scn_unshared} and any $\vec{x}$ with $\norm{\vec{x}}_2\le R$,
\[
    \bigl\|\nabla_{\vec{\theta}}^{2}f_{\vec{\theta}}(\vec{x})\bigr\|_{\mathrm{op}}
    \;\le\; 2(R+1).
\]
\end{lemma}

\begin{proof}
The proof of Lemma~\ref{lemma:cnn_upper_bound_operator_norm} bounds $|\vec{\omega}^{\T}\nabla_{\vec{\theta}}^{2}f(\vec{x})\,\vec{\omega}|$ for a unit perturbation~$\vec{\omega}$ by examining each neuron independently: the only nonzero second derivatives within neuron $(k,j)$ are the mixed partials between $(\vec{w}_{k,j},b_{k,j})$ and $v_{k,j}$, exactly as in~\eqref{eq:cnn_hessian_non_zero_term_1}--\eqref{eq:cnn_hessian_non_zero_term_2}.
Since different neurons occupy orthogonal parameter blocks, their contributions to $\vec{\omega}^{\T}\nabla_{\vec{\theta}}^{2}f(\vec{x})\,\vec{\omega}$ sum without interaction.
Each block contributes at most $2|\gamma_{k,j}|(R\norm{\vec{\alpha}_{k,j}}_2+|\delta_{k,j}|)$ (cf.~\eqref{ineq:neuron_wise_hessian_inequality}), where $(\vec{\alpha}_{k,j},\delta_{k,j},\gamma_{k,j})$ is the perturbation restricted to neuron $(k,j)$.
Summing over $(k,j)$ and applying Cauchy--Schwarz with the normalization $\norm{\vec{\omega}}_2=1$ yields the same bound $2(R+1)$.
\end{proof}

\begin{proposition}[Tangent-feature lower bound for the unshared SCN]\label{prop:TFM_unshared}
Let\/ $\mT_{\CD}=\frac{1}{n}\sum_{i=1}^{n}\nabla_{\vec{\theta}}f_{\vec{\theta}}(\vec{x}_i)\,\nabla_{\vec{\theta}}f_{\vec{\theta}}(\vec{x}_i)^{\T}$.
Then
\begin{equation}\label{eq:TFM_unshared}
    \lambda_{\max}(\mT_{\CD})
    \;\ge\; 1 \;+\; 2\sum_{k=1}^{K}\sum_{j=1}^{J}
    |v_{k,j}|\,\norm{\vec{w}_{k,j}}\;
    g_{\CD,\CS}^{(j)}\!\!\left(\frac{\vec{w}_{k,j}}{\norm{\vec{w}_{k,j}}},\,\frac{b_{k,j}}{\norm{\vec{w}_{k,j}}}\right).
\end{equation}
\end{proposition}

\begin{proof}
Write the gate indicators $m_{k,j,i}:=\mathds{1}\curly{\vec{w}_{k,j}^{\T}\pi_j(\vec{x}_i)>b_{k,j}}$, the normalized direction $\vec{u}_{k,j}:=\vec{w}_{k,j}/\norm{\vec{w}_{k,j}}_2$, and the normalized threshold $t_{k,j}:=b_{k,j}/\norm{\vec{w}_{k,j}}_2$.
The partial derivatives of $f_{\vec{\theta}}$ at sample $\vec{x}_i$ are
\begin{align}
    \frac{\partial f}{\partial \vec{w}_{k,j}}(\vec{x}_i)
      &= v_{k,j}\,m_{k,j,i}\,\pi_j(\vec{x}_i),
    &
    \frac{\partial f}{\partial b_{k,j}}(\vec{x}_i)
      &= -v_{k,j}\,m_{k,j,i},
    \label{eq:unshared_grad_wb} \\[4pt]
    \frac{\partial f}{\partial v_{k,j}}(\vec{x}_i)
      &= \phi\bigl(\vec{w}_{k,j}^{\T}\pi_j(\vec{x}_i)-b_{k,j}\bigr),
    &
    \frac{\partial f}{\partial \beta}(\vec{x}_i)
      &= 1.
    \label{eq:unshared_grad_vbeta}
\end{align}
Since each neuron $(k,j)$ operates on the patch $\pi_j(\vec{x}_i)$ with \emph{independent} parameters $(\vec{w}_{k,j},b_{k,j},v_{k,j})$, the gradient blocks are orthogonal across different $(k,j)$ pairs.

Denoting the tangent-feature matrix $\Phi=[\nabla_{\vec{\theta}}f(\vec{x}_1),\ldots,\nabla_{\vec{\theta}}f(\vec{x}_n)]\in\Rb^{p\times n}$, we use the standard lower bound (cf.~\eqref{eq:top-eig})
\begin{equation}\label{eq:unshared_all_ones}
    \lambda_{\max}(\mT_{\CD})
    = \frac{1}{n}\max_{\norm{\vec{u}}=1}\norm{\Phi\vec{u}}^{2}
    \;\ge\; \frac{1}{n^{2}}\norm{\Phi\,\vec{1}}^{2}
    = \frac{1}{n^{2}}\Bigl\|\sum_{i=1}^{n}\nabla_{\vec{\theta}}f(\vec{x}_i)\Bigr\|^{2}.
\end{equation}
Expanding the squared norm and using the orthogonality of neuron blocks gives
\begin{equation}\label{eq:unshared_expand}
    \frac{1}{n^{2}}\norm{\Phi\,\vec{1}}^{2}
    = 1
    + \sum_{k=1}^{K}\sum_{j=1}^{J}\Biggl[
        \underbrace{v_{k,j}^{2}\biggl(
          \Bigl\|\frac{1}{n}\sum_{i}m_{k,j,i}\,\pi_j(\vec{x}_i)\Bigr\|^{2}
          +\Bigl(\frac{1}{n}\sum_{i}m_{k,j,i}\Bigr)^{\!2}
        \biggr)}_{\displaystyle=:\;P_{k,j}}
        +\underbrace{\norm{\vec{w}_{k,j}}_2^{2}\Bigl(\frac{1}{n}\sum_{i}\phi\bigl(\vec{u}_{k,j}^{\T}\pi_j(\vec{x}_i)-t_{k,j}\bigr)\Bigr)^{\!2}}_{\displaystyle=:\;Q_{k,j}}
    \Biggr].
\end{equation}
Here the ``$1$'' comes from $\bigl(\frac{1}{n}\sum_i \partial f/\partial\beta\bigr)^{2}=1$.
The term $P_{k,j}$ collects the squared sums of the $\vec{w}_{k,j}$- and $b_{k,j}$-gradients (both proportional to $v_{k,j}$), while $Q_{k,j}$ comes from the $v_{k,j}$-gradient (proportional to $\norm{\vec{w}_{k,j}}_2$).
Crucially, the double sum decomposes over $(k,j)$ because the parameters are not shared across locations.

We apply $a^{2}+b^{2}\ge 2ab$ to each $(k,j)$ independently (cf.\ the analogous step in the proof of Proposition~\ref{prop:CNN_stability_induced_regularity}):
\begin{equation}\label{eq:unshared_amgm}
    P_{k,j} + Q_{k,j}
    \;\ge\; 2\,|v_{k,j}|\,\norm{\vec{w}_{k,j}}_2\,
    \underbrace{\Bigl(\frac{1}{n}\sum_{i}\phi\bigl(\vec{u}_{k,j}^{\T}\pi_j(\vec{x}_i)-t_{k,j}\bigr)\Bigr)
    \sqrt{
      \Bigl(\frac{1}{n}\sum_{i}m_{k,j,i}\Bigr)^{\!2}
      +\Bigl\|\frac{1}{n}\sum_{i}m_{k,j,i}\,\pi_j(\vec{x}_i)\Bigr\|^{2}
    }}_{\displaystyle=\;g_{\CD,\CS}^{(j)}(\vec{u}_{k,j},\,t_{k,j})}\,.
\end{equation}
To verify the identification with~\eqref{eq:g_local}, note that the empirical averages over $i\in[n]$ are precisely the expectation, probability, and conditional first moment under the uniform draw $\vec{X}_{\CD}^{(j)}\sim\mathrm{Uniform}(\{\pi_j(\vec{x}_i)\}_{i=1}^{n})$:
\[
    \frac{1}{n}\sum_{i=1}^{n}\phi\bigl(\vec{u}^{\T}\pi_j(\vec{x}_i)-t\bigr)
    = \Eb\bigl[\phi(\vec{u}^{\T}\vec{X}_{\CD}^{(j)}-t)\bigr],
    \qquad
    \frac{1}{n}\sum_{i=1}^{n}m_{k,j,i}
    = \Pb\bigl(\vec{u}^{\T}\vec{X}_{\CD}^{(j)}>t\bigr),
\]
\[
    \frac{1}{n}\sum_{i=1}^{n}m_{k,j,i}\,\pi_j(\vec{x}_i)
    = \Eb\bigl[\vec{X}_{\CD}^{(j)}\,\mathds{1}\curly{\vec{u}^{\T}\vec{X}_{\CD}^{(j)}>t}\bigr].
\]
Summing~\eqref{eq:unshared_amgm} over $(k,j)$ and combining with~\eqref{eq:unshared_all_ones}--\eqref{eq:unshared_expand} yields~\eqref{eq:TFM_unshared}.
\end{proof}

\begin{remark}[Why the per-neuron decomposition succeeds here]\label{rem:gap_vs_unshared}
In the GAP model~\eqref{eq:model_CNN}, the output weight $v_k$ is shared uniformly across all locations via the $1/J$ average.
This forces $\nabla_{\vec{w}_k}f=\frac{v_k}{J}\sum_j m^{(S_j)}_{k,i}\,\vec{x}_i^{(S_j)}$, in which contributions from different locations are combined with a \emph{common} coefficient $v_k/J$.
The common coefficient allows $v_k^2$ to be factored out, and the resulting $\frac{1}{n^2}\norm{\Phi\vec{1}}^2$ factors through the \emph{global} patch distribution $\vec{X}^{\CS}_{\CD}$, producing the global weight function $g_{\CD,\CS}$ of Proposition~\ref{prop:CNN_stability_induced_regularity}.

In the unshared architecture~\eqref{eq:scn_unshared}, each neuron $(k,j)$ has its own output weight $v_{k,j}$, but since $\vec{w}_{k,j}$ is also independent across~$j$, the gradient blocks are orthogonal and the squared norm decomposes cleanly into per-$(k,j)$ terms~\eqref{eq:unshared_expand}.
This per-neuron decomposition makes the AM--GM step~\eqref{eq:unshared_amgm} straightforward.

An intermediate architecture---shared filters $\vec{w}_k$ with location-dependent output weights $v_{k,j}$---would couple locations through $\nabla_{\vec{w}_k}f=\sum_j v_{k,j}\,m_{k,j,i}\,\pi_j(\vec{x}_i)$, in which contributions from different locations can cancel when $v_{k,j}$ have varying signs.
This coupling prevents both the per-location decomposition used above and the common-coefficient factorization available under GAP.
Analyzing such architectures would require different proof techniques.
\end{remark}

\begin{theorem}[BEoS constraint for the unshared SCN]\label{thm:EoS_unshared}
For the architecture~\eqref{eq:scn_unshared}, if\/ $\vec{\theta}\in\BEoS^{\CS}(\eta,\CD)$, i.e., $\lambda_{\max}(\nabla^{2}_{\vec{\theta}}\loss)\le 2/\eta$, then
\begin{equation}\label{eq:EoS_unshared}
    \sum_{k=1}^{K}\sum_{j=1}^{J}
    |v_{k,j}|\,\norm{\vec{w}_{k,j}}\;
    g_{\CD,\CS}^{(j)}\!\!\left(\frac{\vec{w}_{k,j}}{\norm{\vec{w}_{k,j}}},\,\frac{b_{k,j}}{\norm{\vec{w}_{k,j}}}\right)
    \;\le\; \frac{1}{\eta}-\frac{1}{2}
    + (R+1)\sqrt{2\loss(\vec{\theta})}.
\end{equation}
\end{theorem}

\begin{proof}
Decompose the Hessian as $\nabla^{2}_{\vec{\theta}}\loss=\mT_{\CD}+\mR_{\CD}$ (cf.~\eqref{eq:Hessian_1}) with
\[
    \mR_{\CD}=\frac{1}{n}\sum_{i=1}^{n}(f_{\vec{\theta}}(\vec{x}_i)-y_i)\,\nabla^2_{\vec{\theta}}f_{\vec{\theta}}(\vec{x}_i).
\]
Let $\vec{\omega}$ be the unit eigenvector of $\mT_{\CD}$ corresponding to $\lambda_{\max}(\mT_{\CD})$.
Then (cf.~\eqref{equ:qd3})
\[
    \lambda_{\max}(\nabla^{2}_{\vec{\theta}}\loss)
    \;\ge\; \vec{\omega}^{\T}\nabla^{2}_{\vec{\theta}}\loss\,\vec{\omega}
    \;=\; \lambda_{\max}(\mT_{\CD})
          + \underbrace{\frac{1}{n}\sum_{i=1}^{n}(f_{\vec{\theta}}(\vec{x}_i)-y_i)\,
            \vec{\omega}^{\T}\nabla^{2}_{\vec{\theta}}f_{\vec{\theta}}(\vec{x}_i)\,\vec{\omega}}_{\displaystyle=:\;\Delta}.
\]
By the Cauchy--Schwarz inequality and Lemma~\ref{lem:hessian_unshared},
\[
    |\Delta|
    \;\le\; \sqrt{\frac{1}{n}\sum_{i}(f_{\vec{\theta}}(\vec{x}_i)-y_i)^{2}}\;\cdot\;
            \sqrt{\frac{1}{n}\sum_{i}\bigl(\vec{\omega}^{\T}\nabla^{2}_{\vec{\theta}}f_{\vec{\theta}}(\vec{x}_i)\,\vec{\omega}\bigr)^{2}}
    \;\le\; 2(R+1)\sqrt{2\loss(\vec{\theta})}.
\]
Therefore,
\[
    \lambda_{\max}(\mT_{\CD})
    \;\le\; \lambda_{\max}(\nabla^{2}_{\vec{\theta}}\loss) + |\Delta|
    \;\le\; \frac{2}{\eta} + 2(R+1)\sqrt{2\loss(\vec{\theta})}.
\]
Inserting the lower bound from Proposition~\ref{prop:TFM_unshared} and rearranging yields~\eqref{eq:EoS_unshared}.
\end{proof}

\subsection{Generalization analysis}

We now derive the generalization guarantee for the unshared architecture~\eqref{eq:scn_unshared}, following the same interior/exterior decomposition as in Appendix~\ref{app:proof_upper_bound}.

\noindent\textbf{Interior/exterior decomposition.}
For $\varepsilon\in(0,1)$, define
\[
    \mathsf{I}^{\mathrm{all}}_\varepsilon
    := \Big\{\vec{x}\in\Sph^{d-1}:\max_{j\in[J]}\norm{\pi_j(\vec{x})}\le 1-\varepsilon\Big\},
    \qquad
    \mathsf{O}^{\mathrm{any}}_\varepsilon
    := \Big\{\vec{x}\in\Sph^{d-1}:\exists\, j\in[J],\ \norm{\pi_j(\vec{x})}>1-\varepsilon\Big\}.
\]
By a union bound and Proposition~\ref{prop:boundary_tail_for_projected},
$\Pb(\mathsf{O}^{\mathrm{any}}_\varepsilon)\lesssim J\,\varepsilon^{(d-m)/2}$.

\noindent\textbf{Lower bound on $g_{\CD,\CS}^{(j)}$ over the interior.}
On $\mathsf{I}^{\mathrm{all}}_\varepsilon$, every patch satisfies $\norm{\pi_j(\vec{x})}\le 1-\varepsilon$ for all $j$.
When the marginal of $\vec{x}$ is $\mathrm{Uniform}(\Sph^{d-1})$, each projected patch $\pi_j(\vec{x})\in\Rb^{m}$ inherits the same rotational-invariance structure analyzed in Appendix~\ref{app:dd_regularity}.
By the same computation as in Proposition~\ref{prop:CNN_weight_function_population} (see~\eqref{eq:rho-bds}), the population weight function at each location satisfies
\[
    g_{\CP}^{(j)}(\vec{u},t) \;\ge\; c_g(d)\,\varepsilon^{d}
    \qquad\text{for all }\vec{u}\in\Sph^{m-1},\;|t|\le 1-\varepsilon,
\]
with a constant $c_g(d)>0$ depending only on $d$.
The uniform deviation bound of Theorem~\ref{thm:g-deviation}, applied separately at each location~$j$ (the patches $\{\pi_j(\vec{x}_i)\}_{i=1}^{n}$ are i.i.d.\ for fixed~$j$), gives
\[
    \sup_{\vec{u}\in\Sph^{m-1},\,|t|\le 1}
    \bigl|g_{\CD,\CS}^{(j)}(\vec{u},t)-g_{\CP}^{(j)}(\vec{u},t)\bigr|
    \;\le\; C_{\mathrm{ep}}\sqrt{\frac{m+\log(2J/\delta)}{n}}
    \;=:\;\zeta_n,
\]
with probability at least $1-\delta$, simultaneously for all $j\in[J]$ (via a union bound absorbing $J$ into the logarithm).
Hence, for any $\varepsilon$ satisfying $\varepsilon^{d}\ge 2\zeta_n/c_g(d)$, we have
\begin{equation}\label{eq:g_unshared_lower}
    g_{\CD,\CS}^{(j)}(\vec{u},t)
    \;\ge\; \tfrac{1}{2}\,c_g(d)\,\varepsilon^{d}
    \qquad\text{for all }j\in[J],\;\vec{u}\in\Sph^{m-1},\;|t|\le 1-\varepsilon.
\end{equation}

\noindent\textbf{From the BEoS constraint to an unweighted path norm.}
On the interior $\mathsf{I}^{\mathrm{all}}_\varepsilon$, every neuron with $|t_{k,j}|\ge 1-\varepsilon$ is either identically zero or fully affine on all training patches (cf.~\eqref{eq:affine-on-core}), and can be absorbed into the affine component.
For the remaining neurons with $|t_{k,j}|\le 1-\varepsilon$, inserting the lower bound~\eqref{eq:g_unshared_lower} into the BEoS constraint~\eqref{eq:EoS_unshared} and setting $A:=\frac{1}{\eta}-\frac{1}{2}+4M$ yields
\begin{equation}\label{eq:unshared_pathnorm}
    \sum_{\substack{k,j:\\|t_{k,j}|\le 1-\varepsilon}}
    |v_{k,j}|\,\norm{\vec{w}_{k,j}}
    \;\le\; \frac{A}{\frac{1}{2}c_g(d)\,\varepsilon^{d}}
    \;\asymp\; A\,\varepsilon^{-d}.
\end{equation}

\noindent\textbf{Controlling the generalization gap.}
The function class on the interior set $\mathsf{I}^{\mathrm{all}}_\varepsilon$ has variation norm bounded by $A\varepsilon^{-d}$ (equation~\eqref{eq:unshared_pathnorm}).
The generalization gap decomposes as in~\eqref{eq:boundrary_iart_generalization}--\eqref{eq:interior_iart_generalization}:
\begin{itemize}[leftmargin=2em]
    \item \emph{Exterior contribution.}
    Since $\norm{f_{\vec{\theta}}}_{\infty}\le M$ and $|y|\le D\le M$,
    the exterior contributes at most $O(JM^{2}\varepsilon^{(d-m)/2})$
    (cf.~\eqref{ineq:upper_bound_boundrary_error_cnn}).
    \item \emph{Interior contribution.}
    By Lemma~\ref{lem:generalization‐RBV} (metric entropy of variation spaces), the interior generalization gap satisfies
    \[
        \left|\Eb_{\sfI}\bigl[(f_{\vec{\theta}}-y)^{2}\bigr]
        - \frac{1}{n_I}\sum_{i\in I}(f_{\vec{\theta}}(\vec{x}_i)-y_i)^{2}\right|
        \;\lessapprox_d\; (A\varepsilon^{-d})^{\frac{d}{d+3}}\,M^{\frac{d+6}{d+3}}\,n^{-1/2}
    \]
    (cf.~\eqref{ineq:upper_interior_generalization_cnn}).
    \item \emph{Concentration of the interior fraction.}
    By Hoeffding's inequality, $|\Pb(\Iall_\varepsilon)-n_I/n|\lesssim\sqrt{\log(1/\delta)/n}$ with high probability, contributing a lower-order term $O(M^{2}\sqrt{\log(1/\delta)/n})$
    (cf.~\eqref{ineq: high_probability_MSE_upper_bound_cnn}).
\end{itemize}
Combining these three terms gives the following result.

\begin{theorem}[Generalization for the unshared SCN]\label{thm:gen_unshared}
Assume $\vec{x}\sim\mathrm{Uniform}(\Sph^{d-1})$ and $|y|\le D$.
Let $M\ge D$, $\norm{f_{\vec{\theta}}}_{\infty}\le M$, and suppose $\vec{\theta}$ satisfies the BEoS condition $\lambda_{\max}(\nabla^{2}_{\vec{\theta}}\loss)\le 2/\eta$.
Assume $d>3$ and $1\le m<\frac{d(d-3)}{d+3}$.
Then, with probability at least $1-2\delta$,
\begin{equation}\label{eq:gen_unshared}
    \Gen\Gap(f_{\vec{\theta}};\CD)
    \;\lessapprox_d\; J M^{2}\,\varepsilon^{\frac{d-m}{2}}
    + \bigl(A\,\varepsilon^{-d}\bigr)^{\frac{d}{d+3}} M^{\frac{d+6}{d+3}} n^{-1/2},
\end{equation}
for any $\varepsilon\in(0,1)$ satisfying $\varepsilon^{d}\gtrsim d^{2}\sqrt{(m+\log(J/\delta))/n}$, where $A=\frac{1}{\eta}-\frac{1}{2}+4M$.
\end{theorem}

\begin{proof}
The proof follows the decomposition and estimates described above.
The exterior bound uses Proposition~\ref{prop:boundary_tail_for_projected} and Hoeffding's inequality (see~\eqref{ineq:upper_bound_boundrary_error_cnn}).
The interior bound uses the path-norm estimate~\eqref{eq:unshared_pathnorm} together with the metric entropy of variation spaces (Lemma~\ref{lem:generalization‐RBV}; see~\eqref{ineq:upper_interior_generalization_cnn}).
The concentration term $\bigl|\Pb(\Iall_\varepsilon)-n_I/n\bigr|$ is controlled as in~\eqref{ineq: high_probability_MSE_upper_bound_cnn} and is of smaller order.
\end{proof}

Optimizing over $\varepsilon$ as in Corollary~\ref{corollary:optimal_bounds} yields the rate exponent
\[
    -\frac{(d-m)(d+3)}{2\bigl((d-m)(d+3)+2d^{2}\bigr)},
\]
which tends to $-\frac{1}{6}$ for fixed~$m$ and $d\to\infty$, identical to the rate obtained for the weight-shared SCN in Theorem~\ref{thm:upper_bound}.

\bigskip
\noindent\textbf{Conclusion.}
The only architectural prerequisite for the stability-induced implicit regularization is the \emph{sparse connectivity pattern}: each hidden neuron receives input from a small subset of the ambient coordinates.
Weight sharing and GAP change how the patch geometry is aggregated. In the GAP model, the shared filter and common output coefficient couple all locations and yield a single global patch-multiset weight $g_{\CD,\CS}$. In the unshared model, the same mechanism appears as a sum of location-wise weights $g_{\CD,\CS}^{(j)}$. Thus weight sharing and GAP are not necessary for stability-induced regularization. They determine whether the patch-geometry penalty is global and pooled or location-wise.
This conclusion substantiates the design principle articulated in Section~\ref{sec:discussion}: the effective vector dimension seen by each local operator governs the strength of implicit regularization, regardless of whether parameters are shared across locations.
\section{Experimental Details}\label{app:experiment_details.}

We adopt the random-design nonparametric regression setting. 

\noindent\textbf{The setting of nonparametric regression.} The non-parametric regression with noisy labels on the random design is one typical setting in this program. Suppose $\{\vec{x}_i\}_{i=1}^n$ are i.i.d.\ sampled from a distribution $\CP_{\vec{X}}$ supported on $\Bb_R^d$ and $y_i=f_{true}(\vec{x}_i)+\xi_i$ for $i\in [n]$, where $f_{true}\!:\Rb^d\rightarrow \Rb$ is the ground-true function and $\{\xi_i\}_{i=1}^n$ are i.i.d. Gaussian noises $\CN(0,\sigma^2)$. In this setting, the goal of the regression task is to find a predictor $f$ to minimize the mean squared error (MSE):
\begin{equation}\label{eq:train_MSE}
	\MSE(f)=\frac{1}{n}\sum_{i=1}^n(f(\vec{x}_i)-f_{true}(\vec{x}_i))^2.
\end{equation}
The population level of \eqref{eq:train_MSE} is known as \emph{excess risk}, 
\begin{equation}\label{eq:excess_risk}
	\mathrm{Excess}(f)\coloneqq\operatorname{\Eb}_{\vec{x}\sim \CP_{\vec{X}}}[(f(\vec{x})-f_{true}(\vec{x}))^2],
\end{equation}
which is also called the \emph{estimation error} under $L^2(\CP_{\vec{X}})$.

In this setting, the population risk (under squared loss) of a predictor $f$ decomposes as
\begin{equation}\label{eq:population_risk_decom}
\Risk(f)=\Eb_{(\vec{x},y)\sim \CP}\big[(f(\vec{x})-y)^2\big]=\mathrm{Excess}(f) + \sigma^2,
\end{equation}
The additive term $\sigma^2$ is the irreducible error contributed by the label noise, and it is achieved by the Bayes predictor $f^*(\vec{x})=\Eb[y\mid \vec{x}]=f_{true}(\vec{x})$, namely $R(f^*)=\sigma^2$. Consequently, $\Risk(f)-\Risk(f^*)=\operatorname{\Eb}_{\vec{x}\sim \CP_{\vec{X}}}[(f(\vec{x})-f_{true}(\vec{x}))^2]$,
so controlling the excess risk is equivalent to controlling the population regret in squared-loss regression.

\smallskip
Moreover, the generalization gap in this random-design regression model admits the following equivalent form
\begin{equation}\label{eq:gen_gap_excess_risk_decom}
\mathrm{GenGap}(f,\CD)=\Big|\mathrm{Excess}(f) + \sigma^2 -\widehat{\Risk}_{\CD}(f)\Big|.
\end{equation}
This identity makes explicit how the generalization gap compares the \emph{population} performance (excess risk plus irreducible noise) against the \emph{training} performance measured by the empirical squared loss.


\subsection{Experimental details for Figure~\ref{fig:gap_vs_n_loglog} in Section~\ref{subsec:cnn_sphere}}
\label{appex:exp-subsec1}
We use synthetic data to empirically validate our claim that SCN generalizes well on spherical data when $m\ll d$, and the rate does not deteriorate as the ambient dimension increases. 

\noindent\textbf{Data generation:} Fix a collection of receptive fields $\CS=\{S_j\}_{j=1}^J$ with patch size $m$. In all experiments we use disjoint coordinate patches
$S_j=\{(j-1)m+1,\dots,jm\}$ so that $J=\lfloor d/m\rfloor$.
We first sample a ground-truth predictor $f_{true}\in\vec{\Theta}^{\CS}$ from the architecture in~\eqref{eq:model_CNN} with a moderate width $K_{true}=20$ (fixed across all $n$), and generate training inputs $\{\vec{x}_i\}_{i=1}^n$ i.i.d. from $\mathrm{Uniform}(\Sph^{d-1})$. The labels are generated by
$y_i=f_{true}(\vec{x}_i)+\xi_i$, where $\xi_i{\sim}\mathcal{N}(0,\sigma^2)$ are i.i.d Gaussian noise. We also generate an independent test set $\{\tilde{\vec{x}}_r\}_{r=1}^N{\sim}\mathrm{Uniform}(\Sph^{d-1})$ (with $N\gg n$) for Monte Carlo evaluation of generalization gap.

\noindent\textbf{Architecture:} We compare two overparameterized two-layer ReLU models trained by full-batch GD:
(i) \textbf{SCN} (two-layer sparsely connected ReLU network in \eqref{eq:model_CNN}),
and (ii) \textbf{FCN} obtained by taking $m=d$ and $J=1$ in~\eqref{eq:model_CNN}.
Unless otherwise stated, both models use width $K=1024$, which places the training in an overparameterized regime for the sample sizes we consider. Both use Kaiming-normal initialization for hidden weights and zero
initialization for all biases.

\noindent\textbf{Metrics:} For each trained predictor $f$, we report (i) \textbf{train loss} $\widehat{\Risk}_{\CD}(f)$ 
This quantity measures how well the predictor fit the training set; 
(ii) \textbf{estimated generalization gap}, $\widehat{\mathrm{GenGap}}(f,\CD)$, which
is the main object predicted by our theory (Theorem~\ref{thm:upper_bound}).
(iii) \textbf{estimated excess risk} $\widehat{\Excess}(f) = \tfrac{1}{N}\sum_{r=1}^N\bigl(f(\tilde{\vec{x}}_r)-f_{\mathrm{true}}(\tilde{\vec{x}}_r)\bigr)^2$, the Monte Carlo estimator of $\Excess(f)$ defined in Eq~\eqref{eq:excess_risk};
(iv) \textbf{Hessian Sharpness} $\lambda_{\max}(\nabla^2\mathcal{L}(\theta_t))$, computed every $500$ epochs via PyHessian~\cite{yao2020pyhessian} to verify the trajectory satisfies BEoS;

\noindent\textbf{Optimization} All models are trained with full-batch gradient descent (no momentum, no weight decay) on the squared loss $\mathcal{L}(\theta)$, using learning rate $\eta = 0.2$ for $30000$ epochs; the corresponding BEoS threshold is $2/\eta = 10$.


\noindent\textbf{Sweep and Scaling} For SCN, we sweep (Figure~\ref{fig:gen_gap_two_plots}) ambient dimension $d \in \{100, 200, 400\}$; for FCN, we fix $d = 10$ as a baseline where fully connected networks are expected to perform well. Sample sizes range over $n \in \{ 128, 256, 512, 1024\}$. Results averaged over 5 seeds. To estimate the sample-size scaling, we sweep $n\in\{128,256,512,1024\}$ and repeat the full pipeline over multiple random seeds. We plot $\log\widehat{\Gen\Gap}(f_{\hat{\vec{\theta}}},\CD)$ versus $\log n$ and report the least-squares fitted slope (\textbf{SCN}: $d=100,200,400$; \textbf{FCN}: $d=10$).

\subsection{Experimental details for Figure~\ref{fig:flat_interpolation}} Similar setup and metrics as the spherical-data setup above in~\ref{appex:exp-subsec1}, with two changes. First, each patch is sampled independently from the unit sphere and concatenated: $\vec{x}=[\vec{x}^{(1)},\dots,\vec{x}^{(J)}]$ with $\vec{x}^{(j)}\sim\mathrm{Uniform}(\Sph^{m-1})$, instead of the whole-vector sampling $\vec{x}\sim\mathrm{Uniform}(\Sph^{d-1})$. Second, we fix $m=10$, $J=8$ (so $d=80$), $K_{\mathrm{true}}=50$, $n=512$, $\sigma=1$, and train at width $K=1024$ for $50{,}000$ epochs at $\eta=0.2$ with gradient clipping. Additional experiments with similar setup in Fig~\ref{fig:plot3_lcn_combined},~\ref{fig:plot4_fcn_combined} and Table~\ref{tab:gen_gap}.

\subsection{Experimental details for Figure~\ref{fig:patch_vs_image_geometry}} We evaluate four patchification schemes on CIFAR-10: (i) \textit{Conv Patches}, $4\times 4$ patches with stride $2$ (baseline); (ii) \textit{Random Patches}, $4\times 4$ contiguous crops at random spatial locations (same patch content as Conv but irregular grid); (iii) \textit{ShufflePixel}, conv-style patches taken after a global pixel permutation (negative control); and (iv) \textit{Full Image Space}, the flattened $3072$-D image. For each scheme we draw $2{,}000{,}000$ reference and $50{,}000$ query patches and report two geometry metrics: (a) the PCA \emph{effective rank}, taken as the elbow of the cumulative explained variance at $90\%$ (lower values indicate more structured distributions); and (b) the area under the half-space (Tukey) depth concentration curve $\Psi(T)=\Pr[\hat d(\vec{z})\ge T]$ for $T\in[0,0.5]$ (higher values indicate distributions that are harder to shatter), computed via the random-projection approximation of~\citet{liang2025datageometry} with $512$ uniformly random unit directions and $4096$ histogram bins per direction. Full Image Space uses randomized PCA for the spectrum and $10000$ query points.

\subsection{Experimental details for Figure~\ref{fig:cifar10_regression}}

We repeat the above setup in~\ref{appex:exp-subsec1} with two changes. First, inputs are CIFAR-10 images $\vec{x}\in\mathbb{R}^{3072}$, and SCN uses patches ($3\times 3\times 3$ , stride $1$; $L=900$, $m=27$). Second, targets are generated by a \emph{frozen} ground-truth SCN $f_{\mathrm{true}}$ of width $K_{\text{true}}=20$ with the same patchification, $y_i=f_{\mathrm{true}}(\vec{x}_i)+\xi_i$, $\xi_i\sim\mathcal{N}(0,1)$; excess risk is reported as the MSE against $f_{\mathrm{true}}$ on the standard CIFAR-10 test split ($N=10000$). We train SCN and FCN, both at width $K=1024$, on $n_{\text{train}}=1024$ images for $5000$ epochs with $\eta=0.2$ and gradient clipping. Both networks are overparameterized in width ($K\gg n_{\text{train}}$), and SCN's training loss reaches the noise floor $\sigma^2$, ruling out underfitting as the explanation for its small excess risk.

\begin{figure}[h]
    \centering
    \includegraphics[width=\linewidth]{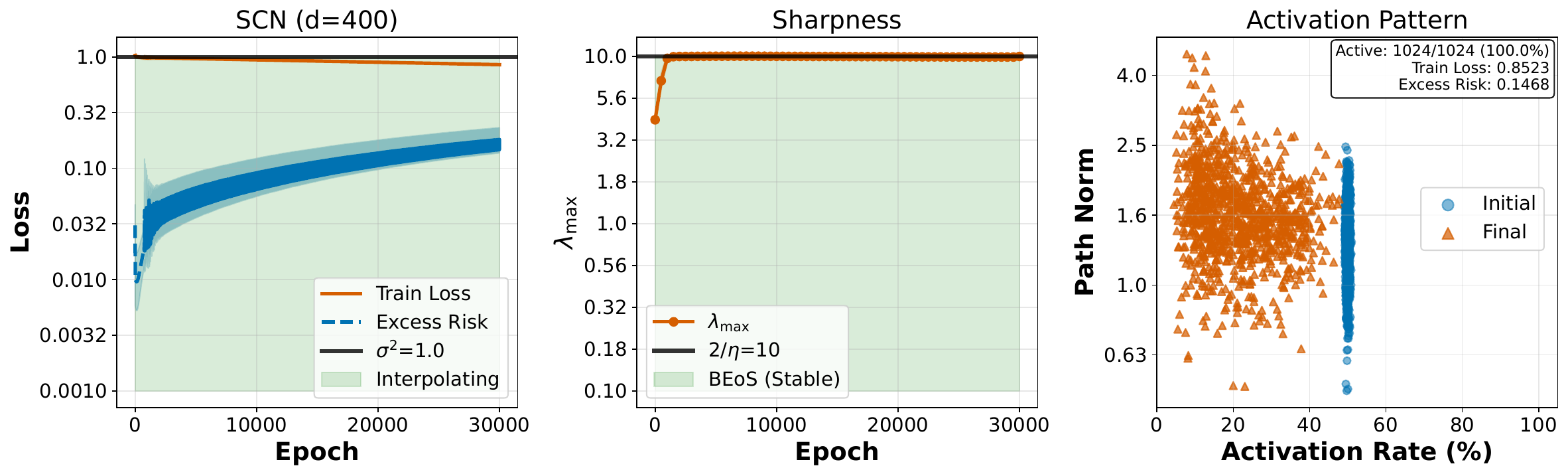}
    \caption{\textbf{SCN generalizes on high-dimensional spherical data.}\textbf{(Left)} With $d=400$ and fixed patch size $m \ll d$, train loss plateaus near $\sigma^2$ while excess risk decreases to $0.15$, confirming generalization rather than memorization. 
    \textbf{(Middle)} Sharpness saturates at BEoS ($\lambda_{\max} \approx 2/\eta$). \textbf{(Right)} Neurons spread across moderate activation rates, unlike the sparse isolation in flat interpolation (Figure~\ref{fig:flat_interpolation}). This validates Theorem~\ref{thm:flat_lcws_gap}: when $m \ll d$, stability-induced regularization prevents overfitting. Results averaged over 5 seeds with $\eta=0.2$, trained for 30k epochs.}
\label{fig:plot3_lcn_combined}
\end{figure}

\begin{table}[h]
\centering
\caption{Generalization gap and excess risk for SCN on spherical data (corresponding to Figure~\ref{fig:gap_vs_n_loglog}). As ambient dimension $d$ increases with fixed patch size $m$, both metrics decrease—confirming the ``blessing of dimensionality''.}
\label{tab:gen_gap}
\begin{tabular}{cccc}
\toprule
$d$ & $J$ & Gen Gap & Excess Risk \\
\midrule
100 & 10 & $1.283 \pm 0.114$ & $0.613 \pm 0.101$ \\
200 & 20 & $0.780 \pm 0.076$ & $0.352 \pm 0.047$ \\
400 & 40 & $0.295 \pm 0.004$ & $0.147 \pm 0.010$ \\
\bottomrule
\end{tabular}
\end{table}

\begin{figure}[h]
    \centering
    \includegraphics[width=\linewidth]{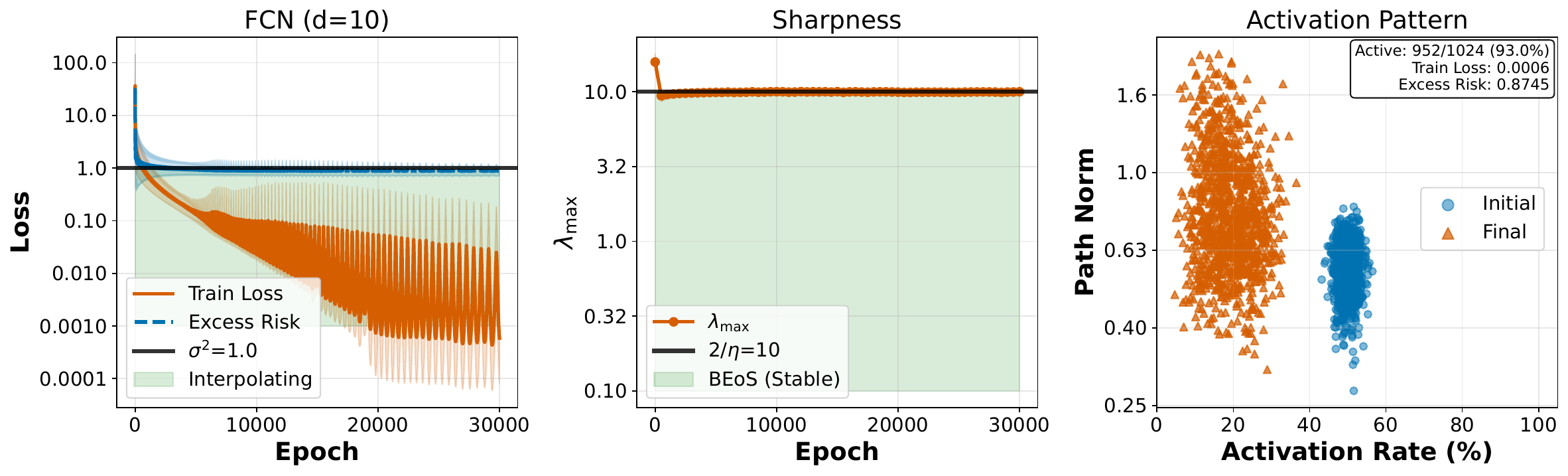}
    \caption{\textbf{FCN satisfies BEoS but still memorizes (Left)} FCN ($d=10$) interpolates noisy labels (train loss $\to 0$) while excess risk remains $\approx \sigma^2$. \textbf{(Middle)} Sharpness saturates at BEoS. \textbf{(Right)} Activation pattern after training. Despite satisfying BEoS, FCN fails to generalize—confirming that on spherical data, stability constraints alone are insufficient without convolutional structure. Averaged over 5 seeds; with similar setting as Figure~\ref{fig:plot3_lcn_combined}}
    \label{fig:plot4_fcn_combined}
\end{figure}

\end{document}